%% file: main.tex
\documentclass{article}
\usepackage{matharticle}
\usepackage{ifthen}
\usepackage{mdwlist}
\usepackage{bm}
\usepackage{xcolor}
\usepackage[colorlinks=true,linkcolor=blue,citecolor=blue,urlcolor=blue]{hyperref}
\usepackage[shortlabels]{enumitem}
\usepackage{graphicx}
\usepackage{xspace}
\usepackage{verbatim}
\usepackage[margin=1in]{geometry}
\usepackage{thm-restate}
\usepackage{latexsym}
\usepackage{epsfig}
\usepackage{mleftright}
\usepackage{bbm}
\usepackage[colorinlistoftodos,textsize=scriptsize]{todonotes}

\include{locdef}

\newcommand{\tcr}[1]{\textcolor{red}{#1}}

\renewcommand\norm[1]{\lVert#1\rVert}

\newcommand{\inner}[2]{\langle #1, #2 \rangle}

\newcommand{\rs}[1]{\textcolor{red}{[RS: #1]}}

\newcommand{\aj}[1]{\textcolor{blue}{[AJ: #1]}}

\title{Efficient List-Decodable Regression using Batches\thanks{Authors are listed in alphabetical order.}} 
\author{
Abhimanyu Das\thanks{Google Research.  Email: \url{abhidas@google.com }.}
\and
Ayush Jain\thanks{UC San Diego. Email: \url{ayjain@eng.ucsd.edu}. This work was done while the author was interning at Google Research.}
\and
Weihao Kong\thanks{Google Research.  Email: \url{weihaokong@google.com}.}
\and
Rajat Sen\thanks{Google Research. Email: \url{senrajat@google.com}.}
}

\usepackage{parskip}
\begin{document}

\maketitle

\begin{abstract}
We begin the study of list-decodable linear regression using batches.
In this setting only an $\alpha \in (0,1]$ fraction of the batches are genuine. Each genuine batch contains $\ge n$ i.i.d. samples from a common unknown distribution and the remaining batches may contain arbitrary or even adversarial samples. 
We derive a polynomial time algorithm that for any $n\ge \tilde \Omega(1/\alpha)$ returns a list of size $\mathcal O(1/\alpha^2)$ such that one of the items in the list is close to the true regression parameter. The algorithm requires only $\tilde\cO(d/\alpha^2)$ genuine batches and works under fairly general assumptions on the distribution.

The results demonstrate the utility of batch structure, which allows for the first polynomial time algorithm for list-decodable regression, which may be impossible for the non-batch setting, as suggested by a recent SQ lower bound~\cite{diakonikolas2021statistical} for the non-batch setting.
\end{abstract}

\section{Introduction}

Linear regression is one of the most fundamental tasks in supervised learning with applications in various sciences and industries~\cite{mcdonald2009handbook, dielman2001applied}. It has been studied extensively in the fixed design setting~\cite{baraud2000model} and the random design setting~\cite{hsu2011analysis}. In the basic random design setting, one is given $m$ samples $(x_i, y_i)$ such that $y_i = \inner{w^*}{x_i} + n_i$ where $n_i$ is the observation noise and the covariates $x_i \in \reals^d$ are drawn i.i.d from some fixed distribution. The task is the recover the unknown regression vector $w^*$. The classical least-squares estimator~\cite{hsu2011analysis} that minimizes the loss $\sum_i f_i(w) := \sum_i (y_i - \inner{w}{x_i})^2/2$ is optimal in the sub-Gaussian covariates setting.

In the presence of even a small amount of adversarial corruption, the classical estimators can fail catastrophically~\cite{huber2011robust}. Such corruptions are common place in many applications like biology~\cite{rosenberg2002genetic, paschou2010ancestry}, machine learning security~\cite{barreno2010security, biggio2012poisoning} and can generally be used to guard against model misspecification. Classical robust estimators have been proposed in~\cite{huber2011robust, rousseeuw1991tutorial} but they suffer from exponential runtime. Recent works~\cite{LaiRV16,diakonikolas2019robust,diakonikolas2017being} have derived efficient algorithms for robust mean estimation with provable guarantees even when a small fraction of the data can be corrupt or adversarial. These works have inspired the efficient algorithms for robust regression~\cite{prasad2018robust,DiakonikolasKKLSS19,diakonikolas2019efficient, pensia2020robust} under the same corruption model. ~\cite{CherapanamjeriATJFB20,jambulapati2021robust} have obtained robust regression algorithms with near-optimal run time and sample complexity.



However, in some applications the majority of the data can be corrupt i.e $\alpha < 1/2$. Such applications include crowd-sourcing~\cite{steinhardt2017certified, steinhardt2016avoiding,charikar2017learning} where a majority of the participants can be unreliable. This setting also generalizes the problem of learning mixture of regressions~\cite{jordan1994hierarchical, zhong2016mixed, kong2020meta, pal2022learning} because any solution of the former immediately yields a solution to the latter by setting $\alpha$ to be the proportion of the data from the smallest mixture component. Note that it is not possible to solve this problem when one is allowed to output only one regression weight $\hat{w}$ as the adversary can hide a completely different solution in the corrupted part of the data. Instead, it is possible to arrive at a solution if the learner is allowed to output a list of size $poly(1/\alpha)$ at least one of which should be close to correct. 
For high dimensional mean estimation~\cite{charikar2017learning} derived the first polynomial algorithm for list decodable setting. \textit{List-decodable} linear regression has been studied in~\cite{karmalkar2019list, raghavendra2020list} yielding algorithms with runtime and sample complexity of $O(d^{\mathrm{poly}(1/\alpha)})$.
In contrast to list-decodable mean estimation, recent work~\cite{diakonikolas2021statistical} has shown that a sub-exponential runtime and sample complexity might be impossible.

These prior results may suggest a pessimistic conclusion for obtaining practical algorithms for one of the most fundamental learning paradigms when a majority of data is corrupt. Fortunately, in many applications like federated learning~\cite{wang2021field}, learning from multiple sensors~\cite{wax1989unique}, and crowd-sourcing~\cite{steinhardt2016avoiding} it is possible to gather batches of data from different sources/agents and a whole batch is either genuine or corrupt. For instance, in crowd-sourcing, an agent can respond to multiple tasks and a particular agent is either adversarial or genuine. Similarly, in federated learning data collected from a device can either be corrupt or not. 

The batch setting has another advantage in the list decodable context where the majority of the data is corrupted. If one is able to identify a set of $poly(1/\alpha)$ answers containing the correct one; then one can use a hold-out portion of a particular batch to find the best-fitting solution for that batch. 
This is especially relevant for the mixture of regression setting where there are $k = 1/\alpha$ different types of sources. A source of type $i$ generates a batch of samples using the true weight $w^{*,i}$. For instance, in federated learning, each device generating a batch of data can belong to one of $k$ different types of users. In this case, if one is able to solve the list-decodable problem of obtaining the weights $\{w^{*,i}\}_{i=1}^k$, then given a few hold-out samples from a batch/source one can quickly identify which of the $k$ answers fit that source the best. Then this weight can be used for future predictions for that particular source. This post-hoc identification of the best weight for a source/batch is naturally not feasible in the single sample setting.

This motivates the problem of list-decodable linear regression using batches. Formally, there are $m$ batches $b \in B$. Each batch $b$ has a collection of $n$ regression samples $(x_i^b, y_i^b)$ which can either all come from a global regression model with true weight $w^*$ (good batch) and noise variance $\sigma^2$ or are arbitrarily corrupted (adversarial batch). The task is to output a list of regression vectors at least one of which is approximately correct given that only $\alpha < 1/2$ fraction of the batches are good. Our main result is the following theorem:

\begin{theorem}[Informal]
\label{thm:informal_main}
There exists an algorithm for list-decodable regression with only $\alpha < 1/2$ fraction of the batches being good, that uses $m = \tilde O_{\bsize,\goodfrac}(d)$ batches each of size $n = \tilde \Omega(1/\alpha)$, and outputs $O(1/\alpha^2)$ weights such that at least one of them, $\tilde{w}$ satisfies $\norm{\tilde w - w^*}_2 = \tilde\cO(\sigma/\sqrt{n\alpha})$ with high probability.
\end{theorem}

It should be noted that list-decodable regression from batches poses several technical challenges that are not addressed by the core techniques in list-decodable mean estimation that has been previously studied~\cite{DiakonikolasKK20}. Our algorithm proceeds by generating a series of weights of different batches. Each such batch weight vector corresponds to a weighted loss function with a stationary point $w$. We need tight concentration bounds on the empirical covariance of the batch gradients of the good batches at all these different $w$'s over the course of the algorithm. In order for our algorithm to succeed with only $O(d)$ batches we use clipped batch gradients and derive novel concentration bounds for these clipped gradients. Further, the clipping level crucially depends on the $w$ among other things. The clipping level is not known as apriori. One of our main contributions is to come up with an algorithm to adaptively set this clipping level such that it is of the order of the typical $\ell_1$ error of using $w$ in good batches.
In Section~\ref{sec:contrib}, we provide a detailed overview of our technical contributions after having defined the problem in Section~\ref{sec:psetting} and presented our main result in Section~\ref{sec:main}. We present our algorithm and proof of the main result in Section~\ref{sec:alg}.

\section{Related Work}
\label{sec:rwork}

{\bf Robust Estimation and Regression.}
Designing estimators which are robust under the presence of outliers has been broadly studied since 1960s~\cite{tukey1960survey, anscombe1960rejection, huber1964robust}. However, most prior works either requires exponential time or have a dimension dependency on the error rate, even for basic problems such as mean estimation. Recently, \cite{diakonikolas2019robust} proposed a filter-based algorithm for mean estimation which achieves polynomial time and has no dependency on the dimensionality in the estimation error. There has been a flurry of research on robust estimation problems, including mean estimation~\cite{LaiRV16,diakonikolas2017beingrobust, dong2019quantum, hopkins2020robust, hopkins2020mean, diakonikolas2018robustly}, covariance estimation~\cite{cheng2019faster, li2020robust}, linear regression and sparse regression \cite{bhatia2015robust,BhatiaJKK17, BDLS17, gao2020robust, prasad2018robust, klivans2018efficient, DiakonikolasKKLSS19, liu2018high, karmalkar2019compressed,dalalyan2019outlier, mukhoty2019globally, diakonikolas2019efficient, karmalkar2019list, pensia2020robust, cherapanamjeri2020optimal}, principal component analysis~\cite{kong2020robust, jambulapati2020robust}, mixture models~\cite{diakonikolas2020robustly, jia2019robustly, kothari2018robust, hopkins2018mixture}. The results on robust linear regression are particularly related to the setting of this work, though those papers considered non-batch settings and the fraction of good examples $\alpha > 1/2$. \cite{prasad2018robust,DiakonikolasKKLSS19,diakonikolas2019efficient, pensia2020robust, cherapanamjeri2020optimal,jambulapati2021robust} considered the setting when both both covariate $x_i$ and label $y_i$ are corrupted. When there are only label corruptions, \cite{bhatia2015robust, dalalyan2019outlier, kong2022trimmed} achieve nearly optimal rates with $O(d)$ samples. Under the oblivious label corruption model, i.e., the adversary only corrupts a fraction of labels in complete ignorance of the data, \cite{ bhatia2017consistent, suggala2019adaptive} provide a consistent estimator whose approximate error goes to zero as the sample size goes to infinity.


{\bf Robust Learning from Batches.}
\cite{QiaoV18} introduced the problem of learning discrete distribution from untrusted batches and derived an exponential time algorithm. Subsequent works~\cite{chen2019learning} improved the run-time to quasi-polynomial and~\cite{JainO20a} obtained polynomial time with an optimal sample complexity. ~\cite{JainO21, chen2020learning} extended these results to one-dimensional structured distributions.
~\cite{JainO20b,konstantinov2020sample} studied the problem of classification from untrusted batches. 
~\cite{acharya2022robust} studies a closely related problem of learning parameter of Erd\H{o}s-R\'enyi random graph when a fraction of nodes are corrupt. 
All these works focus on different problems than ours and only consider the case when a majority of the data is genuine.

{\bf List Decodable Mean Estimation and Regression.}
List decodable framework was first introduced in~\cite{charikar2017learning} to obtain learning guarantees when a majority of data is corrupt. They derived the first polynomial algorithm for list decodable mean estimation under co-variance bound. Subsequent works~\cite{DiakonikolasKK20,cherapanamjeri2020list,diakonikolas2021list} obtained a better run time. \cite{diakonikolas2018list,kothari2018robust} improved the error guarantees, however, under stronger distributional assumptions and has higher sample and time complexities.


\cite{karmalkar2019list} studies the problem of list-decodable linear regression with batch-size $n=1$ and derive an algorithm with sample complexity $(d/\alpha)^{O(1/\alpha^4)}$ and runtime $(d/\alpha)^{O(1/\alpha^8)}$.
\cite{raghavendra2020list} show a sample complexity of $(d/\alpha)^{O(1/\alpha^4)}$ with runtime $(d/\alpha)^{O(1/\alpha^8)}(1/\alpha)^{\log(1/\alpha)}$. 
Polynomial time might indeed be impossible for the single sample setting owing to the statistical query lower bounds in~\cite{diakonikolas2021statistical}.

{\bf Mixed Linear Regression.}
When each batch has only one sample, (i.e. $n=1$) and contains samples of one of the $k$ regression components the problem becomes the classical mixed linear regression which has been widely studied~\cite{diakonikolas2020small, chen2019learning, li2018learning, sedghi2016provable, zhong2016mixed, yi2016solving, CL13}. It is worth noting that no algorithm is known to achieve polynomial sample complexity in this setting. 
The problem is only studied very recently in the batched setting with $n>1$ by~\cite{kong2020meta, kong2020robust}, where all the samples in the batch are from the same component. \cite{kong2020meta} proposed a polynomial time algorithm which requires $O(d)$ batches each with size $O(\sqrt{k})$.  \cite{kong2020robust} leveraged sum-of-squares hierarchy to introduce a class of algorithms which is able to trade off the batch size $n$ and the sample complexity. Both of these works assume that the distributions of covariates for all components is identical and Gaussian.
Since the above problem is a special case of the list-decodable linear regression, our algorithm is able to recover the $k$ regression components with batch size $n = O(k)$ and $O(d)$ number of batches.
Our algorithms allow more general distributions for the covariates than allowed by the Gaussian assumption in the previous works. Further, our algorithms allow the distributions of covariates for the different components to differ.
It is worth noting that list-decodable linear regression is a strictly harder problem than mixed linear regression as shown in~\cite{diakonikolas2021statistical} and thus our result is incomparable to the ones in the mixed linear regression setting. Leaning mixture of linear dynamical systems has been studied in~\cite{chen22t}.


\section{Problem formulation}
\label{sec:psetting}
We have $\btotal$ sources. Of these $\btotal$ sources at least $\goodfrac$-fraction of the sources are 
genuine and provide $\ge \bsize$ i.i.d. samples from a common distribution.
The remaining sources may provide arbitrary data, that may depend even on the samples from genuine sources. The identity of the genuine sources is unknown.
We can use only the first $\bsize$ samples from each source and ignore the rest, hence, wlog we assume that each source provides exactly $\bsize$ samples. We will refer to the collection of all samples from a single source as a batch.

To formalize the setting, let $\allbatches$ be a collection of $\btotal$ batches.
Each batch $\batch\in \allbatches$ in this collection, has $\bsize$ samples $\{(x^\batch_i,y^\batch_i)\}_{i=1}^\bsize$, where $x^\batch_i\in \domain$ and $y^\batch_i\in \reals$.

Among these batches $\allbatches$, there is a sub-collection $\goodbatches$ of \emph{good batches}
such that for each $\batch\in \goodbatches$ and $i\in [\bsize]$ samples $(x^\batch_i,y^\batch_i)$ are generated independently from a common distribution $\dist$ and the size of this sub-collection is $\sizegood \ge \goodfrac \size{\allbatches}$.
The remaining batches $\advbatches\ed \allbatches\setminus\goodbatches$ of \emph{adversarial batches} have arbitrary samples that may be selected by an adversary depending on good batches.

Next, we describe the assumption of distribution $\dist$.
We require the same set of general assumptions on the distribution, as in the recent work~\cite{CherapanamjeriATJFB20}, which focuses on the case when $1-\goodfrac$ is small, that is when all but a small fraction of data is genuine.

\paragraph{Distribution Assumptions.} For an unknown $d$-dimensional vector $w^*$, the \emph{sample noises} $n_i^\batch$, the \textit{covariates} \textit{$x_i^\batch$} and the outputs $y_i^\batch$ are random variables that are related as $y_i^\batch = x_i^\batch\cdot w^* +n_i^\batch$. Let $\Sigma = \E_{\dist}[x_i^\batch (x_i^\batch)^\intercal]$. For scaling purposes, we assume $\|\Sigma\|=1$. We have the following general assumptions.
\begin{enumerate}
    \item $x_i^\batch$ is $L4$-$L2$ hypercontractive, that is for some $C\ge 1$ and all vectors $u$, $\E_{\dist}[(x_i^\batch\cdot u)^4]\le C\E_{\dist}[(x_i^\batch\cdot u)^2]^2$.
    \item For some constant $C_1>0$, $\|x_i^\batch\|\le C_1\sqrt{d}$ a.s.
    \item The condition number of $\Sigma$ is at most $\connum$, that is for each unit vector $u$, we have $u^\intercal\Sigma u \ge \frac{\|\Sigma\|}{\connum} = \frac{1}{\connum}$.
    \item Sample noise $n_i^\batch$ is independent of $x_i^\batch$ and has zero mean $\E{}_\dist[n_i^\batch] = 0$ and bounded co-variance $\E{}_\dist[(n_i^\batch)^2] \le \sigma^2$.
    \item For some $\npower\ge 2$, noise is $L\npower$-$L2$ hypercontractive, that is $\E{}_\dist[|n_i^\batch|^\npower]\le \nhypcon(\E{}_\dist[(n_i^\batch)^2])^{\npower/2} \le \nhypcon\sigma^\npower$. Note that this assumption holds for $\npower=2$ and $\nhypcon=1$ if the assumption in item 4 holds. The assumption for $\npower=2$ suffices for our results. However, if it holds a larger $\npower$ number of batches required in our theorem improves.   
\end{enumerate}

\section{Notation}
For any function $h^\batch$ over batches $\batch$, we use $\E_{\dist}[h^\batch]$ to denote the expected value of $h^\batch$ when all $\bsize$ samples in batch $\batch$ were generated independently from $\dist$.\footnote{A function over batches may be a function of some or all the samples in the batches. With slight abuse of notation, instead of $h(\batch)$, we use $h^\batch$ to denote function over batches. 
}

For any batch sub-collection $\Bsc\subseteq\allbatches$ and any function $h^\batch$ over batches, $\E_{\Bsc} [h^\batch] = \sum_{\batch\in \Bsc} \frac{1}{\size{\Bsc}}h^\batch$ and $\Cov_{\Bsc}(h^\batch) = \sum_{\batch\in \Bsc}\frac{1}{\size{\Bsc}} (h^\batch-\E_{\Bsc} [h^\batch])(h^\batch-\E_{\Bsc} [h^\batch])^\intercal$ denote the expectation and co-variance of $h^\batch$, respectively, when batch $\batch$ is sampled uniformly from $\Bsc$. 

The above notation can be generalized to the weighted expectation and co-variance over batch sub-collections.
A \emph{weight vector} $\weight$ is a collection of weights $\weight^\batch\in [0,1]$ for each batch $\batch \in \allbatches$. For weight vector $\weight$ and any batch sub-collection $\Bsc\subseteq\allbatches$, we use $\weight^{\Bsc}:= \sum_{\batch\in\Bsc}\weight^\batch$ to denote the weight of all batches in $\Bsc$.
It follows that $\weight^\Bsc \le \size{\Bsc}$.
We refer to $\weight^\allbatches$ as the \emph{total weight} of weight vector $\weight$. Any such weight vector $\weight$ will also be referred to a \textit{soft cluster} of batches.

For any function $h^\batch$ over batches, let $\E_{\weight} [h^\batch] \ed \sum_{\batch\in \allbatches} \frac{\weight^\batch}{\weight^{\allbatches}}h^\batch$ and $\Cov_{\weight}(h^\batch) \ed \sum_{\batch\in \allbatches}\frac{\weight^\batch}{\weight^{\allbatches}} (h^\batch-\E_{\weight} [h^\batch])(h^\batch-\E_{\weight} [h^\batch])^\intercal$ denote the expectation and co-variance of $h^\batch$, respectively, when probability of sampling a batch $\batch\in \allbatches$ is $\frac{\weight^\batch}{\weight^{\allbatches}}$. 

We use $f(x) = \tilde \cO(g(x))$ as a shorthand for $f(x) =  \cO(g(x)\log^k x)$, where $k$ is some integer, and $f(x) = \cO_y(g(x))$ implies that if $y$ is bounded then $f(x) = \cO(g(x))$.

\section{Main Results}
\label{sec:main}
Recently there has been a significant interest in the problem of list decodable linear regression.
The prior works considered only the non-batch setting. 
The sample and time complexity of algorithm in \cite{karmalkar2019list,raghavendra2020list} are $d^{\cO(1/\goodfrac^4)}$ and $d^{\cO(1/\goodfrac^8)}$, respectively. ~\cite{raghavendra2020list} achieves an error  $\cO(\sigma/\goodfrac^{3/2})$ with a list of size $(1/\goodfrac)^{\cO(\log(1/\goodfrac))}$, and ~\cite{karmalkar2019list} obtains an error guarantee $\cO(\sigma/\goodfrac)$ with a list of size ($1/\goodfrac$).

\cite{diakonikolas2021statistical} improved the sample complexity for the non-batch setting, 
when covariates are distributed according to standard Gaussian distribution and Gaussian noise. They gave an information-theoretic algorithm that uses $\cO(d/\goodfrac^3)$ samples and estimates $w$ to a accuracy $ \cO(\sigma\sqrt{\log(1/\goodfrac)}/\goodfrac)$ using a list of size $\cO(1/\goodfrac)$. They also showed that no algorithm, even with infinite samples, can achieve an error  $\ll  \sigma/\goodfrac\sqrt{\log(1/\goodfrac)}$ with a Poly($1/\goodfrac$) size list.

As these works considered the non-batch setting, they do not obtain a polynomial time algorithm for this problem, which may in fact be impossible~\cite{diakonikolas2021statistical}.

Our main result shows that using batches one can achieve a polynomial time algorithm for this setting using only $\tilde\cO_{n,\goodfrac}(d)$ genuine samples. 

\begin{theorem}\label{thm:mainresult}
For any $0< \goodfrac <1$, $\bsize \ge \Theta(\frac{\connum^2 C^2\log^2(2/\goodfrac)}{\goodfrac})$ and $\sizegood = \Omega_{C}\mleft(d\bsize^2 \log(d)  \mleft(\frac{\nhypcon\sqrt{\bsize\goodfrac} }{\log(2/\goodfrac)}\mright)^{\frac8{(\npower-1)}}\mright)$, Algorithm~\ref{alg:main} runs in time poly$(\sizegood, \goodfrac, d, \bsize)$ and returns a list $L$ of size at most $4/\goodfrac^2$ such that with probability $\ge 1-4/d^2$,
\[
\min_{w\in L}\|w-w^*\|\le\cO\mleft(\frac{\connum C\log(2/\goodfrac)}{\sqrt{\bsize\goodfrac}}\sigma\mright).\] 
\end{theorem}

Interestingly, for $\bsize = \tilde\Omega(1/\goodfrac)$, the estimation error of our polynomial algorithm has a better dependence on $\goodfrac$ than the best possible $\sigma/\goodfrac\sqrt{\log(1/\goodfrac)}$ for $\bsize = 1$, for any algorithm (with infinite resources) and 
even under more restrictive assumptions on the distribution of genuine data.

We restate the above result as the following corollary, which for a given $\epsilon, d$ and $\goodfrac$ characterizes the number of good batches $\sizegood$ and $\bsize$ required by Algorithm~\ref{alg:main} to achieve an estimation error $\cO(\epsilon\sigma)$.

\begin{corollary}
For any $0< \goodfrac <1$, $0\le \epsilon \le 1$, $n_{\min} = \Theta_{\connum,C}(\frac{\log^2(2/\goodfrac)}{\goodfrac \epsilon^2})$,  $\bsize \ge n_{\min} $, and $\sizegood = \Omega_{C}\mleft(d n_{\min}^2\log(d)  \mleft(\frac1\epsilon\mright)^{\frac8{(\npower-1)}}\mright)$, Algorithm~\ref{alg:main} runs in time poly$(\goodfrac, d, \epsilon)$ and returns a list $L$ of size at most $4/\goodfrac^2$ such that with probability $\ge 1-4/d^2$,
\[
\min_{w\in L}\|w-w^*\|\le\cO\mleft(\epsilon\sigma\mright).\] 
\end{corollary}

For $\epsilon = \Theta(1)$ in the above corollary, we get $n = \tilde \Omega(\frac{1}{\goodfrac})$ and $\sizegood = \tilde\Omega_{C}\mleft( \frac{d\log(d)}{\goodfrac^2}\mright)$.

Recall that we have assumed that the noise is $L\npower$-$L2$ hypercontractive for some $\npower\ge 2$. 
For sub-Gaussian noise, since $p= \infty$, hence 
the number of batches needed in the above corollary are $\sizegood = \Omega_{C}\mleft(d n_{\min}^2\log(d) \mright)$.

\begin{remark}
As discussed earlier, for the case where a majority of data is genuine, i.e. $\alpha >1/2$ then the algorithms in prior works~\cite{prasad2018robust,DiakonikolasKKLSS19,cherapanamjeri2020optimal} estimate the regression parameter even in the non-batch setting, and with a list of size 1. In particular,~\cite{cherapanamjeri2020optimal}, that requires $\cO(d/(1-\alpha)^2)$ genuine samples, runs in time linear in the number of samples and dimension $d$, and estimates the regression parameter $w^*$ to an $\ell_2$ distance $\cO(\connum\sqrt{(1-\alpha)}\sigma)$ for any $1-\alpha = \cO(\frac{1}{\connum^2})$, where $\connum$ is condition number of covariance matrix $\sigma$ of the covariates. A lower bound of $\Omega(\connum\sqrt{(1-\alpha)}\sigma)$ is also known for the non-batch setting.
The above results for the non-batch setting and $\alpha >1/2$ can be extended for the batch setting.
In the batch setting, by using batch gradients, instead of sample gradients in the algorithm of~\cite{cherapanamjeri2020optimal}, one can estimate the regression parameter $w^*$ to an $\ell_2$ distance $\cO(\connum\sqrt{(1-\alpha)}\sigma/\sqrt n)$. 
\end{remark}

\section{Technical Overview}\label{sec:techover}
\label{sec:contrib}



Recall that we have a collection $\allbatches$ of batches, each batch containing $\bsize$ samples.
An unknown sub-collection $\goodbatches\subseteq\allbatches$ of size $\sizegood\ge \goodfrac\size{\allbatches}$ has i.i.d. samples from $\dist$ and the remaining batches may have arbitrary samples.
We will work with the square loss of a regression weight $w$ i.e $f^\batch_i(w) = (w\cdot x_i^\batch-y_i^\batch)^2/2$ which denotes the loss function for $i^{th}$ sample of batch $\batch$ over parameter space $w$. The average loss in batch $b$ would be denoted by $f^\batch(w) = \frac{1}{n} \sum_i f^\batch_i(w)$.

In the presence of no adversarial batches the stationary point of the average of the losses across all batches i.e $\E_B[f^\batch(w)]$ would converge to $w = w^*$. However, even a single adversarial sample can cause this approach to fail catastrophically.


The gradient of the loss function $f^\batch_i(w) $ is $\sampgrad = (w\cdot x_i^\batch-y_i^\batch)\cdot x_i^\batch$.
Since samples in good batches are drawn from distribution $\dist$, for batch $\batch\in \goodbatches$ the expected value of the gradient is $\E_{\dist}[\sampgrad]= \Sigma\, (w-w^*)$.

When $\sizegood$ is sufficiently large,  
\[
\|\E{}_{\goodbatches}[\sampgrad]\| = \Big\|\frac{1}{\sizegood\bsize}\sum_{\batch\in \goodbatches}\sum_{i\in [\bsize]}\sampgrad  \Big\|\approx \|\E{}_{\dist}[\sampgrad]\| = \|\Sigma\, (w-w^*)\|.
\]
Suppose $\tilde w$ is a stationary point of the all samples, namely $\E{}_{\allbatches}[\nabla f^\batch_i (\tilde w)] = 0$.
If $\tilde w$ is far from $w^*$ then the mean of gradients for samples in good batches differ significantly from the mean of gradients for all samples. This difference ensures that the norm of the co-variance of the sample gradients is  at least $\|\Cov_{\allbatches}[\nabla f^\batch_i (\tilde w)]\|\ge \frac{\sizegood}{\size{\allbatches}}\|\E{}_{\goodbatches}[\nabla f^\batch_i (\tilde w)]-\E{}_{\allbatches}[\nabla f^\batch_i (\tilde w)]\| ^2
\approx \goodfrac\|\Sigma\, (\tilde w-w^*)\|^2$.

When the co-variance of good sample points is much smaller than the overall co-variance of all samples it is possible to iteratively divide or filter samples in two (possibly overlapping) clusters such that one of the clusters is ``cleaner" than the original~\cite{steinhardt2016avoiding,DiakonikolasKK20}.
So if we had $\|\Cov_{\allbatches}[\nabla f^\batch_i (\tilde w)]\|\gg \|\Cov_{\goodbatches}[\nabla f^\batch_i (\tilde w)]\|$ then we could have obtained a ``cleaner version" of $\allbatches$, that had a higher fraction of good batches.

However, for batch $\batch\in \goodbatches$ the norm of co-variance of gradients (of a single sample) is $\|\Cov_{\dist}[\nabla f^\batch_i (\tilde w)]\|=\Theta( \sigma^2 + \|\Sigma (\tilde w-w^*)\|^2)$ (using $L_2$ to $L_4$ hypercontractivity).
Even if we had $\|\Cov_{\goodbatches}[\nabla f^\batch_i (\tilde w)]\|\approx \|\Cov_{\dist}[\nabla f^\batch_i (\tilde w)]\|$, no matter how large the difference between the stationary point $w^*$ for the distribution $\dist$ and the stationary point $\tilde w$ for all samples is, since $\Theta( \sigma^2 + \|\Sigma (\tilde w-w^*)\|^2) =\Omega( \goodfrac\|\Sigma\, (\tilde w-w^*)\|^2) $ always holds, it fails to guarantee $\|\Cov_{\allbatches}[\nabla f^\batch_i (\tilde w)]\|\gg \|\Cov_{\goodbatches}[\nabla f^\batch_i (\tilde w)]\|$. We will now see that focusing on batch gradients rather than single sample gradients can alleviate this problem.

\paragraph{How batches help.} In the preceding approach we didn't leverage the batch structure. In fact~\cite{diakonikolas2021statistical} considered the non-batch setting and showed an SQ lower bound that suggests that a polynomial time algorithm for the non-batch setting may be impossible to achieve.

To take the advantage of the batch structure instead of considering the loss function and its gradient for each sample individually, we consider the loss of a batch and the gradient of the batch loss.

The loss function of a batch $\batch$ is $f^\batch(w) =\frac{1}{\bsize} \sum_{i=1}^\bsize f^\batch_i(w)$ i.e the average of the loss function in its samples. From the linearity of differentiation, the gradient of the loss function of a batch is the average of the gradient of the loss function of its samples.

Averaging preserves the expectation of the gradients, therefore, $\|\E{}_{\goodbatches}[\sampgrad]\|=\|\E{}_{\goodbatches}[\batchgrad]\|$ for any $w$. 
But averaging over $\bsize$ samples reduces the co-variance by a factor $\bsize$. Hence, for a good batch $\batch$, and for any $i\in [\bsize]$ and all $w$, we have $\Cov_{\dist}[\batchgrad]= \Cov_{\dist}[\sampgrad]/\bsize$.

And for $\sizegood$ large enough we will have population co-variance  
\[
\|\Cov_{\goodbatches}[\batchgradtilde]\| = \cO(\|\Cov_{\dist}[\batchgradtilde]\|) = \cO\mleft(\frac{\|\Cov_{\dist}[\sampgradtilde]\|}\bsize\mright)  = \cO\mleft(\frac{\sigma^2 + \|\Sigma (\tilde w-w^*)\|^2}\bsize\mright).
\]

If the batch size $\bsize = \Omega(\frac{\log^2(1/\goodfrac)}{\goodfrac})$ and $\tilde w$ is stationary point of average loss of all samples in $\allbatches$, then for a large value of $\|\tilde w-w^*\|=\Omega(\frac{\sigma\log(1/\goodfrac)}{\sqrt{\bsize\goodfrac}})$, we would have 
\[
\|\Cov_{\allbatches}[\batchgradtilde]\|\ge  \goodfrac\|\Sigma\, (\tilde w-w^*)\|^2 \gg  \log^2(1/\goodfrac) \cO\mleft(\frac{\sigma^2 + \|\Sigma (\tilde w-w^*)\|^2)}\bsize\mright)  \gg  \log^2(1/\goodfrac) \|\Cov_{\goodbatches}[\batchgradtilde]\| .
\]

Since the co-variance of good sample points is much smaller than the overall co-variance of all samples we can iteratively divide or filter samples in two (possibly overlapping) clusters such that one of the clusters is ``cleaner" than the original.
We will use \textsc{Multifilter} routine from~\cite{DiakonikolasKK20} for this purpose. Instead of hard clustering, this routine does soft clustering.
The soft clustering produces a membership or weight vector $\weight$ of length $|\allbatches|$ with each entry between $[0,1]$ that denotes the weight of the corresponding batch in the cluster.
While for simplicity we presented the argument for the initial cluster $B$, the same argument can be easily extended to any cluster that retains a major portion of good batches $\goodbatches$.

We start with initial cluster $\allbatches$.
We keep applying \textsc{Multifilter} routine iteratively on the clusters (or weight vectors) until covariance for all the clusters becomes small. The algorithm ensures that at least one of the clusters retains a major portion of good batches $\goodbatches$. From the discussion in the preceding paragraph, it follows that for any cluster that retains a major portion of good batches $\goodbatches$ the covariance is small only if stationary point $\tilde w$ of this cluster $\|\tilde w-w^*\|=\cO(\frac{\sigma\log(1/\goodfrac)}{\sqrt{\bsize\goodfrac}})$.
Therefore, there is at least one cluster in the end for whose stationary point $\tilde w$ satisfy $\|\tilde w-w^*\|=\cO(\frac{\sigma\log(1/\goodfrac)}{\sqrt{\bsize\goodfrac}})$. We also ensure that we do not have more than $\cO(1/\goodfrac^2)$ clusters at any stage. However, to apply this routine for our purpose we face additional challenges, which we address through several technical contributions.

\paragraph{Clipping to improve sample complexity.}
We would like to obtain a high probability concentration bound of $\|\Cov_{\goodbatches}[\batchgrad]\| \le \cO(\frac{\|w-w^*\|^2+\sigma^2}{\bsize})$ on the empirical covariance of the batch gradients in the good batches.
However, if we use known concentration bounds we would need a lot of good batches $\goodbatches$ (or samples). For instance, ~\cite{DiakonikolasKKLSS19} needed $d^5$ samples in total (for $\bsize=1$), and it can be shown that at least $d^2$ samples will be necessary using this approach.
~\cite{CherapanamjeriATJFB20} required $O(d)$ samples (for $\bsize=1$), but for each point $w$ they need to ignore certain samples. These samples can be different depending on $w$. While such guarantees sufficed for their application where a majority of data was genuine, it is unclear if it can be extended to the list-decodable setting.

To overcome this challenge we instead use Huber's loss. For a sample point $(x,y)$ the Huber loss at point $w$ is $(w\cdot x-y)^2/2$ for $|w\cdot x-y| \le \kappa$ and is linear in $|w\cdot x-y|$ otherwise. Here $\kappa$ is a \emph{clipping parameter} we choose in the algorithm. For a batch of samples, the Huber's loss is the average of Huber's loss of samples in this batch. From here onward we refer to Huber's loss as \emph{clipped loss}.

For any batch $\batch\in \allbatches$, the clipped loss for \emph{clipping parameter} $\kappa>0$ is given by
\[
f^\batch_i(w,\kappa) \ed
\begin{cases}
\frac{(w\cdot x_i^\batch-y_i^\batch)^2}2 & \text{if }|w\cdot x_i^\batch-y_i^\batch|\le \kappa\\
\kappa|w\cdot x_i^\batch-y_i^\batch|-\kappa^2/2 &\text{otherwise.}
\end{cases}
\]
Note that  $f^\batch(w,\kappa)\le f^\batch(w)$.

For {clipping parameter} $\kappa>0$, the gradient of clipped loss of $i^{th}$ sample of batch $\batch$ on point $w$ is
\[
\clipsampgrad\ed \frac{(x^\batch_i\cdot w-y^\batch_i)}{|x^\batch_i\cdot w-y^\batch_i|\vee \kappa} \kappa x^\batch_i. 
\]
We refer to the gradient of the clipped loss above as the \emph{clipped gradient}.
Recall that for a good batch $\batch\in\goodbatches$, $y^\batch_i= w^*\cdot x_i^\batch+n_i^\batch$, hence 
\begin{align}\label{eq:gradexp}
\clipsampgrad= \frac{(x^\batch_i\cdot (w-w^*)-n^\batch_i)}{|x^\batch_i\cdot (w-w^*)-n^\batch_i|\vee \kappa} \kappa x^\batch_i.
\end{align}

For a batch $b$, its loss, clipped loss, gradient, and clipped gradient are simply the average of the respective quantity over all its samples.
\emph{Loss} and \emph{clipped loss} for a batch $\batch$ at point $w$ is $f^\batch(w) \ed \frac{1}{\bsize}\sum_{i\in[\bsize]}f^\batch_i(w)$ and 
$f^\batch(w,\kappa) \ed \frac{1}{\bsize}\sum_{i\in[\bsize]}f^\batch_i(w,\kappa)$, respectively. 
Similarly, \emph{gradient} and \emph{clipped gradient} for a batch $\batch$ at point $w$ are $\batchgrad\ed\batchgradexp$ and $\clipbatchgrad\ed\clipbatchgradexp$, respectively. 
From the linearity of gradients, it follows that $\clipbatchgrad$ and $\batchgrad$ are gradients of $f^\batch(w,\kappa)$ and $f^\batch(w)$, respectively.

\paragraph{Ideal choice of Clipping parameter.}
Note that when the clipping parameter $\kappa\to \infty$, then clipped loss is the same as squared loss.
The other extreme is $\kappa \to 0$ which clips too much and potentially introduces bias in the gradients and makes their norm close to zero for a point $w$ far away from $w^*$.
Theorem~\ref{th:clipmean} shows that as long as clipping parameter $\kappa = \Omega(\|w-w^*\|) + \Omega_{\bsize\goodfrac}(\sigma)$
the expected norm of the 
clipped gradient $\|\E_{\dist}[\clipsampgrad]\| = \Omega(\|w-w^*\|) - \tilde\cO(\sigma/\sqrt{\bsize \goodfrac})$.
Therefore,  as long as $\|w-w^*\| =  \tilde\Omega(\sigma/\sqrt{\bsize \goodfrac})$, the norm of the clipped gradient at $w$ is $\|\E_{\dist}[\clipsampgrad]\| = \Omega(\|w-w^*\|)$, which is of the same order as the unclipped gradients.

Furthermore, Theorem~\ref{th:conofcov} shows for all points $w$, and for any clipping parameter $\kappa = \cO(\|w-w^*\|) + \cO_{\bsize\goodfrac}(\sigma)$ the covariance of the clipped gradients satisfies $\|\Cov_{\goodbatches}[\clipbatchgrad]\| \le \cO(\frac{\|w-w^*\|^2+\sigma^2}{\bsize})$ with only $\tilde O_{\bsize,\goodfrac}(d)$ samples. The same bound on the covariance of the un-clipped gradients would instead require $\Omega(d^2)$ samples. 

If we choose clipping parameter $\kappa \gg \Omega(\|w-w^*\|) + \Omega_{\bsize\goodfrac}(\sigma)$ then we pay in terms of the number of samples required for the covariance bound to hold. 
If we choose the clipping parameter to be too small then the clipping operation may cause $\|\E_{\dist}[\clipsampgrad]\|$ to be small even when $\|w-w^*\|$ is large.

For the ideal choice of clipping parameter $\kappa = \Theta(\|w-w^*\|) + \Theta_{\bsize\goodfrac}(\sigma)$ we need a constant factor approximation of $\|w-w^*\|$.
Such an approximation of $\|w-w^*\|$ is also required for estimating the upper bound on the norm of covariance of the clipped gradients $\|\Cov_{\goodbatches}[\clipbatchgrad]\|$.
Further, we need to do this estimation each time we update $w$.

\paragraph{Estimation of $\|w-w^*\|$.} To estimate $\|w-w^*\|$, again we designed a multi-filter type algorithm. For any (soft) cluster of batches $\weight$, that has more than $3\goodfrac\size{\allbatches}/4$ weight of good batches we either produce a cleaner version of this cluster or else find $\kappa$, $w$ and an estimate of $\|w-w^*\|$ such that 
(1) $w$ is a stationary point for clipped gradients $\E_{\weight}[f^\batch(\tilde{w})]$ , (2) $\kappa = \Theta(\E_{\weight}[\frac{1}{\bsize}\sum_{i\in[\bsize]}|w\cdot x_i^\batch-y_i^\batch|]) + \Theta_{\bsize\goodfrac}(\sigma)$ and 3) $\E_{\weight}[\frac{1}{\bsize}\sum_{i\in[\bsize]}|w\cdot x_i^\batch-y_i^\batch|] = \Theta(\|w-w^*\|)\pm \cO(\sigma)$.

Let $v^\batch(w) = \frac{1}{\bsize}\sum_{i\in[\bsize]}|w\cdot x_i^\batch-y_i^\batch|$. 
The purpose of bound on $\E_{\weight}[v^\batch(w)]$ in the third condition is two folds. First, it ensures that the choice of the clipping parameter $\kappa$ falls in the ideal range discussed before. Second, it is used to estimate an upper bound on $\|\Cov_{\goodbatches}[\clipbatchgrad]\|$.

We first derived a subroutine \textsc{FindClippingPparameter} (Algorithm~\ref{alg:clip}) that for any soft cluster $\weight$ finds a $w$ and $\kappa$ such that $w$ is a stationary point for clipped gradients $\clipbatchgrad$ for batches in the cluster and $\kappa = \Theta(\E_{\weight}\mleft[v^\batch(w)\mright])+ \Theta_{\bsize\goodfrac}(\sigma)$. 
Note that is the stationary point $w$ of the clipped loss and hence depends on the clipping parameter $\kappa$, which itself depends on $w$.
\textsc{FindClippingPparameter} overcomes the challenge posed due to the intertwined nature of the two parameters.

For a cluster $\weight$ if the variance of $v^\batch(w)$ is much higher compared to the variance of $v^\batch(w)$ for good batches then we apply multi-filtering of batches based on $v^\batch(w)$ to produce a cleaner cluster and if the variance of $v^\batch(w)$ is not too high then we show that expected value of $v^\batch(w)$ in the cluster is $\Theta(\|w-w^*\|)\pm\cO(\sigma)$. We explain more details of this algorithm as follows.

For $\sizegood= \tilde\Omega(d)$, we prove the following concentration bound
\begin{align}\label{eq:techovers}
\Var_{\goodbatches}\mleft(v^\batch(w)\mright)\le \E{}_{\goodbatches}\mleft[(v^\batch(w)-\E{}_{\dist}\mleft[v^\batch(w) \mright])^2\mright] =\cO\mleft(\frac{\sigma^2+ \E{}_{\dist}[v^\batch(w)]^2}{\bsize}\mright).    
\end{align}
From the above bound, it follows that for most good batches $v^\batch(w) = \Theta(\E{}_{\dist}[v^\batch(w)])\pm\cO(\sigma)$. For a weight vector (or soft cluster) $\weight$ with enough good batches, we can estimate $\thresholdA$ such that $\thresholdA = \Omega(\E{}_{\dist}[v^\batch(w)]-\sigma)$.
Then from Equation~\eqref{eq:techovers} and this bound on $\thresholdA$, we have  $\Var_{\goodbatches}\mleft(v^\batch(w)\mright) = \cO(\frac{(\thresholdA +\sigma)^2}\bsize)$.  

If for a weight vector $\weight$ the variance $\Var_{\weight}(v^\batch(w))$ is much larger than our estimate of the variance of good batches then we can apply filtering for this cluster using $v^\batch(w)$. In this case, the requirement $\E_\weight[v^\batch(w)]  = \Theta(\|w-w^*\|)\pm\cO(\sigma)$ is irrelevant, as we will not be applying the multifiltering using gradients.

When for the weight vector $\weight$ the variance of $\Var_{\weight}(v^\batch(w))$ is bounded, and $\weight$ has a significant weight of good batches then using the fact that for most good batches $v^\batch(w) = \Theta(\E{}_{\dist}[v^\batch(w)])\pm\cO(\sigma)$, 
we show that $\E_\weight[v^\batch(w)]  = \Theta(\|w-w^*\|)\pm\cO(\sigma)$.

Hence, in the (estimation part) either we apply Multi-filter algorithm for the cluster to obtain new clusters with one of them being cleaner, or else our estimate of the parameters is in the desired range to apply the multi-filtering w.r.t. the gradients.

\section{Algorithm and Proof of Theorem~\ref{thm:mainresult}}\label{sec:alg}

We present our main algorithm as Algorithm~\ref{alg:main}. The algorithm starts with an initial weight vector $\weight_{init}$ which has equal weight on all the batches. It then iteratively refines these weights until it ends up with $O(1/\alpha^2)$ soft clusters. We show that at least one of these soft clusters will have a stationary point that satisfies the guarantees in Theorem~\ref{thm:mainresult}.

\begin{algorithm}[!th]
   \caption{\textsc{MainAlgorithm}}
   \label{alg:main}
\begin{algorithmic}[1]
   \STATE {\bfseries Input:} Data $\{\{(x_i^\batch,y_i^\batch)\}_{i\in[\bsize]}\}_{\batch\in\allbatches}$, $\goodfrac$, $C$, $\sigma$,  $\npower$, $\nhypcon$.
   \STATE{For each $\batch\in \allbatches$, $\weight_{init}^\batch\gets1$ and $\weight_{init}\gets\{\weight_{init}^\batch\}_{\batch\in\allbatches}$.}
   \STATE{List $L \gets \{\weight_{init}\}$ and $M\gets \emptyset$.}
   \WHILE{ $L\neq \emptyset$}
        \STATE{Pick any element $\weight$ in $L$ and remove it from $L$.}
      
    \STATE{$a_1 = \frac{256C\sqrt{2}}{3}$ and $a_2 = \frac{a_1}{4} + 2\max\{2(8\nhypcon C)^{1/\npower}, 2(\frac{8\nhypcon\sqrt{\bsize\goodfrac} }{\log(2/\goodfrac)})^{1/{(\npower-1)}}\}$}       
    \STATE{$\kappa, w\gets \textsc{FindClippingPparameter}(\allbatches$, $\weight$,  $a_1$, $a_2$  $\{\{(x_i^\batch,y_i^\batch)\}_{i\in[\bsize]}\}_{\batch\in\allbatches})$}
     \STATE{Find top approximate unit eigenvector $u$ of $\Cov_{\weight}(\clipbatchgrad)$.}
    \STATE{For each batch $\batch\in \allbatches$, let $v^\batch = \frac{1}{\bsize}\sum_{i\in[\bsize]}|w\cdot x_i^\batch-y_i^\batch|$ and  $\tilde v^\batch =\clipbatchgrad\cdot u$.}
     \STATE{$\thresholdA\gets \inf\{v: \weight(\{\batch:v^\batch\ge v\})\le \goodfrac\size{\allbatches}/4$ and}
    \begin{align}\label{eq:inalgthres}
                \textstyle\thresholdB\gets \frac{\ucb}{\bsize}\mleft(\sigma^2+ \mleft(\frac{8\sqrt C\thresholdA}7+\frac{\sigma}7\mright)^2\mright).
    \end{align}
   \STATE{
    \begin{align}\label{eq:inalgthresb}
    \textstyle\thresholdC\gets \frac{\uccb}{\bsize}\mleft(\sigma^2 + 16C^2 \mleft(\E{}_{\weight}[v^\batch] +\sigma\mright)^2\mright).
    \end{align}
    }
   \IF{$\text{Var}_{B,\weight}(v^\batch)> {\ucmf\log^2(2/\goodfrac)}\thresholdB$}
   \STATE $\textsc{NewWeights}\gets$\textsc{Multifilter}($\allbatches, \goodfrac, \weight, \{v^\batch\}_{\batch\in\allbatches}, \thresholdB$)  \COMMENT{\textbf{Type-1 use}}
   \STATE{Append each  weight vector $\widetilde \weight\in \textsc{NewWeights}$ that has total weight $\widetilde \weight^{\allbatches} \ge \goodfrac\size{\allbatches}/2$ to list $L$.} 
    \ELSIF{$\text{Var}_{B,\weight}(\tilde v^\batch)> {\ucmf\log^2(2/\goodfrac)}\thresholdC$}
   \STATE $\textsc{NewWeights}\gets$\textsc{Multifilter}($\allbatches, \goodfrac, \weight, \{\tilde v^\batch\}_{\batch\in\allbatches}, \thresholdC).$ \COMMENT{\textbf{Type-2 use}}
   \STATE{Append each  weight vector $\widetilde \weight\in \textsc{NewWeights}$ that has total weight $\widetilde \weight^{\allbatches} \ge \goodfrac\size{\allbatches}/2$ to list $L$.}
   \ELSE
   \STATE{Append $(\weight,\kappa,w)$ to $M$.}
    \ENDIF
   \ENDWHILE
\end{algorithmic}
\end{algorithm}

In the remainder of the section, we describe the algorithm and prove Theorem~\ref{thm:mainresult}. 
The proof will make use of the following regularity conditions and the results proved later in the appendix.

\paragraph{Regularity Conditions.} Our proof will make use of the following two regularity conditions.

For all unit vectors $u$, all vectors $w$ and for all $\kappa\le c_7\mleft( C^2\sqrt{\E{}_\dist[|x_i^\batch\cdot (w-w^*)|^2]}+ (\nhypcon C)^{1/\npower}\sigma+ (\frac{\nhypcon\sqrt{\bsize\goodfrac} }{\log(2/\goodfrac)})^{1/{(\npower-1)}}\sigma \mright)$, 
\begin{align*}
&\E{}_{\goodbatches}\mleft[\mleft(\clipbatchgrad \cdot u - \E{}_\dist[\clipbatchgrad\cdot u] \mright)^2\mright]\le \uccb\frac{\sigma^2+ C \E{}_{\dist}[((w-w^*)\cdot x_i^\batch)^2]}{\bsize},
\end{align*}
and for all vectors $w$, 
\begin{align*}
\textstyle &\E{}_{\goodbatches}\mleft[\mleft(\frac{1}{\bsize}\textstyle\sum_{i\in[\bsize]}|w\cdot x_i^\batch-y_i^\batch|-\E{}_{\dist}[|w\cdot x_i^\batch-y_i^\batch|]\mright)^2\mright]\le \ucb\mleft(\frac{\sigma^2+ C\E{}_{\dist}[|w\cdot x_i^\batch-y_i^\batch|]^2}{\bsize}\mright).
\end{align*}

In Section~\ref{sec:regcon}, we show that the two regularity condition holds with high probability when the number of good batches $\sizegood=\tilde O_{\bsize,\goodfrac}(d)$.


As a simple consequence, the first regularity condition implies that
\begin{equation}\|\mathrm{Cov}_{\goodbatches}(\clipbatchgrad) \|\le \uccb\frac{\sigma^2+ C \E{}_{\dist}[((w-w^*)\cdot x_i^\batch)^2]}{\bsize}, \label{eq:covbound} \end{equation}
and similarly, the second regularity condition implies that
\begin{equation}\label{eq:covbound2}
\textstyle \Var_{\goodbatches}\mleft(\frac{1}{\bsize}\textstyle\sum_{i\in[\bsize]}|w\cdot x_i^\batch-y_i^\batch|\mright)\le \ucb\mleft(\frac{\sigma^2+ C\E{}_{\dist}[|w\cdot x_i^\batch-y_i^\batch|]^2}{\bsize}\mright).
\end{equation}

\begin{proof}[Proof of Theorem~\ref{thm:mainresult}]
Algorithm~\ref{alg:main} starts with $L = \{\weight_{init}\}$. 
The initial weight vector $\weight_{init}$ assigns an equal weight 1 to each batch in $B$. The total weight of all batches and all good batches in the initial weight vector is $|\allbatches|$ and $\sizegood$, respectively.

While weight vectors correspond to a soft cluster of batches, for simplicity, in the first read one may think of it as just denoting a cluster of batches or a subset of $B$.

Any weight vector $\weight$ is a \emph{nice} weight vector if the total weight assigned to all good batches by it is at least $\weight^{\goodbatches}\ge 3\sizegood/4$. Note that the starting weight vector $\weight_{init}$ in $L$ is nice. 

The algorithm produces a list $M$ of at-most $4/\goodfrac^2$ triplets of weight vectors $\weight$, clipping parameter $\kappa$ and estimate $w$ for regression parameter, such that at least one of the triplets $(\weight,\kappa, w)\in M$ satisfies the following four conditions.
\begin{condition}\label{cond:all}
$(\weight,\kappa, w)$ satisfies
\begin{enumerate}[(a)]
    \item $\weight$ is a nice weight vector, i.e. $\weight^{\goodbatches}\ge 3\sizegood/4$. \label{cond:nice}
    \item $ \max\{8 \sqrt{C\E{}_\dist[|x_i^\batch\cdot (w-w^*)|^2]}, 2(8\nhypcon C)^{1/\npower}\sigma, 2(\frac{8\nhypcon\sqrt{\bsize\goodfrac} }{\log(2/\goodfrac)})^{1/{(\npower-1)}}\sigma \} \le \kappa$ and \\
    $\kappa\le c_7\mleft( C^2\sqrt{\E{}_\dist[|x_i^\batch\cdot (w-w^*)|^2]}+ (\nhypcon C)^{1/\npower}\sigma+ (\frac{\nhypcon\sqrt{\bsize\goodfrac} }{\log(2/\goodfrac)})^{1/{(\npower-1)}}\sigma \mright)$
    \label{cond:kappa}
    \item $w$ is any approximate stationary point w.r.t. $\weight$ for clipped loss with clipping parameter $\kappa$, namely $\|\E_{\weight}[\clipbatchgrad]\|\le \frac{\log(2/\goodfrac)\sigma}{8\sqrt{\bsize\goodfrac}}$. \label{cond:stat}
    \item $\|\Cov_{\weight}(\clipbatchgrad)\| \le \frac{\uccc C^2\log^2(2/\goodfrac)(\sigma^2 + \E{}_\dist[|(w-w^*)\cdot x_i^\batch |]^2)}{\bsize}$. \label{cond:cov}
\end{enumerate}
In the above conditions $c_7,\uccc>0$ are positive universal constants.
\end{condition}

Any triplet $(\weight,w,\kappa)$ that satisfy all of the above condition is a \emph{nice triplet}.
Theorem~\ref{th:tobenamed} shows that for any nice triplet $(\weight,w,\kappa)$, we have $\|w-w^*\|\le\cO(\frac{\connum C\sigma\log(2/\goodfrac)}{\sqrt{\bsize\goodfrac}})$.
To prove the theorem then it suffices to show that the algorithm returns a small list of triplets such that at least one of them is nice.


Algorithm~\ref{alg:main} starts with initial weight vector $\weight_{init}$ in $L$. For this weight vector $\weight_{init}^\goodbatches = \sizegood$, hence it is a nice weight vector.
In each while loop iteration, the algorithm picks one of the weight vectors $\weight$ from the list $L$ until the list $L$ is empty.
Then it uses subroutine \textsc{FindClippingPparameter} for this weight vector $\weight$, which returns $\kappa$ and $w$.
Theorem~\ref{th:clipparbound} in the Appendix~\ref{sec:clippar} shows that the parameters $w$ and $\kappa$ satisfy:
\begin{enumerate}
    \item $w$ is an approximate stationary point for $\{f^\batch(\cdot,\kappa)\}$ w.r.t. weight vector $\weight$.
    \item $ \max\mleft\{\frac{a_1}2 \E_{\weight}\mleft[\frac{1}{\bsize}\sum_{i\in[\bsize]}|w_\kappa\cdot x_i^\batch-y_i^\batch|\mright], a_2\sigma\mright\} \le \kappa \le  \max\mleft\{4 a_1^2 \E_{\weight}\mleft[\frac{1}{\bsize}\sum_{i\in[\bsize]}|w_\kappa\cdot x_i^\batch-y_i^\batch|\mright], a_2\sigma\mright\}.$ 
\end{enumerate}

After \textsc{FindClippingPparameter} sets $\kappa$ and $w$ for the weight vector $\weight$, we calculate parameters $\thresholdA$, $\thresholdB$ and $\thresholdC$.
For a batch $\batch\in \allbatches$, the algorithm defines $v^\batch := \frac{1}{\bsize}\sum_{i\in[\bsize]}|w\cdot x_i^\batch-y_i^\batch|$.

When $\text{Var}_{\weight}\mleft(v^\batch\mright)\ge \ucmf\log^2(2/\goodfrac) \thresholdB$ then \textsc{MainAlgorithm} uses sub-routine \textsc{Multifilter} (Algorithm~\ref{alg:basicmult} stated in Appendix~\ref{sec:mult}) on weight vector $\weight$ w.r.t. $v^\batch$.
This subroutine returns a list \textsc{NewWeights} of weight vectors as a result. 
We refer to application of \textsc{Multifilter} w.r.t. $v^\batch$ as \emph{Type-1 application}. \textsc{MainAlgorithm} appends weight vectors in  \textsc{NewWeights}  that have total weights more than $\goodfrac|\allbatches|/2$ to list $L$ and the iteration terminates.

Equation~\eqref{eq:covbound2} upper bounds $\text{Var}_{\goodbatches}\mleft(v^\batch\mright)$. Combining this upper bound with
Theorem~\ref{thm:contyp1} shows that if weight vector $\weight$ is nice then parameter $\theta_1$ calculated is such that
\begin{align*}
\text{Var}_{\goodbatches}\mleft(v^\batch\mright)\le \thresholdB.
\end{align*}

Type-1 application of \textsc{Multifilter} is only used on $\weight$, when we have
\begin{align*}
\text{Var}_{\weight}\mleft(v^\batch\mright)\ge \ucmf\log^2(2/\goodfrac) \theta_1.     
\end{align*}

This ensures that Type-1 application on a good weight $\weight$ only takes place when,
\begin{align}\label{eq:nicea}
\text{Var}_{\weight}\mleft(v^\batch\mright)\ge \ucmf\log^2(2/\goodfrac) \text{Var}_{\goodbatches}\mleft(v^\batch\mright).     
\end{align}

Let $\tilde v^\batch :=\clipbatchgrad\cdot u$, where $u$ is a top approximate unit eigen-vector of $\Cov_{\weight}(\clipbatchgrad)$ such that $\| \Cov_{\weight}(\clipbatchgrad)u\| \ge 0.5 \|\Cov_{\weight}(\clipbatchgrad)\|$.
We make use of $\tilde v^\batch$ when the iteration doesn't terminate by Type-1 filtering.

If the variance of $\Var_{\weight}(\tilde v^\batch)\ge {\ucmf\log^2(2/\goodfrac)}\thresholdC$ then  \textsc{MainAlgorithm} use sub-routine \textsc{Multifilter} (algorithm~\ref{alg:basicmult}) on weight vector $\weight$ w.r.t. $\tilde v^\batch$.
As before this subroutine returns a list \textsc{NewWeights} of weight vectors as a result.
As before \textsc{MainAlgorithm} appends weight vectors in \textsc{NewWeights} that have total weights more than $\goodfrac|\allbatches|/2$ to list $L$ and the iteration terminates.
We refer to application of \textsc{Multifilter} w.r.t. $\tilde v^\batch$ as \emph{Type-2 application}. 

On the other hand, when the variance of $\tilde v^\batch$ is small then the iteration terminates by appending $(\weight,\kappa,w)$ to $M$.

{Now we prove that $M$ has at least one nice triplet and size of $M$ is not too large.}
To prove that $M$ has at least one nice triplet we show that $M$ ends up with at least one triplet $(\weight,w,\kappa)$ in which $\weight$ satisfies condition~\ref{cond:nice}.
Then we show that for any triplet $(\weight,w,\kappa)$, which ends in $M$, if $\weight$ satisfy condition~\ref{cond:nice} then it satisfies the remaining conditions as well, and therefore is a nice triplet.

To prove this we show that whenever Type-2 use of multi-filter happens for a nice weight vector then
\begin{align}\label{eq:niceb}
\text{Var}_{\weight}\mleft(\tilde v^\batch\mright)\ge \ucmf\log^2(2/\goodfrac)\text{Var}_{\goodbatches}\mleft(\tilde v^\batch\mright).    
\end{align}

Theorem~\ref{th:temp} in Appendix~\ref{sec:mult} shows that as long as for all Type-1 use of this subroutine on nice weight vectors $\weight$ Equation~\eqref{eq:nicea} holds, and for all Type-2 use of this subroutine on nice weight vectors $\weight$, Equation~\eqref{eq:niceb} hold, then at least one of the triplet in the final list $M$ would contain some nice weight vector.
It also shows that the size of $M$ is at most $4/\goodfrac^2$ and at most $\cO(|\allbatches|/\goodfrac^2)$ calls to \textsc{Multifilter} are made. 
Since each iteration ends by a call to \textsc{Multifilter} or by adding a triplet to $M$, hence, the total number of iterations is at most $\cO(|\allbatches|/\goodfrac^2)$.

We have already shown Equation~\eqref{eq:nicea} holds for Type-1 use of \textsc{Multifilter}.

Type-1 use of \textsc{Multifilter} works for all values of $w$ and $\kappa$. But for  Type-2 use of \textsc{Multifilter} the condition $\text{Var}_{\weight}\mleft(v^\batch\mright)\le \ucmf\log^2(2/\goodfrac) \thresholdB$ will be crucial for correctly estimating $\kappa$ and $\thresholdC$. 
The role of Type-1 use is to make progress when this is not the case.

To conclude the proof of Theorem~\ref{thm:mainresult} we show Equation~\eqref{eq:niceb}, and any triplet in $M$ containing a nice weight vector must satisfy the remaining conditions for nice triplets.

The bound $\|\mathrm{Cov}_{\goodbatches}(\clipbatchgrad) \|\le \uccb\frac{\sigma^2+ C \E{}_{\dist}[((w-w^*)\cdot x_i^\batch)^2]}{\bsize}$ in~\eqref{eq:covbound} requires that $\kappa$ obeys the upper bound in condition~\ref{cond:kappa}. Hence, to show Equation~\eqref{eq:niceb} we must show that for Type-2 use of \textsc{Multifilter}, $\kappa$ obeys the upper bound in condition~\ref{cond:kappa} and $\thresholdC\ge  \uccb\frac{\sigma^2+ C \E{}_{\dist}[((w-w^*)\cdot x_i^\batch)^2]}{\bsize}$.
Also, when any triplet with a nice weight vector is added to $M$, we want to ensure
$\kappa$ obeys condition~\ref{cond:kappa} and $\thresholdC\le  \frac{\uccc C^2(\sigma^2 + \E{}_\dist[|(w-w^*)\cdot x_i^\batch |]^2)}{2\ucmf \bsize}$. Observe that a triplet $(\weight,w,\kappa)$ is added to $M$ only if $\text{Var}_{B,\weight}(\tilde v^\batch)\le  {\ucmf\log^2(2/\goodfrac)}\thresholdC$. From the definition of $\tilde v^\batch$ this would imply $ \|\Cov_{\weight}(\clipbatchgrad)\| \le {2\ucmf\log^2(2/\goodfrac)}\thresholdC$, hence the mentioned lower bound on $\thresholdC$ would ensure the the last condition~\ref{cond:cov} of nice triplet as well.

From the preceding discussion it follows that to prove~Theorem~\ref{thm:mainresult} it suffices to show that when $\text{Var}_{\weight}\mleft(v^\batch\mright)\le \ucmf\log^2(2/\goodfrac) \thresholdB$ and if weight vector $\weight$ is nice then $\kappa$ returned by subroutine \textsc{FindClippingPparameter} obeys condition~\ref{cond:kappa} and 
\[
\uccb\frac{\sigma^2+ C \E{}_{\dist}[((w-w^*)\cdot x_i^\batch)^2]}{\bsize}\le \thresholdC\le  \frac{\uccc}{2\ucmf}\frac{ C^2(\sigma^2 + \E{}_\dist[|(w-w^*)\cdot x_i^\batch |]^2)}{ \bsize}.
\]

We prove such a guarantee in Theorem~\ref{thm:bounds}. This concludes the proof of Theorem~\ref{thm:mainresult}

\end{proof}

\bibliographystyle{alpha-all}
\bibliography{biblio}
\clearpage
\appendix

\section{Regularity conditions}\label{sec:regcon}
In this section, we state regularity conditions for genuine data used in proving the guarantees of our algorithm.

\paragraph{Regularity Conditions.} 
\begin{enumerate}
    \item For all $\kappa\le c_7\mleft( C^2\sqrt{\E{}_\dist[|x_i^\batch\cdot (w-w^*)|^2]}+ (\nhypcon C)^{1/\npower}\sigma+ (\frac{\nhypcon\sqrt{\bsize\goodfrac} }{\log(2/\goodfrac)})^{1/{(\npower-1)}}\sigma \mright)$, all unit vectors $u$ and all vectors $w$
\begin{align*}
&\E{}_{\goodbatches}\mleft[\mleft(\clipbatchgrad \cdot u - \E{}_\dist[\clipbatchgrad\cdot u] \mright)^2\mright]\le \uccb\frac{\sigma^2+ C \E{}_{\dist}[((w-w^*)\cdot x_i^\batch)^2]}{\bsize},
\end{align*}
\item 
For all vectors $w$, 
\begin{align*}
 &\E{}_{\goodbatches}\mleft[\mleft(\frac{1}{\bsize}\sum_{i\in[\bsize]}|w\cdot x_i^\batch-y_i^\batch|-\E{}_{\dist}[|w\cdot x_i^\batch-y_i^\batch|]\mright)^2\mright]\le \ucb\mleft(\frac{\sigma^2+ C\E{}_{\dist}[|w\cdot x_i^\batch-y_i^\batch|]^2}{\bsize}\mright).
\end{align*}
\end{enumerate}

The first regularity condition on the set of good batches $\goodbatches$, bounds the mean squared deviation of projections of clipped batch gradients from its true population mean.
The regularity condition requires clipping parameter $\kappa$ to be upper bounded, with the upper bound depending on $\|w-w^*\|$ and $\sigma$.

As discussed in Section~\ref{sec:techover}, when $\kappa\to \infty$, the clipping has no effect, and establishing such regularity condition for unclipped gradients would require 
$\Omega(d^2)$ samples. 
By using clipping, and ensuring that clipping parameter $\kappa$ is in the desired range we are able to achieve $\tilde O_{\bsize,\goodfrac}(d)$ sample complexity.

Theorem~\ref{th:conofcov} characterizes the number of good batches required for regularity condition 1 as a function of the upper bound on $\kappa$.


\begin{theorem}\label{th:conofcov} 
There exist a universal constant $\uccb$ such that for $\mu_{\max}\in [1,\frac{d^4\bsize^2}{C}]$ and $\sizegood =\Omega(\mu_{\max}^4\bsize^2 d\log(d))$, with probability  $\ge 1-\frac{4}{d^2}$, for all unit vectors $u$, all vectors $w$ and for all $\kappa^2\le \mu_{\max}(\sigma^2+ C \E{}_{\dist}[((w-w^*)\cdot x_i^\batch)^2])$, 
\begin{align}\label{eq:eqform}
&\E{}_{\goodbatches}\mleft[\mleft(\clipbatchgrad \cdot u - \E{}_\dist[\clipbatchgrad\cdot u] \mright)^2\mright]\le \uccb\frac{\sigma^2+ C \E{}_{\dist}[((w-w^*)\cdot x_i^\batch)^2]}{\bsize}.
\end{align}
\end{theorem}
We prove the above theorem in Section~\ref{sec:proofcov}.

The second regularity condition on the set of good batches $G$, bounds the mean squared deviation of average absolute error for a batch from its true population mean.
Theorem~\ref{th:typerrorvar} characterizes the number of good batches required for regularity condition 2.

\begin{theorem}\label{th:typerrorvar}
For $\sizegood =\Omega(\bsize^2 d\log(d))$ and universal constant $\ucb>0$, with probability  $\ge 1-\frac{4}{d^2}$, for all vectors $w$, 
\begin{align*}
 &\E{}_{\goodbatches}\mleft[\mleft(\frac{1}{\bsize}\sum_{i\in[\bsize]}|w\cdot x_i^\batch-y_i^\batch|-\E{}_{\dist}[|w\cdot x_i^\batch-y_i^\batch|]\mright)^2\mright]\le \ucb\mleft(\frac{\sigma^2+ C\E{}_{\dist}[|w\cdot x_i^\batch-y_i^\batch|^2]}{\bsize}\mright).
\end{align*}
\end{theorem}
\begin{proof}
Proof of the above theorem is similar to the proof of Theorem~\ref{th:conofcov}, and for brevity, we skip it.
\end{proof}

Combining the two theorems shows that the two regularity conditions hold with high probability with $\tilde O_{\bsize,\goodfrac}(d)$ batches.

\begin{corollary}
For $\sizegood\ge \Omega_C\mleft(d\bsize^2 \log(d)  \mleft(\frac{\nhypcon\sqrt{\bsize\goodfrac} }{\log(2/\goodfrac)}\mright)^{\frac8{(\npower-1)}}  \mright)$, both regularity conditions hold with probability $\ge 1-\frac{8}{d^2}$. 
\end{corollary}

We conclude the sections with the following Lemma which lists some simple consequences of regularity conditions, that we use in later sections.
\begin{lemma}\label{lem:conseq}
If regularity conditions hold then
\begin{enumerate}
    \item  For all vectors $w$ and for all $\kappa\le c_7\mleft( C^2\sqrt{\E{}_\dist[|x_i^\batch\cdot (w-w^*)|^2]}+ (\nhypcon C)^{1/\npower}\sigma+ (\frac{\nhypcon\sqrt{\bsize\goodfrac} }{\log(2/\goodfrac)})^{1/{(\npower-1)}}\sigma \mright)$, 
    \[\|\mathrm{Cov}_{\goodbatches}(\clipbatchgrad) \|\le \uccb\frac{\sigma^2+ C \E{}_{\dist}[((w-w^*)\cdot x_i^\batch)^2]}{\bsize},\]
\item For all vectors $w$
\begin{align*}
 &\Var_{\goodbatches}\mleft(\frac{1}{\bsize}\sum_{i\in[\bsize]}|w\cdot x_i^\batch-y_i^\batch|\mright)\le \ucb\mleft(\frac{\sigma^2+ C\E{}_{\dist}[|w\cdot x_i^\batch-y_i^\batch|]^2}{\bsize}\mright).
\end{align*}
\item
For all $\Gsc\subseteq \goodbatches$ of size $\ge \sizegood/2$, 
\[
\|\E{}_{\Gsc}[\clipbatchgrad] - \E{}_\dist[\clipbatchgrad]\|\le \sqrt{2\uccb}\frac{\sigma+\sqrt{C \E{}_{\dist}[((w-w^*)\cdot x_i^\batch)^2]}}{\sqrt{\bsize}}.\]
\end{enumerate}
\end{lemma}
\begin{proof}
The first item in the lemma follows as
\begin{align*}
\|\Cov_{\goodbatches}(\clipbatchgrad) \| &= \max_{u:\|u\|\le 1} \E{}_{\goodbatches}\mleft[\mleft(\clipbatchgrad \cdot u - \E{}_{\goodbatches}[\clipbatchgrad\cdot u] \mright)^2\mright]\\
&\le \max_{u:\|u\|\le 1}\E{}_{\goodbatches}\mleft[\mleft(\clipbatchgrad \cdot u - \E{}_\dist[\clipbatchgrad\cdot u] \mright)^2\mright]\\
&\le \uccb\frac{\sigma^2+ C \E{}_{\dist}[((w-w^*)\cdot x_i^\batch)^2]}{\bsize},
\end{align*}
where the first inequality follows as the expected squared deviation along the mean is the smallest and the second inequality follows from the first regularity condition.

Similarly, the second item follows from the second regularity condition.

Finally, we prove the last item using the first regularity condition.
Let $u$ be any unit vector and $Z^\batch(u) :=\mleft(\clipbatchgrad \cdot u - \E{}_\dist[\clipbatchgrad\cdot u] \mright)^2$. Then
\[
\|\E{}_{\goodbatches}[Z^\batch](u)\| = \mleft\|\frac{1}{\sizegood}\sum_{\batch\in\goodbatches} Z^\batch(u) \mright\|\ge \mleft\|\frac{1}{\sizegood}\sum_{\batch\in\Bsc} Z^\batch (u)\mright\| = \frac{\size{\Gsc}}{\sizegood} \|\E{}_{\Gsc}[Z^\batch(u)]\|\ge \frac{1}{2} \|\E{}_{\Gsc}[Z^\batch(u)]\|,
\]
where the first inequality used the fact that $Z^\batch(u)$ is a positive and the second inequality used ${\size{\Gsc}}\ge {\sizegood}/2$. Then using the bound on $\|\E{}_{\goodbatches}[Z^\batch(u)]\|$ in in the first regularity condition, we get
\[
\|\E{}_{\Gsc}[Z^\batch(u)]\|\le 2\uccb\frac{\sigma^2+ C \E{}_{\dist}[((w-w^*)\cdot x_i^\batch)^2]}{\bsize}.
\]
Using the Cauchy–Schwarz inequality and the above bound,
\begin{align*}
\E{}_{\Gsc}[\mleft|\clipbatchgrad \cdot u - \E{}_\dist[\clipbatchgrad\cdot u]\mright|]= \E{}_{\Gsc}[\sqrt{Z^\batch(u)}]&\le \sqrt{\E{}_{\Gsc}[Z^\batch(u)}]\le  \sqrt{2\uccb\frac{\sigma^2+ C \E{}_{\dist}[((w-w^*)\cdot x_i^\batch)^2]}{\bsize}}.
\end{align*}
Since the above bound holds for each unit vector $u$, we have
\begin{align*}
\E{}_{\Gsc}[\mleft|\clipbatchgrad - \E{}_\dist[\clipbatchgrad]\mright|]\le  \sqrt{2\uccb\frac{\sigma^2+ C \E{}_{\dist}[((w-w^*)\cdot x_i^\batch)^2]}{\bsize}}\le \sqrt{2\uccb}\frac{\sigma+\sqrt{C \E{}_{\dist}[((w-w^*)\cdot x_i^\batch)^2]}}{\sqrt{\bsize}}.
\end{align*}
\end{proof}

\section{Guarantees for nice triplet}\label{sec:finalgua}

For completeness, we first restate the conditions a nice triplet $(\weight,\kappa,w)$ satisfy.

A triplet $(\weight,\kappa, w)$ is \emph{nice} if 
\begin{enumerate}[(a)]
    \item $\weight$ is a nice weight vector, i.e. $\weight^{\goodbatches}\ge 3\sizegood/4$.
    \item $ \max\{8 \sqrt{C\E{}_\dist[|x_i^\batch\cdot (w-w^*)|^2]}, 2(8\nhypcon C)^{1/\npower}\sigma, 2(\frac{8\nhypcon\sqrt{\bsize\goodfrac} }{\log(2/\goodfrac)})^{1/{(\npower-1)}}\sigma \} \le \kappa$ and \\
    $\kappa\le c_7\mleft( C^2\sqrt{\E{}_\dist[|x_i^\batch\cdot (w-w^*)|^2]}+ (\nhypcon C)^{1/\npower}\sigma+ (\frac{\nhypcon\sqrt{\bsize\goodfrac} }{\log(2/\goodfrac)})^{1/{(\npower-1)}}\sigma \mright)$
    \item $w$ is any approximate stationary point w.r.t. $\weight$ for clipped loss with clipping parameter $\kappa$, namely $\|\E_{\weight}[\clipbatchgrad]\|\le \frac{\log(2/\goodfrac)\sigma}{8\sqrt{\bsize\goodfrac}}$. 
    \item $\|\Cov_{\weight}(\clipbatchgrad)\| \le \frac{\uccc C^2\log^2(2/\goodfrac)(\sigma^2 + \E{}_\dist[|(w-w^*)\cdot x_i^\batch |]^2)}{\bsize}$. 
\end{enumerate}

In this section, we establish the following guarantees for any nice triplets.
In doing so we assume regularity conditions hold for $\goodbatches$.

\begin{theorem}\label{th:tobenamed}
Suppose {($\weight$, $\kappa$, $w$) is a  nice triplet}, $\bsize\ge \max\{128\uccb C \connum,\frac{256}{\goodfrac}\uccc C^2\connum^2\log^2(2/\goodfrac)\}$ and regularity conditions holds,
then $\|w-w^*\|\le\cO(\frac{\connum C\sigma\log(2/\goodfrac)}{\sqrt{\bsize\goodfrac}}) $.
\end{theorem}

In the remainder of this section, we prove the theorem. First, we provide an overview of the proof and state some auxiliary lemma that we use to prove the theorem.

In this section, we show that for any nice triplet $(\weight,\kappa,w)$ if $\|w-w^*\| = \tilde \Omega(\sigma/\sqrt{n\goodfrac})$ then the following lower bound on clipped gradient co-variance, $\|\Cov_{\weight}(\clipbatchgrad)\|\ge \Omega(\goodfrac\|w-w^*\|^2)$ holds.
For $n = \tilde\Omega(\frac{1}{\goodfrac})$ and $\|w-w^*\| = \tilde \Omega(\sigma/\sqrt{n\goodfrac})$ this lower bound contradicts the upper bound in condition~\ref{cond:cov}. Hence, the theorem concludes that $\|w-w^*\| = \tilde \cO(\sigma/\sqrt{n\goodfrac})$.

To show the lower bound $\|\Cov_{\weight}(\clipbatchgrad)\|\ge \Omega(\goodfrac\|w-w^*\|^2)$, we first show $\|\E{}_\dist[\clipsampgrad]\| =\Omega(\|w-w^*\|)-\tilde\cO(\sigma/\sqrt{\bsize\goodfrac})$ in Theorem~\ref{th:clipmean}. Since $\|\E{}_\dist[\clipsampgrad]\|=\|\E{}_\dist[\clipbatchgrad]\|$, the same bound will hold for the norm of expectation of clipped batch gradients.

When clipping parameter $\kappa\to \infty$ then $\clipsampgrad=\sampgrad$ and for unclipped gradients, a straightforward calculation shows the desired lower bound $\|\E{}_\dist[\clipsampgrad]\| =\Omega(\|w-w^*\|)$.
However, if $\kappa$ is too small then clipping may introduce a large bias in the gradients and such a lower bound may no longer hold.

Yet, the lower bound on $\kappa$ in condition~\ref{cond:kappa} ensures that $\kappa$ is much larger than the typical error which is of the order $\|w-w^*\|+\sigma$. 
And when clipping parameter $\kappa$ is much larger than the typical error, it can be shown that with high probability clipped and unclipped gradients for a random sample from $\dist$ would be the same. 
The next theorem uses this observation and for the case when $\kappa$ satisfies the lower bound in condition~\ref{cond:kappa} it shows the desired lower bound on the norm of expectation of clipped gradient.

\begin{theorem}\label{th:clipmean}
If $\kappa \ge \max\{8 \sqrt{C\E{}_\dist[|x_i^\batch\cdot (w-w^*)|^2]}, 2(8\nhypcon C)^{1/\npower}\sigma, 2(\frac{8\nhypcon\sqrt{\bsize\goodfrac} }{\log(2/\goodfrac)})^{1/{(\npower-1)}}\sigma \}$, then
\[\mleft\|\E{}_\dist[\clipsampgrad]\mright\| \ge \ \frac{11 }{16\connum}\|w-w^*\|
  - \frac{\log(2/\goodfrac)\sigma}{8\sqrt{\bsize\goodfrac}}.
  \]
\end{theorem}
We prove the above theorem in subsection~\ref{sec:lempdis}

Since $\E{}_\dist[\clipsampgrad]= \E{}_\dist[\clipbatchgrad]$, the same bound holds for the clipped batch gradients.

Next, in Lemma~\ref{lem:meanhalfgoodbound} we show that for any sufficiently large collection $\Gsc\subseteq\goodbatches$ of the good batches $\|\E{}_{\Gsc}[\clipbatchgrad] \| \approx \|\E{}_\dist[\clipbatchgrad]\|$.

\begin{lemma}\label{lem:meanhalfgoodbound}
Suppose {$\kappa$ and $w$ are part of a nice triplet}, $\bsize\ge 128\uccb C \connum$ and regularity conditions holds, then for all $\Gsc\subseteq \goodbatches$ of size $\ge \sizegood/2$, 
\[
\mleft\|\E{}_{\Gsc}[\clipbatchgrad]\mright\| \ge  
  \frac{1}{2\connum}\|w-w^*\| - \frac{\log(2/\goodfrac)\sigma}{8\sqrt{\bsize\goodfrac}}-\frac{\sqrt{2\uccb}\sigma}{\sqrt{\bsize}}.
 \]
\end{lemma}
\begin{proof}
From item 3 in Lemma~\ref{lem:conseq}, 
\begin{align*}
\|\E{}_{\Gsc}[\clipbatchgrad] - \E{}_\dist[\clipbatchgrad]\|&\le \sqrt{2\uccb}\cdot \frac{\sigma+\sqrt{C \E{}_{\dist}[((w-w^*)\cdot x_i^\batch)^2]}}{\sqrt{\bsize}}\\
&\le \frac{\sqrt{2\uccb}\sigma}{\sqrt{\bsize}}+ \frac{\sqrt{2\uccb C \|w-w^*\|^2\|\Sigma\|}}{\sqrt{\bsize}}\\
&\le \frac{\sqrt{2\uccb}\sigma}{\sqrt{\bsize}}+ \|w-w^*\|\cdot\frac{\sqrt{2\uccb C }}{\sqrt{\bsize}}.    
\end{align*}
Using $\bsize\ge 128\uccb C\connum^2$,
\[
\|\E{}_{\Gsc}[\clipbatchgrad] - \E{}_\dist[\clipbatchgrad]\|\le
\frac{1}{8\connum}\|w-w^*\|+ \frac{\sqrt{2\uccb}\sigma}{\sqrt{\bsize}}. \]

From Theorem~\ref{th:clipmean}, and the observation $\E{}_\dist[\clipsampgrad]= \E{}_\dist[\clipbatchgrad]$, we get
\[\mleft\|\E{}_\dist[\clipbatchgrad]\mright\| \ge  \frac{11 }{16\connum}\|w-w^*\|
  - \frac{\log(2/\goodfrac)\sigma}{8\sqrt{\bsize\goodfrac}}.
\]
The lemma follows by combining the above equation using 
triangle inequality.
\end{proof}

Next, the general bound on the co-variance will be useful in proving Theorem~\ref{th:tobenamed}.
\begin{lemma}\label{lem:yettoname}
For any weight vector $\weight$, any set of vectors $z^\batch$ associated with batches, and any sub-collection of vectors $\Bsc\subseteq \{\batch\in\allbatches:\weight^\batch\ge 1/2\}$,
\[
\Cov_{\weight}(z^\batch)
\ge \frac{\size{\Bsc}}{2\size{\allbatches}}\|\E{}_{\weight}[z^\batch]-\E{}_{\Bsc}[z^\batch]\|^2.
\]
\end{lemma}
The proof of the lemma appears in Section~\ref{sec:yettoname}.



In Theorem~\ref{th:tobenamed} we show that since $\weight^{\goodbatches} \ge 3/4\sizegood$, we can find a sub-collection $\Gsc$ of size $\sizegood/2$ such that for each $\batch\in \Gsc$, its weight  $\weight^\batch\ge 1/2$. 
The we use the previous results for $\Bsc = \Gsc$ and $z = \clipbatchgrad$ to get,
$\Cov_{\weight}(\clipbatchgrad)\ge \frac{\size{\Gsc}}{4\size{\allbatches}}\|\E_{\weight}[\clipbatchgrad]-\E_{\Gsc}[\clipbatchgrad]\|\ge \frac{\sizegood}{8\size{\allbatches}}\|\E_{\weight}[\clipbatchgrad]-\E_{\Gsc}[\clipbatchgrad]\|\ge \frac{\goodfrac}{8}\|\E_{\weight}[\clipbatchgrad]-\E_{\Gsc}[\clipbatchgrad]\|^2$.

From condition~\ref{cond:stat} of nice triplets we have  $\E_{\weight}[\clipbatchgrad]\approx 0$ and from Lemma~\ref{lem:yettoname} we have $\E_{\Gsc}[\clipbatchgrad]\gtrsim \|w-w^*\|$. Then we get an upper bound $\Cov_{\weight}(\clipbatchgrad)\gtrsim \goodfrac\cdot\|w-w^*\|^2$.

As discussed before, combining this lower bound with the  
upper bound in condition~\ref{cond:cov}, the theorem concludes $\|w-w^*\| = \tilde \cO(\sigma/\sqrt{n\goodfrac})$.
Next, we formally prove Theorem~\ref{th:tobenamed} using the above auxiliary lemmas and theorems.


\begin{proof}[Proof of Theorem~\ref{th:tobenamed}]

Let $\Gsc := \{\batch\in\goodbatches:\weight^\batch\ge 1/2\}$.
Next, we show that $|\Gsc|\ge \sizegood/2$. To prove it by contradiction assume the contrary that $|\Gsc|< \sizegood/2$. 
Then
\[
\weight^\goodbatches = \sum_{\batch\in \goodbatches}\weight^\batch = \sum_{\batch\in \goodbatches\setminus\Gsc}\weight^\batch +\sum_{\batch\in \Gsc}\weight^\batch \ineqlabel{a}\le \sum_{\batch\in \goodbatches\setminus\Gsc}\frac{1}{2}+ \sum_{\batch\in \Gsc} 1 \le \frac{|\goodbatches\setminus\Gsc|}{2}+|\Gsc|= \frac{|\goodbatches|-|\Gsc|}{2}+|\Gsc|< 3\sizegood/4,
\]
here (a) follows as the definition of $\Gsc$ implies that for any $\batch\notin\Gsc$, $\weight^\batch<1/2$ and for all batches $\weight^\batch\le 1$.
Above is a contradiction, as we assumed in the Theorem that $\weight^\goodbatches \ge 3\sizegood/4$.

Applying Lemma~\ref{lem:yettoname} for $\Bsc = \Gsc$ and $z^\batch = \clipbatchgrad$ we have
\begin{align}\label{eq:covlb}
\|\Cov_{\weight} ( \clipbatchgrad )\|
    &\ge\frac{\Gsc}{2\size{\allbatches}} \mleft( \|\E{}_{\Gsc}[\clipbatchgrad]\|-\|\E{}_{\weight} [ \clipbatchgrad]\|\mright)^2\nonumber\\
    &\ge \frac{\sizegood}{4\size{\allbatches}} \mleft( \|\E{}_{\Gsc}[\clipbatchgrad]\|-\|\E{}_{\weight} [ \clipbatchgrad]\|\mright)^2\nonumber\\
    &\ge \frac{\goodfrac}{4} \mleft( \|\E{}_{\Gsc}[\clipbatchgrad]\|-\|\E{}_{\weight} [ \clipbatchgrad]\|\mright)^2.
\end{align}
In the above equation, using the bound in Lemma~\ref{lem:meanhalfgoodbound} and bound on $\|\E{}_{\weight} [ \clipbatchgrad]\|$ in condition~\ref{cond:stat} for nice triplet we get,
\begin{align*}
\|\Cov_{\weight} ( \clipbatchgrad )\|
    &\ge \frac{\goodfrac}{4} \mleft(\max\mleft\{0, \frac{1}{2\connum}\|w-w^*\|- \frac{\log(2/\goodfrac)\sigma}{8\sqrt{\bsize\goodfrac}}-\frac{\sqrt{2\uccb}\sigma}{\sqrt{\bsize}}-\frac{\log(2/\goodfrac)\sigma}{8\sqrt{\bsize\goodfrac}}\mright\}\mright)^2.
\end{align*}
We show that when $\|w-w^*\|\le\cO(\frac{\connum C\sigma\log(2/\goodfrac)}{\sqrt{\bsize\goodfrac}}) $, the above upper bound contradicts the following lower bound in condition~\ref{cond:cov},
\[
 \|\Cov_{\weight}(\clipbatchgrad)\| \le \frac{\uccc C^2\log^2(2/\goodfrac)(\sigma^2 + \E{}_\dist[|(w-w^*)\cdot x_i^\batch |]^2)}{\bsize}\le \frac{\uccc C^2\log^2(2/\goodfrac)(\sigma^2 + \|w-w^*\|^2)}{\bsize}.
\]

To prove the contradiction assume
\begin{align*}
\frac{\|w-w^*\|}{8\connum} >  \max\mleft\{
\frac{\log(2/\goodfrac)\sigma}{4\sqrt{\bsize\goodfrac}},\frac{\sqrt{2\uccb}\sigma}{\sqrt{\bsize}}, \frac{2\sqrt{\uccc}\connum C\sigma\log(2/\goodfrac)}{\sqrt{\bsize\goodfrac}}\mright\}.    
\end{align*}
Using this lower bound on $\|w-w^*\|$, we lower bound the co-variance.
Combining the above lower bound on $\|w-w^*\|$ and equation~\eqref{eq:covlb}, we get,
\begin{align*}
  \|\Cov_{\weight} ( \clipbatchgrad )\|&\ge \frac{\goodfrac}{4} \mleft(  \frac{1}{4\connum}\|w-w^*\|\mright)^2\\
  &\ge \frac{\goodfrac}{4} \mleft(  \frac{2\sqrt{\uccc}\connum C\sigma\log(2/\goodfrac)}{\sqrt{\bsize\goodfrac}}+\frac{1}{8\connum}\|w-w^*\|\mright)^2\\
  &\ge \frac{\goodfrac}{4} \mleft(  \frac{2\sqrt{\uccc}\connum C\sigma\log(2/\goodfrac)}{\sqrt{\bsize\goodfrac}}\mright)^2+ \frac{\goodfrac}{4} \mleft(\frac{1}{8\connum}\|w-w^*\|\mright)^2\\
      &\ge \frac{\uccc C^2\log^2(2/\goodfrac)\sigma^2}{\bsize}+\frac{\goodfrac}{256}
  \frac{\|w-w^*\|^2}{\connum}\\
   &\ge \frac{\uccc C^2\log^2(2/\goodfrac)\sigma^2}{\bsize}+\frac{\uccc C^2\log^2(2/\goodfrac)\|w-w^*\|^2}{\bsize},
\end{align*}
here the last step used $\bsize\ge \frac{256}{\goodfrac}\uccc C^2\connum^2\log^2(2/\goodfrac)$.

This completes the proof of the contradiction. Hence,
\begin{align*}
\frac{\|w-w^*\|}{8\connum} \le  \max\mleft\{
\frac{\log(2/\goodfrac)\sigma}{4\sqrt{\bsize\goodfrac}},\frac{\sqrt{2\uccb}\sigma}{\sqrt{\bsize}}, \frac{2\sqrt{\uccc}\connum C\sigma\log(2/\goodfrac)}{\sqrt{\bsize\goodfrac}}\mright\}.    
\end{align*}
The above equation implies $\|w-w^*\|\le\cO(\frac{\connum C\sigma\log(2/\goodfrac)}{\sqrt{\bsize\goodfrac}}) $.
\end{proof}




\subsection{Proof of Theorem~\ref{th:clipmean}}\label{sec:lempdis}
\begin{proof}
When $w\neq w^*$ the bound holds trivially. Hence, in the remainder of the proof, we assume $w\neq w^*$.
Let $u \ed \frac{w-w^*}{\|w-w^*\|}$ and $Z_i^b\ed\mathbbm 1\mleft( (|x^\batch_i\cdot (w-w^*)|\ge \kappa/2)\cup (|n^\batch_i|\ge \kappa/2)\mright)$.
Then,
\[
\E{}_\dist[\clipsampgrad] \ge |\E{}_\dist[\clipsampgrad\cdot u]| \ge |\E{}_\dist[\sampgrad\cdot u]|- \E{}_\dist[|\sampgrad\cdot u-\clipsampgrad\cdot u|].
\]

Next,
\begin{align*}
|\sampgrad\cdot u-\clipsampgrad\cdot u|
&\ineqlabel{a}= Z_i^\batch|\sampgrad\cdot u-\clipsampgrad\cdot u|\\
&\ineqlabel{b}\le Z_i^\batch|\sampgrad\cdot u| \\
&= Z_i^\batch|(x^\batch_i\cdot (w-w^*)-n^\batch_i)x_i^\batch\cdot u|\\
&\ineqlabel{c}\le  Z_i^\batch \|w-w^*\| (x_i^\batch\cdot u)^2 + Z_i^\batch \cdot |n^\batch_i|\cdot|x_i^\batch\cdot u|,
\end{align*}
here (a) follows as $Z_i^\batch = 0$ implies $|x^\batch_i\cdot (w-w^*)|\le \kappa/2$ and $|n^\batch_i|\le \kappa/2$ and hence $ \sampgrad=\clipsampgrad$, (b) follows since  $\clipsampgrad\cdot u$ and $\sampgrad \cdot u$ has the same sign and $|\clipsampgrad\cdot u|\le |\sampgrad \cdot u|$ and (c) follows as  $u = \frac{w-w^*}{\|w-w^*\|}$ and the triangle inequality. 

From the above two equations,
\begin{align}
|\E{}_\dist[\clipsampgrad]|& \ge |\E{}_\dist[\sampgrad\cdot u]|- \E{}_\dist[|\sampgrad\cdot u-\clipsampgrad\cdot u|]\nonumber\\
 &\ge |\E{}_\dist[(x^\batch_i\cdot (w-w^*)-n^\batch_i)x_i^\batch\cdot u]|  -\|w-w^*\|\cdot\E{}_\dist[Z_i^\batch(x_i^\batch\cdot u)^2] - \E{}_\dist[Z_i^\batch\cdot|n^\batch_i|\cdot|x_i^\batch\cdot u|] \nonumber \\
 &\ineqlabel{a}=|\E{}_\dist[(x^\batch_i\cdot (w-w^*))x_i^\batch\cdot u]|  -\|w-w^*\|\cdot\E{}_\dist[Z_i^\batch(x_i^\batch\cdot u)^2] - \E{}_\dist[Z_i^\batch\cdot|n^\batch_i|\cdot|x_i^\batch\cdot u|]  \nonumber \\
 &\ge \|w-w^*\|\cdot\E{}_\dist[(x_i^\batch\cdot u)^2]  -\|w-w^*\|\cdot\E{}_\dist[Z_i^\batch(x_i^\batch\cdot u)^2] - \E{}_\dist[Z_i^\batch\cdot|n^\batch_i|\cdot|x_i^\batch\cdot u|] ,\label{eq:gradexpbound}
\end{align}
here (a) follows as $n_i^\batch$ is zero mean and independent of $x_i^\batch$.

From the definition of $Z_i^\batch$ it follows that 
$Z_i^\batch  \le  \mathbbm 1 (|x^\batch_i\cdot (w-w^*)|\ge \kappa/2)+ \mathbbm 1(|n^\batch_i|\ge \kappa/2)$ and $\E_\dist[Z_i^\batch]\le \Pr[|x^\batch_i\cdot (w-w^*)|\ge \kappa/2] + \Pr[|n^\batch_i|\ge \kappa/2]$.

Now we bound the second term on the right in equation~\eqref{eq:gradexpbound},
\begin{align}
\E{}_\dist[Z_i^\batch(x_i^\batch\cdot u)^2] 
&\ineqlabel{a}\le   \sqrt{\E[(Z_i^\batch)^2]
\cdot \E{}_\dist[(x_i^\batch\cdot u)^4]}\nonumber\\
&\ineqlabel{b}=\sqrt{\E[Z_i^\batch]
\cdot \E{}_\dist[(x_i^\batch\cdot u)^4]}\nonumber\\
&\ineqlabel{c}\le\sqrt{\E[Z_i^\batch]}\sqrt {C
 (\E{}_\dist[(x_i^\batch\cdot u)^2]^2)^2}\nonumber\\
&\ineqlabel{d}\le\sqrt{\Pr[|x^\batch_i\cdot (w-w^*)|\ge \kappa/2] + \Pr[|n^\batch_i|\ge \kappa/2]}\sqrt {C
 (\E{}_\dist[(x_i^\batch\cdot u)^2]^2)^2}\nonumber\\
 &\le\mleft(\sqrt{\Pr[|x^\batch_i\cdot (w-w^*)|\ge \kappa/2]} + \sqrt{\Pr[|n^\batch_i|\ge \kappa/2]}\mright)\sqrt {C}
 \E{}_\dist[(x_i^\batch\cdot u)^2],\label{eq:othereq}
\end{align}
where (a) used the Cauchy-Schwarz inequality, (b) used the fact that $Z_i^\batch$ is an indicator random variable, hence, $(Z_i^\batch)^2=Z_i^\batch$, and (c) uses $L4-L2$ hypercontractivity and (d) uses the upper bound on $\E[Z_i^\batch]$.

Next, we bound the last term on the right in equation~\eqref{eq:gradexpbound},
\begin{align}
&\E{}_\dist[Z_i^\batch\cdot|n^\batch_i|\cdot|x_i^\batch\cdot u|] \le \E{}_\dist\mleft[\mleft( \mathbbm 1(|n^\batch_i|\ge \kappa/2)+\mathbbm 1 (|x^\batch_i\cdot (w-w^*)|\ge \kappa/2)\mright)\cdot|n^\batch_i|\cdot|x_i^\batch\cdot u|\mright]\nonumber  \\ 
 &\le \E{}_\dist[|\mathbbm 1(|n^\batch_i|\ge \kappa/2)\cdot|n^\batch_i|\cdot|x_i^\batch\cdot u|] +\E{}_\dist[\mathbbm 1 (|x^\batch_i\cdot (w-w^*)|\ge \kappa/2)\cdot|n^\batch_i|\cdot|x_i^\batch\cdot u|]\nonumber\\
 &= \E{}_\dist[|x_i^\batch\cdot u|]\cdot\E{}_\dist[\mathbbm 1(|n^\batch_i|\ge \kappa/2)\cdot|n^\batch_i|]+\E{}_\dist[|n^\batch_i|]\cdot\E{}_\dist[\mathbbm 1 (|x^\batch_i\cdot (w-w^*)|\ge \kappa/2)\cdot|x_i^\batch\cdot u|] \nonumber\\
  &\le \sqrt{\E{}_\dist[|x_i^\batch\cdot u|^2]}\cdot\E{}_\dist[\mathbbm 1(|n^\batch_i|\ge \kappa/2)\cdot|n^\batch_i|]+\sqrt{\E{}_\dist[|n^\batch_i|^2}]\cdot \E{}_\dist[\mathbbm 1 (|x^\batch_i\cdot (w-w^*)|\ge \kappa/2)\cdot|x_i^\batch\cdot u|]\nonumber\\
   &\le \E{}_\dist[\mathbbm 1(|n^\batch_i|\ge \kappa/2)\cdot|n^\batch_i|]+\frac{\sigma}{\|w-w^*\|} \E{}_\dist[\mathbbm 1 (|x^\batch_i\cdot (w-w^*)|\ge \kappa/2)\cdot|x_i^\batch\cdot (w-w^*)|].\label{eq:tempmany}
\end{align}
Next, we bound the two expectations involving the indicator functions. For $p\ge 2$,
\begin{align*}
\E{}_\dist[\mathbbm 1(|n^\batch_i|\ge \kappa/2)\cdot|n^\batch_i|]& \le \E{}_\dist[\mathbbm 1(|n^\batch_i|\ge \kappa/2)^{\frac{p-1}{p}}]^\frac{p}{p-1} \E{}_\dist[|n^\batch_i|^{p}]^{1/p} \\
&\le \Pr[|n^\batch_i|\ge \kappa/2]^\frac{p-1}p \E{}_\dist[|n^\batch_i|^{p}]^{1/p}. 
\end{align*}.

Applying Markov inequality for $|n^\batch_i|^p$
\begin{align}
\Pr[ |n^\batch_i|\ge \kappa/2] \le \frac{\E{}_\dist[|n_i^\batch|^p]}{(\kappa/2)^p}.\label{eq:othereqa}    
\end{align}
Combining the two bounds we get,
\begin{align}
\E{}_\dist[\mathbbm 1(|n^\batch_i|\ge \kappa/2)\cdot|n^\batch_i|]& \le  \frac{\E{}_\dist[|n_i^\batch|^p]}{(\kappa/2)^{p-1}}. \label{eq:tempmanya}
\end{align}

Similarly, one can show,
\begin{align}
\Pr[ |x^\batch_i\cdot (w-w^*)|\ge \kappa/2] \le \frac{\E{}_\dist[|x_i^\batch\cdot (w-w^*)|^p]}{(\kappa/2)^p}\le \frac{\E{}_\dist[|x_i^\batch\cdot (w-w^*)|^4]}{(\kappa/2)^4}, \label{eq:othereqb}  
\end{align}
and
\begin{align}
\E{}_\dist[\mathbbm 1 (|x^\batch_i\cdot (w-w^*)|\ge \kappa/2)\cdot|x_i^\batch\cdot (w-w^*)|]
\ge\frac{\E{}_\dist[|x_i^\batch\cdot (w-w^*)|^p]}{(\kappa/2)^{p-1}}
\ge \frac{\E{}_\dist[|x_i^\batch\cdot (w-w^*)|^4]}{(\kappa/2)^3}. \label{eq:tempmanyb}   
\end{align}

Combining Equations~\eqref{eq:tempmany},~\eqref{eq:tempmanya} and~\eqref{eq:tempmanyb},
\begin{align}
\E{}_\dist[Z_i^\batch\cdot|n^\batch_i|\cdot|x_i^\batch\cdot u|] &\le    \frac{\E{}_\dist[|n_i^\batch|^p]}{(\kappa/2)^{p-1}}+\frac{\sigma}{\|w-w^*\|} \frac{\E{}_\dist[|x_i^\batch\cdot (w-w^*)|^4]}{(\kappa/2)^3}\nonumber\\
&\ineqlabel{a}\le\frac{\nhypcon\sigma^{\npower}}{(\kappa/2)^{p-1}}+\frac{\sigma}{\|w-w^*\|} \frac{C\E{}_\dist[|x_i^\batch\cdot (w-w^*)|^2]^2}{(\kappa/2)^3}\nonumber\\
&\ineqlabel{b}\le\frac{\nhypcon\sigma^{\npower}}{(\kappa/2)^{p-1}}+\frac{1}{\|w-w^*\|} \frac{C\E{}_\dist[|x_i^\batch\cdot (w-w^*)|^2]^2}{(\kappa/2)^2}\nonumber\\
&\ineqlabel{b}=\frac{\nhypcon\sigma^{\npower}}{(\kappa/2)^{p-1}}+{\|w-w^*\|} \frac{C(\E{}_\dist[|x_i^\batch\cdot u|^2])(\E{}_\dist[|x_i^\batch\cdot (w-w^*)|^2])}{(\kappa/2)^2}\nonumber\\
&\ineqlabel{c}\le\frac{\log(2/\goodfrac)\sigma}{8\sqrt{\bsize\goodfrac}}+{\|w-w^*\|} \frac{\E{}_\dist[|x_i^\batch\cdot u|^2]}{16}\nonumber,
\end{align}
here (a) use hypercontractivity of $x_i^\batch$ and $n_i^\batch$, (b) uses $\kappa\ge 2(8\nhypcon C)^{1/\npower}\sigma\ge  2\sigma$ as hypercontractivity constants $\nhypcon,C\ge 1$, and (c) uses $\kappa\ge  2(\frac{8\nhypcon\sqrt{\bsize\goodfrac} }{\log(2/\goodfrac)})^{1/{(\npower-1)}}\sigma $ and $\kappa\ge 8 \sqrt{C\E{}_\dist[|x_i^\batch\cdot (w-w^*)|^2]}$.

Combining Equations~\eqref{eq:othereq},~\eqref{eq:othereqa} and~\eqref{eq:othereqb},
\begin{align}
\E{}_\dist[Z_i^\batch(x_i^\batch\cdot u)^2]  &\le\mleft(\sqrt{\frac{\E{}_\dist[|x_i^\batch\cdot (w-w^*)|^4]}{(\kappa/2)^4}} + \sqrt{\frac{\E{}_\dist[|n_i^\batch|^p]}{(\kappa/2)^p}}\mright)\sqrt {C}
 \E{}_\dist[(x_i^\batch\cdot u)^2]  \nonumber\\  
& \ineqlabel{a}\le\mleft(\sqrt{\frac{C\E{}_\dist[|x_i^\batch\cdot (w-w^*)|^2]^2}{(\kappa/2)^4}} + \sqrt{\frac{\nhypcon\sigma^{\npower}}{(\kappa/2)^p}}\mright)\sqrt {C}
 \E{}_\dist[(x_i^\batch\cdot u)^2]  \nonumber\\
 & \ineqlabel{b}\le
 \frac{\E{}_\dist[(x_i^\batch\cdot u)^2] }4, \nonumber  
\end{align}
here (a) use hypercontractivity of $x_i^\batch$ and $n_i^\batch$ and $\E[\|x_i^\batch\cdot u\|^2]\le \|\Sigma\|= 1$ and (b) uses $\kappa \ge \sqrt{C\E{}_\dist[|x_i^\batch\cdot (w-w^*)|^2]}$ and $\kappa \ge 2(8\nhypcon C)^{1/\npower}\sigma$.

Combining the above two bounds and Equation~\eqref{eq:gradexpbound}, 
\begin{align*}
  |\E{}_\dist[\clipsampgrad]| &\ge \frac{11}{16}
  \|w-w^*\|\cdot\E{}_\dist[(x_i^\batch\cdot u)^2]  - \frac{\log(2/\goodfrac)\sigma}{8\sqrt{\bsize\goodfrac}}\\
  &\ge \frac{11}{16}
  \|w-w^*\|\cdot\frac{\|\Sigma\|}{\connum}  - \frac{\log(2/\goodfrac)\sigma}{8\sqrt{\bsize\goodfrac}}\\
  &\ge \frac{11 \|w-w^*\|}{16\connum}
  - \frac{\log(2/\goodfrac)\sigma}{8\sqrt{\bsize\goodfrac}},
\end{align*}
where in the last step use the bound on condition number of $\Sigma$ and the assumption that {$\|\Sigma\| = 1$}.
\end{proof}

\subsection{Proof of Lemma~\ref{lem:yettoname}}\label{sec:yettoname}
\begin{proof}
Note that
\begin{align*}
\|\Cov_{\weight} ( z^\batch )\| &=\mleft\|\sum_{\batch\in \allbatches}\frac{\weight^\batch}{\weight^{\allbatches}} ( z^\batch-\E{}_{\weight} [ z^\batch])( z^\batch-\E{}_{\weight} [ z^\batch])^\intercal\mright\|\\
 &\ge\mleft\|\sum_{\batch\in \Bsc}\frac{\weight^\batch}{\weight^{\allbatches}} ( z^\batch-\E{}_{\weight} [ z^\batch])( z^\batch-\E{}_{\weight} [ z^\batch])^\intercal\mright\|\\
 &\ineqlabel{a}\ge\mleft\|\sum_{\batch\in \Bsc}\frac{1}{2\size{\allbatches}} ( z^\batch-\E{}_{\weight} [ z^\batch])( z^\batch-\E{}_{\weight} [ z^\batch])^\intercal\mright\|\\
  &\ineqlabel{b}\ge\frac{1}{2\size{\allbatches}}\mleft\|\size{\Bsc} ( \E{}_{\Bsc}[z^\batch]-\E{}_{\weight} [ z^\batch])( \E{}_{\Bsc}[]z^\batch]-\E{}_{\weight} [ z^\batch])^\intercal\mright\|\\
& =\frac{\size{\Bsc}}{2\size{\allbatches}}\|\E{}_{\weight}[z^\batch]-\E{}_{\Bsc}[z^\batch]\|^2
  ,
\end{align*}
where (a) used $\weight^\batch\ge 1/2$ for $\batch\in \Bsc$ and the trivial bound $\weight^\allbatches\le\size{\allbatches}$ and (b) follows from the fact that
any $Z$,
\begin{align*}
\mleft\|\sum_{\batch\in \Bsc} (z^\batch- Z)(z^\batch- Z)^\intercal\mright\| \ge |\Bsc|\cdot\mleft\|\mleft(\E{}_{\Bsc}[z^\batch]-Z\mright)\mleft(\E{}_{\Bsc}[z^\batch]-Z\mright)^\intercal\mright\|.    
\end{align*}
We complete the proof of the lemma by proving the above fact.
\begin{align*}
&\mleft\|\sum_{\batch\in \Bsc} (z^\batch- Z)(z^\batch- Z)^\intercal\mright\| 
= \mleft\|\sum_{\batch\in \Bsc} (z^\batch-\E{}_{\Bsc}[z^\batch]+\E{}_{\Bsc}[z^\batch]- Z)(z^\batch-\E{}_{\Bsc}[z^\batch]+\E{}_{\Bsc}[z^\batch]- Z)^\intercal\mright\|\\
&\ineqlabel{a}= \mleft\|\sum_{\batch\in \Bsc}\mleft( (z^\batch-\E{}_{\Bsc}[z^\batch])(z^\batch-\E{}_{\Bsc}[z^\batch])^\intercal +(\E{}_{\Bsc}[z^\batch]- Z)(\E{}_{\Bsc}[z^\batch]- Z)^\intercal\mright) \mright\|\\
&\ineqlabel{b}\ge \size{\Bsc}\cdot\mleft\|(\E{}_{\Bsc}[z^\batch]- Z)(\E{}_{\Bsc}[z^\batch]- Z)^\intercal \mright\|,
\end{align*}
here (a) follows as $ \sum_{\batch\in \Bsc} z^\batch = |\Bsc|\E{}_{\Bsc}[z^\batch]$ and hence, $ \sum_{\batch\in \Bsc} (z^\batch-\E{}_{\Bsc}[z^\batch])(\E{}_{\Bsc}[z^\batch]- Z)^\intercal = \sum_{\batch\in \Bsc} (\E{}_{\Bsc}[z^\batch]- Z)(z^\batch-\E{}_{\Bsc}[z^\batch])^\intercal = 0$, and (b) follows as $ (z^\batch-\E{}_{\Bsc}[z^\batch])(z^\batch-\E{}_{\Bsc}[z^\batch])^\intercal$ are positive semi-definite matrices.

\end{proof}

\section{Subroutine \textsc{FindClippingParameter} and its analysis}\label{sec:clippar}

\begin{algorithm}[!th]
   \caption{\textsc{FindClippingPparameter}}
   \label{alg:clip}
\begin{algorithmic}[1]
   \STATE {\bfseries Input:} Set $\allbatches$, $\weight$, $\sigma$, $a_1\ge 1$, $a_2$ data $\{\{(x_i^\batch,y_i^\batch\}_{i\in[\bsize]}\}_{\batch\in\allbatches}.$
        \STATE{$\kappa\gets \infty$}
        \WHILE{True}
        \STATE{$w_\kappa\gets$ any approximate stationary point of clipped losses $\{ f^\batch(\,\cdot\,,\kappa)\}$ w.r.t. weight vector $\weight$ such that $\|\E_{\weight}[f^\batch(w_\kappa,\kappa)]\|\le \frac{\log(2/\goodfrac)\sigma}{8\sqrt{\bsize\goodfrac}}$}
        \STATE{\begin{align}\label{eq:setkappa}
   \textstyle\kappa_{new}\gets \max\mleft\{a_1 \sqrt{\E{}_{\weight}[f^\batch(w_\kappa,\kappa)]} , a_2\sigma \mright\}.
   \end{align}}
        \IF{$\kappa_{new}\ge \kappa/2$ }
        \STATE{Break}
        \ENDIF
        \STATE{$\kappa\gets \kappa_{new}$}
        \ENDWHILE
    \STATE{Return($\kappa,w_\kappa$)}
\end{algorithmic}
\end{algorithm}

\begin{theorem}\label{th:clipparbound}
For any weight vector $\weight$, $a_1\ge 1$, and $a_2>0$,  Algorithm~\textsc{FindClippingParameter} runs at most 
$\log \mleft(\cO\mleft(\frac{\max_{i,\batch} |y_i^\batch|}\sigma\mright)\mright)$ iterations of the while loop and returns $\kappa$ and $w_\kappa$ such that
\begin{enumerate}
    \item $w_\kappa$ is a (approximate) stationary point for $\{f^\batch(\cdot,\kappa)\}$ w.r.t. weight vector $\weight$ such that \[
   \textstyle \|\E{}_{\weight}[f^\batch(w_\kappa,\kappa)]\|\le \frac{\log(2/\goodfrac)\sigma}{8\sqrt{\bsize\goodfrac}}.\]
    \item  $ \max\mleft\{a_1 \sqrt{\E_{\weight}[f^\batch(w_\kappa,\kappa)]}, a_2\sigma\mright\} \le \kappa \le  2\max\mleft\{a_1 \sqrt{\E_{\weight}[f^\batch(w_\kappa,\kappa)]}, a_2\sigma\mright\}.$
    \item  $ \max\mleft\{\frac{a_1}2 \E_{\weight}\mleft[\frac{1}{\bsize}\sum_{i\in[\bsize]}|w_\kappa\cdot x_i^\batch-y_i^\batch|\mright], a_2\sigma\mright\} \le \kappa \le  \max\mleft\{4 a_1^2 \E_{\weight}\mleft[\frac{1}{\bsize}\sum_{i\in[\bsize]}|w_\kappa\cdot x_i^\batch-y_i^\batch|\mright], a_2\sigma\mright\}.$ 
\end{enumerate}
\end{theorem}
\begin{proof}

First, we bound the number of iterations of the while loop.
Since $w_{\kappa}$ is a stationary point for $f^\batch(.,\kappa)$, hence its will achieve a smaller loss than $w=0$, hence $\E_{\weight}[f^\batch(w_{\kappa},\kappa)]\le \E_{\weight}[f^\batch(0,\kappa)]$.
And, since the clipped loss is smaller than unclipped loss,  $\E_{\weight}[f^\batch(0,\kappa)] \le \E_{\weight}[f^\batch(0)] = \E_{\weight}[\frac{1}{\bsize}\sum_{i\in [\bsize]}(y_i^\batch)^2]\le \max_{i,\batch} (y_i^\batch)^2$.
Therefore after the first iteration $\kappa\le  \max\mleft\{a_1 \max_{i,\batch} |y_i^\batch| , a_2\sigma \mright\}$. 
Also in each iteration apart from the last one $\kappa$ decreases by a factor $2$ and $\kappa$ can't be smaller than $a_2\sigma$.
Hence, the number of iterations between the first one and the last one are at most $\log(\frac{a_1 \max_{i,\batch} |y_i^\batch|}{a_2\sigma}))$. Therefore the total number of iterations are at most $\log(\frac{a_1 \max_{i,\batch} |y_i^\batch|}{a_2\sigma}))+2$.

The first item follows from the definition of $w_\kappa$ in the subroutine \textsc{FindClippingPparameter}.

Next to prove the lower bound in item 2 we prove the claim that if in an iteration $\kappa\ge \max\mleft\{a_1 \sqrt{\E{}_{\weight}[f^\batch(w_\kappa,\kappa)]}, a_2\sigma \mright\}$ then the same condition will hold in the next iteration.

The condition $\kappa\ge  \max\mleft\{a_1 \sqrt{\E{}_{\weight}[f^\batch(w_\kappa,\kappa)]} , a_2\sigma \mright\}$ in the claim implies that $\kappa\ge \kappa_{new}$.
Then from the definition of clipped loss, for each $w$ and each $\batch$ we have $f^\batch(w,\kappa)\ge f^\batch(w,\kappa_{new})$.
It follows that $\E_{\weight}[f^\batch(w_{\kappa},\kappa)]\ge \E_{\weight}[f^\batch(w_{\kappa},\kappa_{new})]$.
And further $w_{\kappa_{new}}$ is stationary point for $f^\batch(.,\kappa_{new})$, hence it will achieve a smaller loss,  $\E_{\weight}[f^\batch(w_{\kappa_{new}},\kappa_{new})]\le \E_{\weight}[f^\batch(w_{\kappa},\kappa_{new})]$.
Therefore, $\E_{\weight}[f^\batch(w_{\kappa_{new}},\kappa_{new})]\le \E_{\weight}[f^\batch(w_{\kappa},\kappa)]$.
Hence, $\kappa_{new}= \max\mleft\{a_1 \sqrt{\E{}_{\weight}[f^\batch(w_\kappa,\kappa)]} , a_2\sigma \mright\}\ge \max\mleft\{a_1 \sqrt{\E_{\weight}[f^\batch(w_{\kappa_{new}},\kappa_{new})]} , a_2\sigma \mright\}$. This completes the proof of the claim.

Since the initial value of $\kappa$ is infinite the claim must hold in the first iteration, and therefore in each iteration thereafter. 
Therefore it must hold in the iteration when the algorithm terminates. This completes the proof of the lower bound in item 2.

The upper bound in the second item follows by observing that when the algorithm ends $\kappa\le 2\kappa_{new}$ and $\kappa_{new} =  a_1 \sqrt{\E_{\weight}[f^\batch(w_\kappa,\kappa)]} + a_2\sigma$.

Finally, we prove item 3 using item 2. 
We start by proving the lower bound in item 3.
From the lower bound in item 2, we have, $\kappa\ge a_2\sigma$. Then to complete the proof of the lower bound in item 3, it suffices to prove $\kappa  > \frac{a_1}2 \E_{\weight}\mleft[\frac{1}{\bsize}\sum_{i\in[\bsize]}|w_\kappa\cdot x_i^\batch-y_i^\batch|\mright]$.
To prove this by contradiction suppose $\kappa  < \frac{a_1}2 \E_{\weight}\mleft[\frac{1}{\bsize}\sum_{i\in[\bsize]}|w_\kappa\cdot x_i^\batch-y_i^\batch|\mright]$.
Then
\begin{align}
&\E{}_{\weight}[f^\batch(w_\kappa,\kappa)]\nonumber\\   &=\E{}_{\weight}\mleft[\frac{1}{\bsize}\sum_{i\in[\bsize]}f^\batch_i(w_\kappa,\kappa)\mright]\nonumber\\  
&=\frac{1}{\bsize}\sum_{i\in[\bsize]}\E{}_{\weight}\mleft[\mathbbm 1(|w_\kappa\cdot x_i^\batch-y_i^\batch|\le \kappa)\cdot\frac{(w_\kappa\cdot x_i^\batch-y_i^\batch)^2}2+\mathbbm 1(|w_\kappa\cdot x_i^\batch-y_i^\batch|> \kappa)\cdot\mleft(\kappa|w_\kappa\cdot x_i^\batch-y_i^\batch|-\frac{\kappa^2}2\mright)\mright]\nonumber\\   
&\ge\frac{1}{\bsize}\sum_{i\in[\bsize]}\mleft(\E{}_{\weight}\mleft[\mleft(\mathbbm 1(|w_\kappa\cdot x_i^\batch-y_i^\batch|\le \kappa)\cdot\frac{(w_\kappa\cdot x_i^\batch-y_i^\batch)^2}2\mright]+\E{}_{\weight}\mleft[\mathbbm 1(|w_\kappa\cdot x_i^\batch-y_i^\batch|> \kappa)\cdot\mleft(\frac{\kappa|w_\kappa\cdot x_i^\batch-y_i^\batch|}2\mright)\mright)\mright]\mright)\nonumber\\   
&\ineqlabel{a}\ge\frac{1}{2\bsize}\sum_{i\in[\bsize]}\E{}_{\weight}\mleft[\mathbbm 1(|w_\kappa\cdot x_i^\batch-y_i^\batch|\le \kappa)\cdot{|w_\kappa\cdot x_i^\batch-y_i^\batch|}\mright]^2+\frac{\kappa}{2\bsize}\sum_{i\in[\bsize]}\E{}_{\weight}\mleft[\mathbbm 1(|w_\kappa\cdot x_i^\batch-y_i^\batch|> \kappa)\cdot{|w_\kappa\cdot x_i^\batch-y_i^\batch|}\mright]\nonumber\\
&\ineqlabel{b}\ge\frac12\mleft(\frac{1}{\bsize}\sum_{i\in[\bsize]}\E{}_{\weight}\mleft[\mathbbm 1(|w_\kappa\cdot x_i^\batch-y_i^\batch|\le \kappa)\cdot{|w_\kappa\cdot x_i^\batch-y_i^\batch|}\mright]\mright)^2+\frac{\kappa}{2\bsize}\sum_{i\in[\bsize]}\E{}_{\weight}\mleft[\mathbbm 1(|w_\kappa\cdot x_i^\batch-y_i^\batch|> \kappa)\cdot{|w_\kappa\cdot x_i^\batch-y_i^\batch|}\mright]\nonumber\\ 
&\ineqlabel{c}\ge\frac{\kappa}{a_1}\mleft(\frac{1}{\bsize}\sum_{i\in[\bsize]}\E{}_{\weight}\mleft[\mathbbm 1(|w_\kappa\cdot x_i^\batch-y_i^\batch|\le \kappa)\cdot{|w_\kappa\cdot x_i^\batch-y_i^\batch|}\mright]\mright)+ \frac{\kappa}{2\bsize}\sum_{i\in[\bsize]}\E{}_{\weight}\mleft[\mathbbm 1(|w_\kappa\cdot x_i^\batch-y_i^\batch|> \kappa)\cdot{|w_\kappa\cdot x_i^\batch-y_i^\batch|}\mright]\nonumber\\ 
&\ineqlabel{d}\ge\frac{\kappa}{2a_1\bsize}\sum_{i\in[\bsize]}\mleft(\E{}_{\weight}\mleft[\mathbbm 1(|w_\kappa\cdot x_i^\batch-y_i^\batch|\le \kappa)\cdot{|w_\kappa\cdot x_i^\batch-y_i^\batch|}\mright]+ \E{}_{\weight}\mleft[\mathbbm 1(|w_\kappa\cdot x_i^\batch-y_i^\batch|> \kappa)\cdot{|w_\kappa\cdot x_i^\batch-y_i^\batch|}\mright]\mright)\nonumber\\ 
&=\frac{\kappa}{2a_1}\E{}_{\weight}\mleft[\frac{1}{\bsize}\sum_{i\in[\bsize]}|w_\kappa\cdot x_i^\batch-y_i^\batch|\mright]\nonumber\\ 
&\ineqlabel{e}\ge \frac{\kappa^2}{a_1^2}, 
\end{align}
here (a) and (b) follows the Cauchy-Schwarz inequality, (c) and (e) follows from our assumption $\kappa  < \frac{a_1}2 \E_{\weight}\mleft[\frac{1}{\bsize}\sum_{i\in[\bsize]}|w_\kappa\cdot x_i^\batch-y_i^\batch|\mright]$ and (d) follows since $a_1\ge 1$.

This contradicts the lower bound $\kappa\ge a_1\sqrt{\E{}_{\weight}[f^\batch(w_\kappa,\kappa)]}$ in item 2.
Hence we conclude, $\kappa  \ge  \frac{a_1}2 \E_{\weight}\mleft[\frac{1}{\bsize}\sum_{i\in[\bsize]}|w_\kappa\cdot x_i^\batch-y_i^\batch|\mright]$. This completes the proof of the lower bound in item 3.

Next, we prove the upper bound in item 3. We consider two cases. 
For the case when $a_1 \sqrt{\E{}_{\weight}[f^\batch(w_\kappa,\kappa)]}\le  a_2\sigma$ then upper bound in item 3 follows from the upper bound in item 2.
Next we prove for the other case, when $a_1 \sqrt{\E{}_{\weight}[f^\batch(w_\kappa,\kappa)]}>  a_2\sigma$. 
For this case item 2 implies $ \E{}_{\weight}[f^\batch(w_\kappa,\kappa)]\ge \frac{\kappa^2}{4a_1^2}$.

Next, from the definition of $f^\batch(w,\kappa)$,
\begin{align}
 \E{}_{\weight}[f^\batch(w_{\kappa},\kappa)]=\E{}_{\weight}\mleft[\frac{1}{\bsize}\sum_{i\in[\bsize]}f^\batch_i(w_\kappa,\kappa)\mright]\le \E{}_{\weight}\mleft[\frac{1}{\bsize}\sum_{i\in[\bsize]}\kappa|w_{\kappa}\cdot x_i^\batch-y_i^\batch|\mright]\le \kappa\E{}_{\weight}\mleft[\frac{1}{\bsize}\sum_{i\in[\bsize]}|w_{\kappa}\cdot x_i^\batch-y_i^\batch|\mright].   
\end{align}
Combining the above equation and $ \E{}_{\weight}[f^\batch(w_\kappa,\kappa)]\ge \frac{\kappa^2}{4a_1^2}$, we get,
\[
\frac{\kappa^2}{4a_1^2}\le \kappa\E{}_{\weight}\mleft[\frac{1}{\bsize}\sum_{i\in[\bsize]}|w_{\kappa}\cdot x_i^\batch-y_i^\batch|\mright].
\]
The upper bound in item 3 then follows from the above equation.
\end{proof}

\section{Correctness of estimated parameters for nice weight vectors}\label{sec:type1}


For batch $\batch\in \allbatches$, let $v^\batch(w) := \frac{1}{\bsize}\sum_{i\in[\bsize]}|w \cdot x_i^\batch-y_i^\batch|$. Since $w$ will be fixed in the proofs, we will often denote $v^\batch(w)$ as $v^\batch$.

In this section, we state and prove Theorems~\ref{thm:contyp1},~\ref{thm:x}  and~\ref{thm:bounds}.
For any triplet with a nice weight vector, Theorem~\ref{thm:contyp1} ensures the correctness of parameters calculated for Type-1 use of \textsc{Multifilter}.
For any triplet with a nice weight vector, Theorem~\ref{thm:bounds} ensures the correctness of parameters calculated for the case when it gets added to $M$ or goes through Type-2 use of \textsc{Multifilter}.
Theorem~\ref{thm:x} serves as an intermediate step in proving Theorem~\ref{thm:bounds}.

\begin{theorem}\label{thm:contyp1}
In Algorithm~\ref{alg:main} if the weight vector $\weight$ is such that $\weight^\goodbatches \ge 3\sizegood/4$, $\bsize\ge (16)^2\ucb C$, and Theorem~\ref{th:typerrorvar}'s conclusion holds, then for any $w$, the parameter $\thresholdB$ computed in the subroutine satisfies
\[
\thresholdB\ge 
\ucb\mleft(\frac{\sigma^2+ C\E{}_{\dist}[|w \cdot x_i^\batch-y_i^\batch|]^2}{\bsize}\mright),
\]
where $\ucb$ is the same universal positive constant as item 2 in Lemma~\ref{lem:conseq}.
\end{theorem}
\begin{proof}

To prove the theorem we first show that $\thresholdA$ calculated in the algorithm is $\ge \frac{7\E{}_{\dist}[|w \cdot x_i^\batch-y_i^\batch|]}8-\frac{\sigma}{8\sqrt C}$.

Let $\text{MED}$ denote median of the set $\{v^\batch:\batch\in \goodbatches\}$. From Theorem~\ref{th:typerrorvar} and Markov's inequality, it follows that  
\begin{align}
\mleft|\text{MED}-\E{}_{\dist}[|w \cdot x_i^\batch-y_i^\batch|]\mright| &\le  2\sqrt{\ucb\mleft(\frac{\sigma^2+ C\E{}_{\dist}[|w \cdot x_i^\batch-y_i^\batch|]^2}{\bsize}\mright)}\nonumber\\
&\le \frac{\E{}_{\dist}[|w \cdot x_i^\batch-y_i^\batch|]}8+\frac{\sigma}{8\sqrt C}. \label{eq:corthm}
\end{align}
where the last inequality uses $\bsize\ge (16)^2\ucb C$.
It follows that 
\[
\text{MED}\ge \frac{7\E{}_{\dist}[|w \cdot x_i^\batch-y_i^\batch|]}8-\frac{\sigma}{8\sqrt C}.
\]

Then to complete the proof we show that $\text{MED}\le \thresholdA$.
Note that
\[
\sum_{\batch\in \goodbatches: v^\batch< \text{MED}} \weight^\batch\le |\{\batch\in \goodbatches: v^\batch< \text{MED} \}|< \frac{\sizegood}{2}.
\]
Then,
\begin{align}\label{eq:med}
    \sum_{\batch\in \allbatches: v^\batch\ge  \text{MED}} \weight^\batch\ge \sum_{\batch\in \goodbatches: v^\batch\ge  \text{MED}} \weight^\batch = \sum_{\batch\in \goodbatches} \weight^\batch  -\sum_{\batch\in \goodbatches: v^\batch<  \text{MED}} \weight^\batch > \weight^\goodbatches- \frac{\sizegood}{2}\ge \frac{3\sizegood}{4}-\frac{\sizegood}{2}\ge \frac{\sizegood}{4}.
\end{align}
And since from the definition of $\thresholdA$, we have $\sum_{\batch: v^\batch>  \thresholdA} \weight^\batch\le \goodfrac\size{\allbatches}/4\le \frac{\sizegood}{4}$, it follows that $\text{MED}\le \thresholdA$. 

Therefore, $\thresholdA\ge  \frac{7\E{}_{\dist}[|w \cdot x_i^\batch-y_i^\batch|]}8-\frac{\sigma}{8\sqrt C}$.
The lower bound in the theorem on $\thresholdB$ then follows from the relation between $\thresholdA$ and $\thresholdB$.
\end{proof}


\begin{theorem}\label{thm:x}
Suppose regularity conditions holds, and $\weight$, $w$ and $\bsize$ satisfy $\bsize\ge \max\{\frac{(32)^2\ucmf \ucb C\log^2(2/\goodfrac) }{\goodfrac},(16)^2\ucb C\}$, $\weight^\goodbatches \ge 3\sizegood/4$, and  
\[
\text{Var}{}_{\weight}\mleft(v^\batch(w)\mright)\le \ucmf \log^2(2/\goodfrac) \thresholdB,
\]
then
\begin{align*}
    \frac{3\E{}_{\dist}[|(w-w^*)\cdot x_i^\batch|]}4-{\sigma}\le     \E{}_{\weight}\mleft[v^\batch(w)\mright]\le \frac{4\E{}_{\dist}[|(w-w^*)\cdot x_i^\batch|]}3+2\sigma.
\end{align*}
\end{theorem}

In proving Theorem~\ref{thm:x} the following auxiliary lemma will be useful. 
We prove this lemma in Subsection~\ref{sec:someaux}.
\begin{lemma}\label{lem:expvar}
Let $Z$ be any random variable over the reals. For any $z\in \reals$, such that $\Pr[Z> z]\le 1/2$, we have
\[
z -\sqrt{\frac{\Var(Z)}{\Pr[Z\ge z]}}\le \E[Z] \le z +\sqrt{2\Var(Z)}.
\]
and for all $z\in Z$,
\[
 |\E[Z]-z| \le {\sqrt{\frac{\Var(Z)}{\min\{\Pr[Z\le z], \Pr[Z\ge z],0.5\}}}}.
\]
\end{lemma}

Now we prove Theorem~\ref{thm:x} using the above Lemma.

\begin{proof}[Proof of Theorem~\ref{thm:x}]
Let $\text{MED}$ denote median of the set $\{v^\batch:\batch\in \goodbatches\}$. In Equation~\eqref{eq:med} we showed,
\[
\sum_{\batch\in \allbatches: v^\batch\ge  \text{MED}} \weight^\batch\ge \frac{\sizegood}{4}.
\]
Hence,
\[
\frac{\sum_{\batch\in \allbatches: v^\batch\ge  \text{MED}} \weight^\batch}{\weight^\allbatches}\ge \frac{\sizegood}{4\size{\allbatches}}\ge \frac{\goodfrac}{4}.
\]
Similarly, by symmetry, one can show
\[
\frac{\sum_{\batch\in \allbatches: v^\batch\le  \text{MED}} \weight^\batch}{\weight^\allbatches}\ge  \frac{\goodfrac}{4}.
\]
Then from the second bound in Lemma~\ref{lem:expvar},
\begin{align}
  |\E{}_{\weight}[v^\batch]-  \text{MED} |\le \sqrt{\frac{4\Var_{\weight}[v^b] }{\goodfrac}}.
\end{align}

From Equation~\eqref{eq:corthm}, the above equation, and the triangle inequality,
\begin{align}\label{eq:expbound}
  |\E{}_{\weight}[v^\batch]-  \E{}_{\dist}[|w \cdot x_i^\batch-y_i^\batch|] |\le \sqrt{\frac{4\Var_{\weight}[v^b] }{\goodfrac}}+\frac{\E{}_{\dist}[|w \cdot x_i^\batch-y_i^\batch|]}8+\frac{\sigma}{8\sqrt C}.
\end{align}

Next, from the definition of $\thresholdA$, we have $\sum_{\batch:v^\batch\ge\thresholdA}\weight^\batch\ge \goodfrac\size{\allbatches}/4$ and $\sum_{\batch:v^\batch>\thresholdA}\weight^\batch< \goodfrac\size{\allbatches}/4$.  
Then
\[
\frac{\sum_{\batch:v^\batch\ge\thresholdA}\weight^\batch}{\weight^\allbatches} \ge \frac{\goodfrac\size{\allbatches}}{4\weight^\allbatches} \ge \frac{\goodfrac\size{\allbatches}}{4\size{\allbatches}} \ge  \frac{\goodfrac}{4},
\]
and 
\[
\frac{\sum_{\batch:v^\batch>\thresholdA}\weight^\batch}{\weight^\allbatches} < \frac{\goodfrac\size{\allbatches}}{4\weight^\allbatches} \le\frac{\goodfrac\size{\allbatches}}{4\weight^\goodbatches} \le  \frac{\goodfrac\size{\allbatches}}{4(3\sizegood/4)} \le  \frac{1}{3}.
\]

Then from the first bound in Lemma~\ref{lem:expvar},
\begin{align}\label{eq:nsnsd}
\thresholdA -\sqrt{\frac{4\Var_{\weight}[v^b] }{\goodfrac}}\le \E{}_{\weight}[v^\batch].
\end{align}

In this lemma, we had assumed the following bound on the variance of $v^\batch$,
\[
\Var_{\weight}[v^\batch] 
\le \ucmf \log^2(2/\goodfrac) \thresholdB.
\]
Next, 
\begin{align*}
\frac{\Var_{\weight}[v^\batch] }{\goodfrac}\le {\frac{\ucmf \log^2(2/\goodfrac) \thresholdB}{\goodfrac}} =  {\frac{\ucmf \log^2(2/\goodfrac) \ucb(\sigma^2+ (2\sqrt C\thresholdA+\sigma)^2)}{\bsize\goodfrac}}\le\frac{(\sigma^2+ 4 C\thresholdA^2+2\sigma^2)}{32^2 C}\le\frac{\sigma^2}{256 C}+\frac{\thresholdA^2}{256},    
\end{align*}
here the equality follows from the relation between $\thresholdA$ and $\thresholdB$ and the first inequality follows as $\bsize\ge \frac{(32)^2 C\ucmf \ucb\log^2(2/\goodfrac) }{\goodfrac}$.

Then
\begin{align*}
\sqrt{\frac{\Var_{\weight}[v^\batch] }{\goodfrac}}\le \sqrt{\frac{\sigma^2}{256 C}+\frac{\thresholdA^2}{256}}\le \frac{\sigma}{16\sqrt{C}}+\frac{\thresholdA}{16}\le \frac{\sigma}{16\sqrt{C}}+\frac{1}{16}\mleft(\E{}_{\weight}[v^\batch]+ 2\sqrt{\frac{\Var_{\weight}[v^\batch]}\goodfrac}\mright),
\end{align*}
here the second inequality used $\sqrt{a^2+b^2}\le |a|+|b|$ and the last inequality used~\eqref{eq:nsnsd}.
From the above equation, it follows that
\begin{align*}
\sqrt{\frac{\Var_{\weight}[v^\batch] }{\goodfrac}}\le\frac{\sigma}{14\sqrt{C}}+\frac{1}{14}\E{}_{\weight}[v^\batch].
\end{align*}

Combining the above bound and Equation~\eqref{eq:expbound}
\begin{align*}
 |\E{}_{\weight}[v^\batch]-  \E{}_{\dist}[|w \cdot x_i^\batch-y_i^\batch|] |\le \frac{\sigma}{7\sqrt{C}}+\frac{1}{7}\E{}_{\weight}[v^\batch]+\frac{\E{}_{\dist}[|w \cdot x_i^\batch-y_i^\batch|]}8+\frac{\sigma}{8\sqrt C}.    
\end{align*}

From the above equation it follows that
\begin{align}\label{eq:prefinal}
\frac{49\E{}_{\dist}[|w \cdot x_i^\batch-y_i^\batch|]}{64}-\frac{15\sigma}{64\sqrt C}\le \E{}_{\weight}[v^\batch]\le \frac{21\E{}_{\dist}[|w \cdot x_i^\batch-y_i^\batch|]}{16}+\frac{5\sigma}{16\sqrt C}.    
\end{align}

Finally, we upper bound and lower bound $\E{}_{\dist}[|w \cdot x_i^\batch-y_i^\batch|]$ to complete the proof.
To prove the upper bound, note that,
\begin{align*}
\E{}_{\dist}[|w\cdot x_i^\batch-y_i^\batch|] = \E{}_{\dist}[|(w-w^*)\cdot x_i^\batch-n_i^\batch|]\le \E{}_{\dist}[|(w-w^*)\cdot x_i^\batch|]+\E{}_{\dist}[|n_i^\batch|] 
\le \E{}_{\dist}[|(w-w^*)\cdot x_i^\batch|]+\sigma,
\end{align*}
here the last inequality used $\E{}_{\dist}[|n_i^\batch|]\le \sqrt{\E{}_{\dist}[|n_i^\batch|^2] }$. Combining the above upper bound with the upper bound in~\eqref{eq:prefinal} and using $C\ge 1$ proves the upper bound in the lemma.
Similarly, we can show 
\begin{align*}
\E{}_{\dist}[|w\cdot x_i^\batch-y_i^\batch|] 
\ge \E{}_{\dist}[|(w-w^*)\cdot x_i^\batch|]-\sigma,
\end{align*}
Combining the above lower bound with the lower bounds in~\eqref{eq:prefinal} and using $C\ge 1$ proves the lower bound in the lemma.
\end{proof}

\begin{theorem}\label{thm:bounds}
Suppose regularity conditions holds, and $\weight$, $w$ and $\bsize$ satisfy $\bsize\ge \max\{\frac{(32)^2\ucmf \ucb C\log^2(2/\goodfrac) }{\goodfrac},(16)^2\ucb C\}$, $\weight^\goodbatches \ge 3\sizegood/4$, and  
\[
\text{Var}{}_{\weight}\mleft(\frac{1}{\bsize}\sum_{i\in[\bsize]}|w \cdot x_i^\batch-y_i^\batch|\mright)\le \ucmf \log^2(2/\goodfrac) \thresholdB,
\]
then for $\kappa$, $w$ returned by subroutine \textsc{FindClippingParameter} and $\thresholdC$ calculated by \textsc{MainAlgorithm}, we have
\begin{enumerate}
    \item $\uccb\frac{\sigma^2+ C \E{}_{\dist}[((w-w^*)\cdot x_i^\batch)^2]}{\bsize}  \le \thresholdC\le \frac{c_6 C^2(\sigma^2 + \E{}_\dist[|(w-w^*)\cdot x_i^\batch |]^2)}{\bsize}$, where $\uccb$ is the same positive constant as in item 1 of Lemma~\ref{lem:conseq} and $c_6$ is some other positive universal constant. 
    \item $  \max\{8 \sqrt{C\E{}_\dist[|x_i^\batch\cdot (w-w^*)|^2]}, 2(8\nhypcon C)^{1/\npower}\sigma, 2(\frac{8\nhypcon\sqrt{\bsize\goodfrac} }{\log(2/\goodfrac)})^{1/{(\npower-1)}}\sigma \}\le \kappa$ and \\$\kappa\le c_7\mleft( C^2\sqrt{\E{}_\dist[|x_i^\batch\cdot (w-w^*)|^2]}+ (\nhypcon C)^{1/\npower}\sigma+ (\frac{\nhypcon\sqrt{\bsize\goodfrac} }{\log(2/\goodfrac)})^{1/{(\npower-1)}}\sigma \mright)$, where $c_7$ is some other positive universal constant
\end{enumerate}
\end{theorem}
Note that the range of $\kappa$ in item 2 of the above Theorem is the same as that in~\ref{cond:kappa}.

In proving the theorem the following lemma will be useful.
\begin{lemma}\label{lem:l2l1} For any vectors $u$, we have 
\[
\sqrt{\frac{\E{}_{\dist}[|u\cdot x_i^\batch|^2]}{8C}}
\le \E{}_{\dist}[|u\cdot x_i^\batch|] 
\le \sqrt{{\E{}_{\dist}[|u\cdot x_i^\batch|^2]}}
\]
\end{lemma}
We prove the above auxiliary lemma in Section~\ref{sec:cors} using the Cauchy-Schwarz inequality for the upper bound and $L4-L2$ hypercontractivity for the lower bound.

Next, we prove Theorem~\ref{thm:bounds} using the above lemma and Theorem~\ref{thm:x}.  
\begin{proof}[Proof of Theorem~\ref{thm:bounds}]
We start by proving the first item.
For convenience, we recall the definition of $\thresholdC$ in ~\eqref{eq:inalgthresb},
\begin{align*}
    \textstyle\thresholdC= \frac{\uccb}{\bsize}\mleft(\sigma^2 + 16C^2 \mleft(\E{}_{\weight}[v^\batch] +\sigma\mright)^2\mright).
\end{align*}
The upper bound in the item follows from this definition of $\thresholdC$ and the upper bound on $\E{}_{\weight}[v^\batch]$ in Lemma~\ref{thm:x}.

Using the lower bound bound on $\E{}_{\weight}[v^\batch]$ in Lemma~\ref{thm:x} and definition of $\thresholdC$,
\begin{align*}
    \textstyle\thresholdC\ge  \frac{\uccb}{\bsize}\mleft(\sigma^2 +  9 C^2 \E{}_\dist[|(w-w^*)\cdot x_i^\batch |]^2\mright)\ge \frac{\uccb}{\bsize}\mleft(\sigma^2 +  \frac{9}{8} C^2 \E{}_\dist[|(w-w^*)\cdot x_i^\batch |^2]\mright),
\end{align*}
where the last step used Lemma~\ref{lem:l2l1}. This completes the proof of lower bound item 1.

Next, we prove item 2.
From Theorem~\ref{th:clipparbound},
\[
\textstyle\max\mleft\{\frac{a_1}2 \E_{\weight}\mleft[\frac{1}{\bsize}\sum_{i\in[\bsize]}|w\cdot x_i^\batch-y_i^\batch|\mright], a_2\sigma\mright\} \le \kappa \le  \max\mleft\{4 a_1^2 \E_{\weight}\mleft[\frac{1}{\bsize}\sum_{i\in[\bsize]}|w\cdot x_i^\batch-y_i^\batch|\mright], a_2\sigma\mright\}.
\]
Since for any $a,b>0$,  $(a+b)/2\le \max(a,b)\le a+b$. Then from the above bound,
\[
\textstyle\frac{a_1}4 \E_{\weight}\mleft[\frac{1}{\bsize}\sum_{i\in[\bsize]}|w_\kappa\cdot x_i^\batch-y_i^\batch|\mright]+\frac{ a_2\sigma}2\le \kappa \le  4 a_1^2 \E_{\weight}\mleft[\frac{1}{\bsize}\sum_{i\in[\bsize]}|w_\kappa\cdot x_i^\batch-y_i^\batch|\mright]+a_2\sigma.
\]
Using the bound on $\E{}_{\weight}[v^\batch]$ in Lemma~\ref{thm:x} in the above equation
\[
\textstyle\frac{a_1}4 \mleft(\frac{3\E{}_{\dist}[|(w-w^*)\cdot x_i^\batch|]}4-\frac\sigma2\mright)+\frac{ a_2\sigma}2\le \kappa \le  4 a_1^2\mleft(\frac{4\E{}_{\dist}[|(w-w^*)\cdot x_i^\batch|]}3+2\sigma\mright)+a_2\sigma.
\]
Using Lemma~\ref{lem:l2l1}, and the above equation,
\[
\textstyle\frac{3a_1}{32\sqrt{2C}} \sqrt{\E{}_{\dist}[|(w-w^*)\cdot x_i^\batch|^2]} +\frac{ (4a_2-a_1)\sigma}8\le \kappa \le 
\frac{16a_1^2}{3} \sqrt{\E{}_{\dist}[|(w-w^*)\cdot x_i^\batch|^2]}
+(8a_1^2+a_2)\sigma.
\]
The upper bound and lower bound in item 2 then follow by using the values $a_1 = \frac{256C\sqrt{2}}{3}$ and $a_2 = \frac{a_1}{4} + 2\max\{2(8\nhypcon C)^{1/\npower}, 2(\frac{8\nhypcon\sqrt{\bsize\goodfrac} }{\log(2/\goodfrac)})^{1/{(\npower-1)}}\}$.
\end{proof}

\subsection{Proof of Lemma~\ref{lem:expvar}}\label{sec:someaux}

\begin{proof}[Proof of Lemma~\ref{lem:expvar}]
We only prove the first statement as the second statement and then follow from the symmetry.

We start by proving the upper bound in the first statement. We consider two cases, $\E[Z]\le z$ and $\E[Z]> z$.
For the first case, the upper bound automatically follows. Next, we prove the second case. In this case,
\[
\Var(Z) = \E[(Z-\E[Z])^2 ]\ge \E[ \mathbbm 1(Z\le z) (Z-\E[Z])^2 ] \ge \E[ \mathbbm 1(Z\le z) (z-\E[Z])^2 ] = \Pr[Z\le z] (z-\E[Z])^2.
\]
Then using $\Pr[Z\le z]= 1-\Pr[Z> z]\ge 1/2$, we get
\[
\Var(Z)\ge  \frac{(z-\E[Z])^2}2.
\]
The upper bound from the above equation.

Next, we prove the lower bound.
Again, we consider two cases, $\E[Z]\ge z$ and $\E[Z]< z$.
For the first case, the lower bound automatically follows. Next, we prove the second case. In this case,
\[
\Var(Z) = \E[(Z-\E[Z])^2 ]\ge \E[ \mathbbm 1(Z\ge  z) (Z-\E[Z])^2 ] \ge \E[ \mathbbm 1(Z\ge z) (z-\E[Z])^2 ] = \Pr[Z\ge z] (z-\E[Z])^2,
\]
from which the lower bound follows.

By symmetry, for any $z\in \reals$, such that $\Pr[Z<  z]\le 1/2$, one can show that
\[
z -\sqrt{2\Var(Z)}\le \E[Z] \le z +\sqrt{\frac{\Var(Z)}{\Pr[Z\le z]}}.
\]
Since for any $z$, either $\Pr[Z>  z]\le  1/2$ or $\Pr[Z<  z]\le 1/2$, Hence, either the first bound in the Lemma or the above bound holds for each $z$, therefore for any $z\in \reals$,
\[
z -\max\mleft\{\sqrt{\frac{\Var(Z)}{\Pr[Z\ge z]}},\sqrt{2\Var(Z)}\mright\}\le \E[Z] \le z +\max\mleft\{\sqrt{\frac{\Var(Z)}{\Pr[Z\le z]}},\sqrt{2\Var(Z)}\mright\}.
\]
The second bound in the lemma is implied by the above bound.
\end{proof}

\subsection{Proof of Lemma~\ref{lem:l2l1}}\label{sec:cors}

\begin{proof}[Proof of Lemma~\ref{lem:l2l1}]
The upper bound on $\E{}_{\dist}[|u\cdot x_i^\batch|] $ follows from the Cauchy-Schwarz inequality,
\begin{align*}
\E{}_{\dist}[|u\cdot x_i^\batch|]\le \sqrt{\E{}_{\dist}[|u\cdot x_i^\batch|^2] }\le  \|\Sigma\|\le 1.
\end{align*}

Next, we prove the lower bound. From Markov’s inequality
\[
\Pr{}_\dist[\|x_i^\batch\cdot u\|^2\ge 2C\E{}_{\dist}[\|x_i^\batch\cdot u\|^2]] = \Pr{}_\dist[\|x_i^\batch\cdot u\|^4\ge 4C^2(\E{}_{\dist}[\|x_i^\batch\cdot u\|^2])^2] = \frac{\E{}_\dist [\|x_i^\batch\cdot u\|^4] }{4C^2(\E{}_{\dist}[\|x_i^\batch\cdot u\|^2])^2}\le \frac{1}{4C},
\]
where the last step uses $L4-L2$ hypercontractivity. 

Then, from the Cauchy-Schwarz inequality,
\begin{align*}
\E{}_{\dist}[\|x_i^\batch\cdot u\|^2\cdot \mathbbm 1(\|x_i^\batch\cdot u\|^2\ge 2C \E{}_{\dist}[\|x_i^\batch\cdot u\|^2])] &\le\sqrt{\E{}_{\dist}\mleft[\mathbbm 1\mleft(\|x_i^\batch\cdot u\|^2\ge 2C \E{}_{\dist}[\|x_i^\batch\cdot u\|^2]\mright)\mright]\cdot \E{}_{\dist}[ \|x_i^\batch\cdot u\|^4]}\\ 
&\le\sqrt{\Pr{}_\dist\mleft[\|x_i^\batch\cdot u\|^2\ge 2C \E{}_{\dist}[\|x_i^\batch\cdot u\|^2]\mright]\cdot C  \E{}_{\dist}[ \|x_i^\batch\cdot u\|^2]^2}  \\
&\le \frac{1}{2} \E{}_{\dist}[ \|x_i^\batch\cdot u\|^2]. 
\end{align*}
Then,
\begin{align*}
\E{}_{\dist}[\|x_i^\batch\cdot u\|^2\cdot \mathbbm 1(\|x_i^\batch\cdot u\|^2< 2C \E{}_{\dist}[\|x_i^\batch\cdot u\|^2])] &=\E{}_{\dist}[\|x_i^\batch\cdot u\|^2]- \E{}_{\dist}[\|x_i^\batch\cdot u\|^2\cdot \mathbbm 1(\|x_i^\batch\cdot u\|^2\ge 2C \E{}_{\dist}[\|x_i^\batch\cdot u\|^2])]\\
&\ge \frac12 \E{}_{\dist}[ \|x_i^\batch\cdot u\|^2]
\end{align*}
Next,
\begin{align*}
\E{}_{\dist}\mleft[\|x_i^\batch\cdot u\|^2\cdot \mathbbm 1\mleft(\|x_i^\batch\cdot u\|^2< 2C \E{}_{\dist}[\|x_i^\batch\cdot u\|^2]\mright)\mright]&\le    \E{}_{\dist}\mleft[\|x_i^\batch\cdot u\|\cdot \sqrt{2C \E{}_{\dist}[\|x_i^\batch\cdot u\|^2]}\cdot\mathbbm 1\mleft(\|x_i^\batch\cdot u\|^2< 2C \E{}_{\dist}[\|x_i^\batch\cdot u\|^2]\mright)\mright] \\
&\le    
\sqrt{2C \E{}_{\dist}[\|x_i^\batch\cdot u\|^2]}\cdot\E{}_{\dist}[\|x_i^\batch\cdot u\|].
\end{align*}
Combining the above two equations we get
\[
\E{}_{\dist}[\|x_i^\batch\cdot u\|]\ge \frac{\E{}_{\dist}[\|x_i^\batch\cdot u\|^2]}{2\sqrt{2C\E{}_{\dist}[\|x_i^\batch\cdot u\|^2]}} = \frac{\sqrt{\E{}_{\dist}[\|x_i^\batch\cdot u\|^2]}}{2\sqrt{2C}} 
.
\]
\end{proof}

\section{Multi-filtering}\label{sec:mult}

In this section, we state the subroutine \textsc{Multifilter}, a simple modification of \textsc{BasicMultifilter} algorithm in~\cite{DiakonikolasKK20}.

The subroutine takes a weight vector $\weight$, a real function $z^\batch$ on batches, and a parameter $\theta$ as input and produces new weight vectors.

This subroutine is used only when:
\begin{align}\label{con}
\text{Var}_{B,\weight}(z^\batch)> {\ucmf\log^2(2/\goodfrac)}\theta,    
\end{align}
where $\ucmf$ is an universal constant (Same as $2* C$, where $C$ is the constant in \textsc{BasicMultifilter} algorithm in~\cite{DiakonikolasKK20}). 

When the variance of $z^\batch$ for good batches is smaller than $\theta$ and the weight vector $\weight$ is nice that is $\weight^{\goodbatches} \ge 3/4\sizegood$, then at least one of the new weight vectors produced by this subroutine has a higher fraction of weights in good vector than the original weight vector $\weight$.

\begin{algorithm}[tbh]
   \caption{\textsc{Multifilter}}
   \label{alg:basicmult}
\begin{algorithmic}
   \STATE {\bfseries Input:} Set $\allbatches$, $\goodfrac$, $\weight$, $\{z^\batch\}_{\batch\in \allbatches}$, $\theta$.
   \COMMENT{Input must satisfy Condition~\eqref{con}}   
   \STATE{Let $a= \inf\{z: \sum_{\batch: z^\batch < z} \weight^\batch \le \goodfrac\weight^{\allbatches}/8\}$ and $b= \sup\{z: \sum_{\batch: z^\batch > z} \weight^\batch \le \goodfrac\weight^{\allbatches}/8\}$}
   \STATE{Let $\Bsc = \{\batch\in \allbatches: z^\batch\in [a,b] \}$}
   \IF{$\text{Var}_{\Bsc,\weight}(z^\batch)\le \frac{\ucmf\log^2(2/\goodfrac)\theta}2$}
   %
   \STATE{Let $f^\batch = \min_{z\in [a,b]} |z^\batch-z|^2$, and the new weight of each batch $\batch\in \allbatches$ be}
   \begin{align}\label{eq:newweighta}
   \textstyle\weight_{\text{new}}^\batch = \mleft(1-\frac{f^\batch}{\max_{\batch\in \allbatches:\weight^\batch > 0} f^\batch}\mright)\weight^\batch
   \end{align}
   \STATE{$\textsc{NewWeights}\gets\{\weight_{\text{new}}\}$}
   \ELSE
   \STATE{Find $z\in \reals$ and $R>0$ such that sets $\Bsc = \{\batch\in \allbatches: z^\batch \ge z- R \}$ and $\allbatches'' = \{\batch\in \allbatches: z^\batch < z+ R \}$ satisfy}
   \begin{align}\label{eq:newweightb}
   (\weight^{\Bsc})^2 + (\weight^{\allbatches''})^2\le (\weight^{\allbatches})^2,  
   \end{align}
   and
   \begin{align}\label{eq:newweightc}
      \textstyle\min\mleft(1-\frac{\weight^{\Bsc}}{\weight^{\allbatches}},1-\frac{\weight^{\allbatches''}}{\weight^{\allbatches}}\mright) \ge \frac{48\log(\frac{2}{\goodfrac})}{R^2}.
   \end{align}
   \COMMENT{Existence of such $z$ and $R$ is guaranteed as shown in Lemma 3.6 of~\cite{DiakonikolasKK20}.} 
   \STATE{For each $\batch\in \allbatches$, let $\weight_1^\batch = \weight^\batch \cdot \mathbbm 1(\batch\in \Bsc)$ and $\weight_2^\batch = \weight^\batch \cdot \mathbbm 1(\batch\in \allbatches'')$. Let $\weight_1=\{\weight_1^\batch\}_{\batch\in \allbatches}$ and $\weight_2=\{\weight_2^\batch\}_{\batch\in \allbatches}$. }
   \STATE{$\textsc{NewWeights}\gets \{\weight_1, \weight_2$\}}
   \ENDIF
   \STATE{Return($\textsc{NewWeights}$)}
\end{algorithmic}
\end{algorithm}

In \textsc{BasicMultifilter} subroutine of~\cite{DiakonikolasKK20} input is not restricted by the condition in Equation~\eqref{con}. However, when input meets this condition \textsc{BasicMultifilter} and its modification \textsc{Multifilter} behaves the same. 

Therefore, the guarantees for weight vectors returned by \textsc{Multifilter} follows from the guarantees of \textsc{BasicMultifilter} in~\cite{DiakonikolasKK20}.
We characterize these guarantees in Theorem~\ref{th:basicmultgua}.

\begin{theorem}~\label{th:basicmultgua}
Let $\{z^\batch\}_{\batch\in \allbatches}$ be collection of real numbers associated with batches, $\weight$ be a weight vector, and threshold $\theta>0$ be such that condition in~\eqref{con} holds.
Then \textsc{Multifilter}($\allbatches, \weight, \{z^\batch\}_{\batch\in\allbatches}, \thresholdB$) returns a list 
$\textsc{NewWeights}$ containing either one  or two new weight vectors such that,
\begin{enumerate}
    \item Sum of square of the total weight of new weight vectors is bounded by the square of the total weight of $\weight$, namely
    \begin{align}\label{eq:weightevolve}
        \sum_{\widetilde \weight\in \textsc{NewWeights}}(\widetilde \weight^\allbatches)^2\le (\weight^\allbatches)^2.
    \end{align} 
    \item In the new weight vectors returned the weight of at least one of the weight vectors has been set to zero, that is for each weight vector ${\widetilde \weight\in \textsc{NewWeights}}$,
\begin{align}
    \{\batch: \tilde\weight^\batch > 0\} \subset \{\batch: \weight^\batch > 0\}, 
\end{align}
\item If weight vector $\weight$ is such that $\weight^\goodbatches \ge 3\sizegood/4$ and for good batches the variance $\text{Var}_{\goodbatches}(z^\batch)\le \theta$ is bounded, then for at least one of the weight vector ${\widetilde \weight\in \textsc{NewWeights}}$,
          \begin{align}\label{eq:weightev2}
        \frac{\weight^\goodbatches-\tilde \weight^\goodbatches}{\weight^\goodbatches}\le  \frac{\weight^\allbatches-\widetilde \weight^\allbatches}{\weight^\allbatches}\cdot\frac{1}{24\log(2/\goodfrac)}.
        \end{align}
\end{enumerate}
\begin{proof}
When the list \textsc{NewWeights} contains one weight vector it is generated using Equation~\eqref{eq:newweighta}, and when the list \textsc{NewWeights} contains one weight vector it is generated using Equations~\eqref{eq:newweightb} and~\eqref{eq:newweightc}.
In both cases, item 1 and item 2 of the Theorem follow immediately from these equations.
The last item follows from Corollary 3.8 in~\cite{DiakonikolasKK20}.
\end{proof} 
 
\end{theorem}

\subsection{Guarantees for the use of \textsc{Multifilter} in Algorithm~\ref{alg:main}}

The following Theorem characterizes the use \textsc{Multifilter} by our algorithm.
The proof of the theorem is similar to the proofs for the main algorithm in~\cite{DiakonikolasKK20}. 

\begin{theorem}\label{th:temp}
At the end of Algorithm~\ref{alg:main} the size of $M$ is at most $4/\goodfrac^2$ and the algorithm makes at most $\cO(|\allbatches|/\goodfrac^2)$ calls to \textsc{Multifilter}. 
And, if for every use of subroutine \textsc{Multifilter} by the algorithm we have $\text{Var}_{\goodbatches}(z^\batch)\le \theta$ then there is at least one triplet $(\weight,w,\kappa)$ in $M$ such that $\weight^\goodbatches \ge 3\sizegood/4$.
\end{theorem}

\begin{proof}
First note that the if blocks in Algorithm~\ref{alg:main} ensures that for every use of subroutine \textsc{Multifilter} Equation~\eqref{con} is satisfied, therefore we can use the guarantees in Theorem~\ref{th:basicmultgua}.

First we upper bound the size of $M$.

The progress of Algorithm~\ref{alg:main} may be described using a tree.
The internal nodes of this tree are the weight vectors that have gone through subroutine \textsc{Multifilter} at some point of the algorithm, and children of these internal nodes are new weight vectors returned by  \textsc{Multifilter}.
Observe that any weight vector $\weight$ encountered in Algorithm~\ref{alg:main} is ignored iff $\weight^\allbatches< \goodfrac\size{\allbatches}/2$. If it is not ignored then either it is added to $M$ (in form of a triplet), or else it goes through subroutine \textsc{Multifilter}.

It follows that, if a node $\weight$ is an internal node or a leaf in $M$ then
\begin{align}\label{eq:leafweight}
    \weight^\allbatches\ge \goodfrac\size{\allbatches}/2.
\end{align}

From Equation~\eqref{eq:weightevolve}, it follows that the total weight squared for each node is greater than equal to that of its children.
It follows that the total weight squared of the root, $\weight_{\text{init}}$ is greater than equal to the sum of the square of weights of all the leaves.
And since all weight vectors in $M$ are among the leaves of the tree, and have total weight at least $\goodfrac\size{\allbatches}/2$,
\begin{align*}
    (\weight_{\text{init}}^\allbatches)^2 \ge \sum_{\weight\in M}  (\weight^\allbatches)^2 \ge \sum_{\weight\in M}  (\frac{\goodfrac\size{\allbatches}}{2})^2,
\end{align*}
here the last step follows from Equation~\eqref{eq:leafweight}.
Using $\weight_{\text{init}}^\allbatches= |\allbatches|$, in the above equation we get $|M|\le 4/\goodfrac^2$.

Similarly, it can be shown that the number of branches in the tree is at most $\cO(1/\goodfrac^2)$.
Item 2 in Theorem~\ref{th:basicmultgua} implies that each iteration of \textsc{Multifilter} zeroes out the weight of one of the batches. Hence for any weight $\weight$ at depth $d$, we have $\weight^\allbatches\le |\allbatches|-d$.
Therefore, the depth of the tree can't be more than $|\allbatches|$. Hence, the number of nodes in the tree is upper bounded by $\cO(|\allbatches|/\goodfrac^2)$.
And since each call to \textsc{Multifilter} corresponds to a non-leaf node in the tree, the total calls to \textsc{Multifilter} by Algorithm~\ref{alg:main} are upper bounded by $\cO(|\allbatches|/\goodfrac^2)$.

Next, we show that if for each use of \textsc{Multifilter} we have $\text{Var}_{\goodbatches}(z^\batch)\le \theta$ then one of the weight vector $\weight\in M $ must satisfy  $\weight^\goodbatches \ge 3\sizegood/4$.

Let $\weight_0 = \weight_{\text{init}}$ and suppose for each $i$, weight vectors $\weight_i$ and $\weight_{i+1}$ are related as follows:
 \begin{align}\label{eq:desc}
        \frac{\weight^\goodbatches_{i}- \weight^\goodbatches_{i+1}}{\weight^\goodbatches_i}\le  \frac{\weight^\allbatches_i- \weight^\allbatches_{i+1}}{\weight^\allbatches_i}\cdot\frac{1}{24\log(2/\goodfrac)}.
\end{align}
Then Lemma 3.12 in~\cite{DiakonikolasKK20} showed that under the above relation, for each $i$, we have $\weight^\goodbatches_i \ge 3\sizegood/4$.

We show that there is a branch of the tree such that $\weight_i$ and $\weight_{i+1}$ are related using the above equation, where for each $i$, $\weight_i$ denote the weight vector corresponding to the node at $i^{th}$ level in this branch. From the preceding discussion, this would imply that for each $i$, $\weight^\goodbatches_i \ge 3\sizegood/4$.

We prove it by induction. For $i = 0$, we select $\weight_i = \weight_{\text{init}}$. Note that $\weight_{\text{init}}^\goodbatches=\sizegood$, hence  $\weight^\goodbatches_i \ge 3\sizegood/4$. 

If $\weight_i$ is a leaf then the branch is complete.
Else, since $\weight^\goodbatches_i \ge 3\sizegood/4$, item 3 in Theorem~\ref{th:basicmultgua} implies that we can select one of the child of $\weight_i$ as $\weight_{i+1}$ so that~\eqref{eq:desc} holds.  Then from the preceding discussion, we have $\weight^\goodbatches_{i+1} \ge 3\sizegood/4$. By repeating this argument, we keep finding the next node in the branch, until we reach the leaf. 
Next, we argue that the leaf at the end of this branch must be in $M$.

Let $\weight$ denote the weight vector for the leaf. From the above discussion, it follows that $\weight^\goodbatches \ge 3\sizegood/4$. Hence, $\weight^\allbatches\ge \weight^\goodbatches \ge 3\sizegood/4\ge 3\goodfrac\size{\allbatches}/4> \goodfrac\size{\allbatches}/2$.

As discussed earlier any leaf $\weight$ is not part of $M$ iff $\weight^\allbatches\le \goodfrac\size{\allbatches}/2$.
Hence, the leaf at the end of the above branch must be in $M$. This concludes the proof of the Theorem.
\end{proof}

\section{Proof of Theorem~\ref{th:conofcov}}\label{sec:proofcov}

In Section~\ref{sec:auxl}, we state and prove two auxiliary lemmas that will be used in proving Theorem~\ref{th:conofcov}, and in Section~\ref{sec:mainconproof}, we prove Theorem~\ref{th:conofcov}.

We will use the following notation in describing the auxiliary lemmas and in the proofs.

Let $S:= \{( x_i^\batch,y_i^\batch): b\in\goodbatches, i\in[\bsize] \}$ denote the collection of all good samples. Note that $\size{S} = \sizegood \bsize$.

For any function $h$ over $(x,y)$, we denote the expectation of $h$ w.r.t. uniform distribution on subset $S'\subseteq S$ by $\E_{S'}[h(x_i^\batch,y_i^\batch)] := \sum_{(x_i^\batch,y_i^\batch)\in S'}\frac{h(x_i^\batch,y_i^\batch)}{|S'|}$.

\subsection{Auxiliary lemmas}\label{sec:auxl}

In this subsection, we state and prove Lemmas~\ref{lem:distlargenormbound} and~\ref{lem:poplargenormbounf}. We will use these lemmas in proof of Theorem~\ref{th:conofcov} in the following subsection.

In the next lemma, for any unit vectors $u$, we bound the expected second moment of the tails of $|x_i^\batch\cdot u|$, for covariate $x_i^\batch$ of a random sample from the distribution $\dist$.
\begin{lemma}\label{lem:distlargenormbound}
For all $\theta >  1 $, and all unit vectors $u\in\reals^d$,  
\[
\Pr{}_\dist[\|x_i^\batch\cdot u\|^2\ge \sqrt{C} \theta] \le \frac{1}{\theta^2} \text{ and } \E{}_{\dist}[\mathbbm 1(\|x_i^\batch\cdot u\|^2\ge \sqrt C\theta)\cdot  \|x_i^\batch\cdot u\|^2]\le \frac{\sqrt C}{\theta}
\]
\end{lemma}
\begin{proof}
The first part of the lemma follows from Markov’s inequality,
\[
\Pr{}_\dist[\|x_i^\batch\cdot u\|^2\ge \sqrt{C} \theta] = \Pr{}_\dist[\|x_i^\batch\cdot u\|^4\ge C \theta^2] = \frac{\E{}_\dist [\|x_i^\batch\cdot u\|^4] }{C\theta^2}\le \frac{1}{\theta^2},
\]
where the last step uses $L4-L2$ hypercontractivity.
This proves the first bound in the lemma.

For the second bound, note that
\begin{align*}
\E{}_{\dist}[\mathbbm 1(\|x_i^\batch\cdot u\|^2\ge \sqrt C\theta)\cdot  \|x_i^\batch\cdot u\|^2] &\ineqlabel{a}\le\sqrt{\E{}_{\dist}[\mathbbm 1(\|x_i^\batch\cdot u\|^2\ge \sqrt C\theta)]\cdot \E{}_{\dist}[ \|x_i^\batch\cdot u\|^4]}\\ 
&\ineqlabel{b}\le\sqrt{\Pr{}_\dist[\|x_i^\batch\cdot u\|^2\ge \sqrt{C} \theta]\cdot C  \E{}_{\dist}[ \|x_i^\batch\cdot u\|^2]^2} \\
&\ineqlabel{c}\le \frac{\sqrt C}{\theta} \E{}_{\dist}[ \|x_i^\batch\cdot u\|^2], 
\end{align*}
here (s) follows from the Cauchy-Schwarz inequality, (b) uses $L4-L2$ hypercontractivity, and (c) follows from the first bound in the lemma.
\end{proof}

In the next lemma, for any unit vectors $u$, we provide a high probability bound on the expected second moment of the tails of $|x_i^\batch\cdot u|$, wheres $x_i^\batch$ are covariates of samples in good batches $\goodbatches$.

\begin{lemma}\label{lem:poplargenormbounf}
For any given $\theta >  1 $, and $\sizegood \bsize =\Omega( d\theta^2\log(\frac{C_1 d\theta}{C}))$, with probability at least $1-2/d^2$, for all unit vectors $u$,
\[
 \E{}_{S}\mleft[\mathbbm 1\mleft(\|x_i^\batch\cdot u\|^2\ge 3\sqrt C\theta\mright)\cdot  \|x_i^\batch\cdot u\|^2\mright]\le \cO\mleft(\frac{\sqrt C}{\theta}\mright).
\]
\end{lemma}

The following lemma restates Lemma 5.1 of~\cite{CherapanamjeriATJFB20}. 
The lemma shows that for any large subset of $S$, the covariance of covariates $x_i^\batch$ in $S$ is close to the true covariance for distribution $\dist$ of samples. We will use this result in proving Lemma~\ref{lem:poplargenormbounf}. 
\begin{lemma}\label{lem:chera}
For any fix $\theta >  1 $, and $\sizegood \bsize =\Omega(  d\theta^2\log(d\theta))$, 
with probability at least $1-1/d^2$ for all subsets of $S'\subseteq S$ of size $\ge (1-\frac{1}{\theta^2})|S|$, we have
    \[
    \Sigma -\cO\mleft(\frac{\sqrt C}{\theta}\mright)\cdot I \preceq \E{}_{S'}[x_{i}^\batch (x_{i}^\batch)^\intercal] \preceq\Sigma +\cO\mleft(\frac{\sqrt C}{ \theta}\mright)\cdot I.
    \]
\end{lemma}

\begin{remark}
Lemma 5.1 of~\cite{CherapanamjeriATJFB20} assumes that hypercontractive parameter $C$ is a constant and its dependence doesn't appear in their lemma but is implicit in their proof.
hides/ignores its dependence.
\end{remark}

The following corollary is a simple consequence. We will use this corollary in proving Lemma~\ref{lem:poplargenormbounf}.
\begin{corollary}\label{cor:cheraconseq}
For any fix $\theta >  1 $, and $\sizegood \bsize =\Omega(  d\theta^2\log(d\theta))$, 
with probability at least $1-1/d^2$ for all subsets $S'\subseteq S$ of size $\le \frac{|S|}{\theta^2}$ and all unit vectors $u$, we have
    \[
   \frac{|S'|}{|S|}\cdot  \E{}_{S'}[ (x_{i}^\batch\cdot u)^2] \preceq  \cO\mleft(\frac{\sqrt C}{\theta}\mright).
    \]
\end{corollary}
\begin{proof}
Consider any set $S'$ of size $\le \frac{|S|}{\theta^2}$.
Since $\size{S\setminus S'}\ge (1-\frac{1}{\theta^2})|S|$, applying Lemma~\ref{lem:chera} for $S\setminus S'$ and $S$,
\[
\Sigma -\cO\mleft(\frac{\sqrt C}{\theta}\mright)\cdot I \preceq \E{}_{S\setminus S'}[x_{i}^\batch (x_{i}^\batch)^\intercal],
\]
and
\[
\E{}_{S}[x_{i}^\batch (x_{i}^\batch)^\intercal] \preceq\Sigma +\cO\mleft(\frac{\sqrt C}{\theta}\mright)\cdot I.
\]
Next,
\begin{align*}
&\E{}_{S}[x_{i}^\batch (x_{i}^\batch)^\intercal]  = \frac{|S'|}{|S|} \E{}_{S'}[x_{i}^\batch (x_{i}^\batch)^\intercal] +  \frac{|S\setminus S'|}{|S|}\E{}_{S\setminus S'}[x_{i}^\batch (x_{i}^\batch)^\intercal]\\
\implies &{|S'|} \E{}_{S'}[x_{i}^\batch (x_{i}^\batch)^\intercal] = |S|\E{}_{S}[x_{i}^\batch (x_{i}^\batch)^\intercal] - {|S\setminus S'|}\E{}_{S\setminus S'}[x_{i}^\batch (x_{i}^\batch)^\intercal]).
\end{align*}
Combining the previous three equations,
\begin{align*}
 {|S'|} \E{}_{S'}[x_{i}^\batch (x_{i}^\batch)^\intercal] &\preceq |S|\mleft(\Sigma +\cO\mleft(\frac{\sqrt C}{\theta}\mright)\cdot I\mright) - {|S\setminus S'|}\mleft(\Sigma -\cO\mleft(\frac{\sqrt C}{\theta}\mright)\cdot I\mright)\\
 &\preceq |S'|\Sigma +(|S|+|S\setminus S'|)\cO\mleft(\frac{\sqrt C}{\theta}\mright)\cdot I\preceq \frac{1}{\theta^2}|S|\Sigma+2|S|\cO\mleft(\frac{\sqrt C}{\theta}\mright)\cdot I\preceq 3|S|\cO\mleft(\frac{\sqrt C}{\theta}\mright)\cdot I,
\end{align*}
where the last line used $\Sigma\preceq I$, $|S'|\le |S|/\theta^2$, $C\ge 1$, and $1/\theta^2\le 1/{\theta}$ for $\theta\ge 1$.

Finally, observing that for any unit vector $ u^\intercal \E{}_{S'}[x_{i}^\batch (x_{i}^\batch)^\intercal] u =  \E{}_{S'}[ (x_{i}^\batch\cdot u)^2]$ completes the proof.
\end{proof}

Now we complete the proof of the Lemma~\ref{lem:poplargenormbounf} with help of the above corollary.
\begin{proof}[Proof of Lemma~\ref{lem:poplargenormbounf}]
From Lemma~\ref{lem:distlargenormbound} we have $\E_\dist[1(\|x_i^\batch\cdot u\|^2 \ge  \sqrt C\theta)] =\Pr[\|x_i^\batch\cdot u\|^2 \ge  \sqrt C\theta]\le \frac{1}{\theta^2}$ Applying Chernoff bound for random variable $\mathbbm 1(\|x_i^\batch\cdot u\|^2 \ge \sqrt C\theta)$, 
\[
\Pr\mleft[\E{}_{S}[1(\|x_i^\batch\cdot u\|^2  \ge \sqrt C\theta)] \le \frac{2}{\theta^2}\mright] =\Pr\mleft[\frac{1}{|S|}\sum_{(i,\batch)\in S}1(\|x_i^\batch\cdot u\|^2  \ge \sqrt C\theta) \le \frac{2}{\theta^2}\mright] \le \exp\mleft(-\frac{|S|}{3\theta^2}\mright).
\]
Hence, for a fix unit vector $u$, with probability $\ge 1- \exp\mleft(-\frac{|S|}{3\theta^2}\mright)$ 
\[
\E{}_{S}[1(\|x_i^\batch\cdot u\|^2  \le \sqrt C\theta)] \le |S|\frac{2}{\theta^2}.
\]
Next, we show that this bound holds uniformly over all unit vectors $u$.

Consider an $\sqrt{\frac{ \sqrt C\theta}{2C_1d}}-$ net of unit sphere $\{u\in \reals^d:\|u\|\le 1\}$ such that for any vector $u$ in this ball there exist a $u'$ in the net such that $\|u-u'\|\le \sqrt{\frac{ \sqrt C\theta}{2C_1d}}$. The standard covering argument~\cite{Vershynin12} shows the existence of such a net of size $e^{\cO(d\log(\frac{C_1 d}{C\theta}))}$.
Then from the union bound, for all vectors $u$ in this net with probability at least $1-e^{\cO(d\log(\frac{C_1 d}{C\theta}))}e^{-\frac{|S|}{3\theta^2}}$,  
\[
   \E{}_{S}[1(\|x_i^\batch\cdot u\|^2  \le \sqrt C\theta)] \le |S|\frac{2}{\theta^2}.
\]
Since $\frac{|S|}{3\theta^2} = \frac{\sizegood\bsize}{3\theta^2} \gg d\log(\frac{C_1 d\theta}{C})\ge d\log(\frac{C_1 d}{C\theta}))$, therefore, $e^{\cO(d\log(\frac{C_1 d}{C\theta}))}e^{-\frac{|S|}{3\theta^2}} \ll e^{-\frac{|S|}{6\theta^2}} \ll 1/d^2$. 

Now consider any vector $u$ in unit ball and $u'$ in the net such that $\|u-u'\|\le \sqrt{\frac{ \sqrt C\theta}{2C_1d}}$.
Then
\begin{align*}
   (x_i^\batch\cdot u)^2 = (x_i^\batch\cdot (u'+(u-u')))^2 &= 2(x_i^\batch\cdot u')^2+ (x_i^\batch\cdot (u-u'))^2 \\
   &\le   2(x_i^\batch\cdot u')^2+2\|u-u'\|^2 \|x_i^\batch\|^2 \\
   &\le  2(x_i^\batch\cdot u')^2  +2 {\frac{ \sqrt C\theta}{2C_1d}}C_1 d \le  2(x_i^\batch\cdot u')^2  +\sqrt C\theta,
\end{align*}
where in the last line we used the assumption that  $\|x_i^b\|\le C_1\sqrt{d}$. When $(x_i^\batch\cdot u')^2\le \sqrt C\theta$, then above sum is bounded by $2\sqrt{C}\theta$.
It follows that with probability $\ge 1-1/d^2$, for all unit vectors $u$, 
\[
    \E{}_{S}[1(\|x_i^\batch\cdot u\|^2  \le 3\sqrt C\theta)] \le |S|\frac{2}{\theta^2}.
\]

Applying Corollary~\ref{cor:cheraconseq} for $S'=\{\|x_i^\batch\cdot u\|^2  \le 3\sqrt C\theta\}$, proves the lemma
\[
\E{}_{S}[\mathbbm 1(\|x_i^\batch\cdot u\|^2\ge 3\sqrt C\theta)\cdot  \|x_i^\batch\cdot u\|^2]= \frac{|S'|}{|S|} \E{}_{S'}[ \|x_i^\batch\cdot u\|^2]\le \cO\mleft(\frac{\sqrt C}{\theta}\mright).
\]
\end{proof}

\subsection{Proof of Theorem~\ref{th:conofcov}}\label{sec:mainconproof}

\begin{proof}[Proof of Theorem~\ref{th:conofcov}]
Note that 
\begin{align*}
\E{}_{\goodbatches}\mleft[\mleft(\clipbatchgrad \cdot u - \E{}_\dist[\clipbatchgrad\cdot u] \mright)^2\mright] &= \frac{1}{\sizegood}\sum_{\batch\in \goodbatches}\mleft(\clipbatchgrad \cdot u - \E{}_\dist[\clipbatchgrad\cdot u] \mright)^2\\
 &= \frac{1}{\sizegood}\sum_{\batch\in \goodbatches}\mleft(\frac{1}{\bsize}\sum_{i\in\bsize}\clipsampgrad \cdot u - \E{}_\dist[\clipbatchgrad\cdot u] \mright)^2\\
 &= \frac{1}{\sizegood}\sum_{\batch\in \goodbatches}\mleft(\frac{1}{\bsize}\sum_{i\in\bsize}\mleft(\clipsampgrad \cdot u - \E{}_\dist[\clipsampgrad\cdot u] \mright)\mright)^2,
\end{align*}
where in the last step we used the expectation of batch and sample gradients are the same, namely $\E{}_\dist[\clipsampgrad\cdot u]=\E{}_\dist[\clipbatchgrad\cdot u]$.

For any positive $\rho>0$ and unit vector $u$,  define 
\[
g^\batch_i(w,\kappa,u,\rho)\ed \frac{\clipsampgrad \cdot u}{\|x_i^\batch\cdot u\|\vee \rho}\rho.
\]
For any good batch $\batch\in\goodbatches$, using the expression of $\clipsampgrad $ in~\eqref{eq:gradexp},
\begin{align}\label{eq:gexp}
g^\batch_i(w,\kappa,u,\rho)= \kappa \rho\mleft(\frac{(x^\batch_i\cdot (w-w^*)-n^\batch_i)}{\|x^\batch_i\cdot (w-w^*)-n^\batch_i\|\vee \kappa}\mright)\mleft(
\frac{x_i^\batch \cdot u}{\|x_i^\batch\cdot u\|\vee \rho}\mright)
.     
\end{align}
From the above expression it follows that $|g^\batch_i(w,\kappa,u,\rho)|\le \kappa\rho$ a.s.

We will choose $\rho$ later in the proof.
Let
\[
Z_i^\batch(w,\kappa,u,\rho) \ed g^\batch_i(w,\kappa,u,\rho)- \E{}_\dist\mleft[g^\batch_i(w,\kappa,u,\rho)\mright].
\]
and 
\begin{align*}
\tilde Z_i^\batch(w,\kappa,u,\rho)& \ed \clipsampgrad\cdot u- \E{}_\dist\mleft[\clipsampgrad\cdot u\mright]-Z_i^\batch(w,\kappa,u,\rho)\\
&= \clipsampgrad\cdot u-g^\batch_i(w,\kappa,u,\rho)- \E{}_\dist\mleft[\clipsampgrad\cdot u-g^\batch_i(w,\kappa,u,\rho)\mright].    
\end{align*}

When $w$, $u$, $\kappa$, and $\rho$ are fixed or clear from the context, we will omit them from the notation of $Z_i^\batch$ and $\tilde Z_i^\batch$.
Then,
\begin{align}
\E{}_{\goodbatches}\mleft[\mleft(\clipbatchgrad \cdot u - \E{}_\dist[\clipbatchgrad\cdot u] \mright)^2\mright]
 &= \frac{1}{\sizegood}\sum_{\batch\in \goodbatches}\mleft(\frac{1}{\bsize}\sum_{i\in\bsize}(Z_i^\batch+\tilde Z_i^\batch) \mright)^2\nonumber\\
  &\le \frac{2}{\sizegood}\sum_{\batch\in \goodbatches}\mleft(\frac{1}{\bsize}\sum_{i\in\bsize}Z_i^\batch \mright)^2 +  \frac{2}{\sizegood}\sum_{\batch\in \goodbatches}\mleft(\frac{1}{\bsize}\sum_{i\in\bsize}\tilde Z_i^\batch \mright)^2\nonumber\\
   &\le \frac{2}{\sizegood}\sum_{\batch\in \goodbatches}\mleft(\frac{1}{\bsize}\sum_{i\in\bsize}Z_i^\batch \mright)^2 +  \frac{2}{\sizegood}\sum_{\batch\in \goodbatches}\frac{1}{\bsize}\sum_{i\in\bsize}(\tilde Z_i^\batch)^2\label{eq:sumbreak},
\end{align}
here in the last step we used Jensen's inequality $(\E[Z])^2\le \E[Z^2]$.

We bound the two summations separately. To bound the first summation we first show that $Z_i^\batch$ are bounded, and then use Bernstein's inequality.
We bound the second term using Lemma~\ref{lem:poplargenormbounf} and Lemma~\ref{lem:distlargenormbound}.


From~\eqref{eq:gexp}, it follows that $|g^\batch_i(w,\kappa,u,\rho)|\le \kappa\rho$ a.s., and therefore, $|Z_{i}^\batch|\le 2\kappa \rho$.

Since $|Z_i^\batch|$ is bounded by $2\kappa \rho$, it is a $(2\kappa \rho)^2$ sub-gaussian random variable. Using the fact that the sum of sub-gaussian random variables is sub-gaussian, the sum $\sum_{i=1}^{\bsize} Z_i^\batch$ is $\bsize (2\kappa \rho)^2$ sub-gaussian random variable.
Since square of a sub-gaussian is sub-exponential~\cite{Philippe15} (Lemma 1.12), hence $ (\sum_{i=1}^{\bsize} Z_i^\batch)^2- \E_{\dist}(\sum_{i=1}^{\bsize} Z_i^\batch)^2$ is sub-exponential with parameter $16\bsize (2\kappa \rho)^2$.

Bernstein’s inequality~\cite{Philippe15} (Theorem 1.12) for sub-Gaussian random variables implies that with probability $\ge 1-\delta$,
\[
\frac{1}{\sizegood}\sum_{b\in\goodbatches}\mleft(\mleft(\sum_{i=1}^{\bsize} Z_i^\batch\mright)^2- \E{}_{\dist}\mleft[\mleft(\sum_{i=1}^{\bsize} Z_i^\batch\mright)^2\mright]\mright) \le 16\bsize (2\kappa \rho)^2 \max\mleft\{\frac{2\ln(1/\delta)}{\sizegood},\sqrt{\frac{2\ln(1/\delta)}{\sizegood}}\mright\} . 
\]

Since $Z_i^\batch$ are zero mean independent random variables,
\[
 \E{}_{\dist}\mleft[\mleft(\sum_{i=1}^{\bsize} (Z_i^\batch)\mright)^2\mright]=  \bsize \E{}_{\dist}\mleft[(Z_i^\batch)^2\mright] .
\]

We bound the expectation on the right,
\begin{align*}
\E{}_{\dist}\mleft[(Z_i^\batch)^2\mright] &= \E{}_{\dist}\mleft[\mleft(g^\batch_i(w,\kappa,u,\rho)- \E{}_\dist\mleft[g^\batch_i(w,\kappa,u,\rho)\mright]\mright)^2\mright]\\
&\ineqlabel{a}\le \E{}_{\dist}\mleft[\mleft(g^\batch_i(w,\kappa,u,\rho)\mright)^2\mright]\\
&\ineqlabel{b}\le \E{}_{\dist}[(\samplenoise +(w-w^*)\cdot x_i^\batch) ^2(x_i^\batch\cdot u)^2] \\
    &\ineqlabel{c}=\E{}_{\dist}[(\samplenoise)^2 (x_i^\batch\cdot u)^2]+\E{}_{\dist}[((w-w^*)\cdot x_i^\batch)^2( x_i^\batch\cdot u)^2]\\
    &\ineqlabel{d}\le \E{}_{\dist}[(\samplenoise)^2] \E{}_{\dist}[ (x_i^\batch\cdot u)^2]+\sqrt{\E{}_{\dist}[((w-w^*)\cdot x_i^\batch)^4]\E{}_{\dist}[(u\cdot x_i^\batch)^4]}\\
    &\ineqlabel{e}\le \sigma^2 \E{}_{\dist}[(x_i^\batch\cdot u)^2]+\sqrt{C^2 \E{}_{\dist}[((w-w^*)\cdot x_i^\batch)^2]^2\E{}_{\dist}[(u\cdot x_i^\batch)^2]^2}\\
     &\ineqlabel{f}\le \sigma^2 +C\E{}_{\dist}[((w-w^*)\cdot x_i^\batch)^2],
\end{align*}
here inequality (a) uses that squared deviation is smaller than mean squared deviation, inequality (b) follows from the definition of $g_i^\batch$ in~\eqref{eq:gexp},  inequality (c) follows from the independence of $n_i^\batch$ and $x_i^\batch$,  inequality (d) follows the Cauchy–Schwarz inequality, (e) uses the L-2 to L-4 hypercontractivity assumption $\E{}_{\dist}[(u\cdot (x_i^\batch))^4]\le C$, and (f) follows as for any unit vector $\E{}_{\dist}[(x_i^\batch\cdot u)^2]\le \|\Sigma\|\le 1$.

Combining the last three equations, we get that with probability $\ge 1-\delta$,
\begin{align}\label{eq:firstcon}
\frac{1}{\sizegood}\sum_{b\in\goodbatches}\mleft(\sum_{i=1}^{\bsize} Z_i^\batch\mright)^2\le \bsize(\sigma^2 +C\E{}_{\dist}[((w-w^*)\cdot x_i^\batch)^2])  + 64\bsize (\kappa \rho)^2 \max\mleft\{\frac{2\ln(1/\delta)}{\sizegood},\sqrt{\frac{2\ln(1/\delta)}{\sizegood}}\mright\} .     
\end{align}
The above bound holds for given fixed values of parameters $\kappa$, $w$, and $u$. To extend the bound for all values of these parameters (for appropriate ranges of interest), we will use the covering argument. 

With the help of the covering argument, we show that with probability  $\ge 1-\delta e^{\cO(d\log(C_1d\bsize)}-\frac{1}{d^2}$, for all unit vectors $u$, all vectors $w$ and $\kappa\le {(\sigma+\|w-w^*\|)}{d^2\bsize}$,
\begin{align}
\frac{1}{\sizegood}\sum_{b\in\goodbatches}\mleft(\sum_{i=1}^{\bsize} Z_i^\batch (w,\kappa,u,\rho)\mright)^2\le \frac{5}{2}\sigma^2\bsize+13C\bsize  \E{}_{\dist}[((w-w^*)\cdot x_i^\batch)^2]+ 384\bsize (\kappa \rho)^2 \max\mleft\{\frac{2\ln(1/\delta)}{\sizegood},\sqrt{\frac{2\ln(1/\delta)}{\sizegood}}\mright\}.\label{eq:firsttermbound}
\end{align}

We delegate the proof of Equation~\eqref{eq:firsttermbound} using Equation~\eqref{eq:firstcon} and the covering argument to the very end.
The use of covering argument is rather standard. The main subtlety is that the above bound holds for all vectors $w$.
The cover size of all $d$ dimensional vectors is infinite.
To overcome this difficulty we first take union bound for vectors for all $w$ such that $\|w-w^*\|\le R$ for an appropriate choice of $R$. To extend it to any $w$ for which $\|w-w^*\|>R$ is large we show that the behavior of the above quantity on the left for such a $w$ can be approximated by its behavior for $w' = w^* + (w-w^*) \frac{R}{\|w-w^*\|}$.

Note that dividing Equation~\eqref{eq:firsttermbound} by $\bsize^2$ bounds the first term in Equation~\eqref{eq:sumbreak}.
Next, we bound the second term in Equation\eqref{eq:sumbreak}.
Note that
\begin{align*}
 \frac{1}{\bsize\sizegood}\sum_{\batch\in \goodbatches}\sum_{i\in\bsize}(\tilde Z_i^\batch)^2 
 &\le \frac{1}{\bsize\sizegood}\sum_{\batch\in \goodbatches}\sum_{i\in\bsize}\mleft(\clipsampgrad\cdot u-g^\batch_i(w,\kappa,u,\rho) - \E{}_\dist\mleft[\clipsampgrad\cdot u-g^\batch_i(w,\kappa,u,\rho)\mright]\mright)^2\\
 &\le \frac{2}{\bsize\sizegood}\sum_{\batch\in \goodbatches}\sum_{i\in\bsize}\mleft(\mleft(\clipsampgrad\cdot u-g^\batch_i(w,\kappa,u,\rho) \mright)^2+\mleft(\E{}_\dist\mleft[\clipsampgrad\cdot u-g^\batch_i(w,\kappa,u,\rho)\mright]\mright)^2\mright)\\
 &\le \frac{2}{\bsize\sizegood}\sum_{\batch\in \goodbatches}\sum_{i\in\bsize}\mleft(\mleft(\clipsampgrad\cdot u-g^\batch_i(w,\kappa,u,\rho) \mright)^2+\E{}_\dist\mleft[\mleft(\clipsampgrad\cdot u-g^\batch_i(w,\kappa,u,\rho)\mright)^2\mright]\mright).
\end{align*}
From the definitions of $g_i^\batch(w,\kappa,u,\rho)$ and $\clipsampgrad$,
\begin{align*}
|\clipsampgrad\cdot u-g^\batch_i(w,\kappa,u,\rho)| &= \mathbbm 1(\|x_i^\batch\cdot u\|\ge \rho)\mleft|\clipsampgrad\cdot u-\frac{\rho }{\|x_i^\batch\cdot u\|}\clipsampgrad\cdot u \mright|\\
&\le \mathbbm 1(\|x_i^\batch\cdot u\|\ge \rho)\mleft|\clipsampgrad\cdot u \mright|\\
&\le \kappa\mleft|x_i^\batch\cdot u\mright|\cdot\mathbbm 1(\|x_i^\batch\cdot u\|\ge \rho).
\end{align*}
From the above equation, it follows that
\[
\E{}_\dist\mleft[\mleft(\clipsampgrad\cdot u-g^\batch_i(w,\kappa,u,\rho)\mright)^2\mright]\le \kappa^2\E{}_\dist\mleft[\mathbbm 1(\|x_i^\batch\cdot u\|\ge \rho)\mleft| x_i^\batch\cdot u \mright|^2\mright].
\]
Combining the above three bounds,
\begin{align*}
 \frac{1}{\bsize\sizegood}\sum_{\batch\in \goodbatches}\sum_{i\in\bsize}(\tilde Z_i^\batch)^2 
 &\le \frac{2\kappa^2}{\bsize\sizegood}\sum_{\batch\in \goodbatches}\sum_{i\in\bsize}\mleft(\mathbbm 1(\|x_i^\batch\cdot u\|\ge \rho)\mleft| x_i^\batch\cdot u \mright|^2+\E{}_\dist\mleft[\mathbbm 1(\|x_i^\batch\cdot u\|\ge \rho)\mleft| x_i^\batch\cdot u \mright|^2\mright]\mright)\\
 &={2\kappa^2}\mleft(\E{}_{S}\mleft[\mathbbm 1(\|x_i^\batch\cdot u\|\ge \rho)\mleft| x_i^\batch\cdot u \mright|^2\mright]+\E{}_\dist\mleft[\mathbbm 1(\|x_i^\batch\cdot u\|\ge \rho)\mleft| x_i^\batch\cdot u \mright|^2\mright]\mright),
\end{align*}
here the last line uses the fact that $S$ is the collection of all good samples.


For $\rho^2\ge 3\sqrt C$, and $\sizegood \bsize =\Omega( d\rho^4\log(\frac{C_1 d\rho}{C}))$, Lemma~\ref{lem:poplargenormbounf} implies that with probability at least $1-2/d^2$, for all unit vectors $u$, we have 
\[
\E{}_{S}\mleft[\mathbbm 1(\|x_i^\batch\cdot u\|\ge \rho)\mleft| x_i^\batch\cdot u \mright|^2\mright]= \E{}_{S}\mleft[\mathbbm 1(\|x_i^\batch\cdot u\|^2\ge \rho^2)\mleft| x_i^\batch\cdot u \mright|^2\mright]\le \cO(\sqrt{C}/\rho^2) .
\]
And from Lemma~\ref{lem:distlargenormbound}, for $\rho^2\ge \sqrt C$ and any unit vectors $u$,
\[
\E{}_\dist\mleft[\mathbbm 1(\|x_i^\batch\cdot u\|\ge \rho)\mleft| x_i^\batch\cdot u \mright|^2\mright] \le \cO(\sqrt{C}/\rho^2).
\]

By combining the above three bounds it follows that, if  $\rho^2\ge \sqrt C$, and $\sizegood \bsize =\Omega( d\rho^4\log(\frac{C_1 d\rho}{C}))$, with probability at least $1-2/d^2$, for all unit vectors $u$,
\begin{align*}
 \frac{1}{\bsize\sizegood}\sum_{\batch\in \goodbatches}\sum_{i\in\bsize}(\tilde Z_i^\batch(w,\kappa,u,\beta))^2 \le \cO\mleft(\frac{\sqrt{C}\kappa^2}{\rho^2}\mright).   
\end{align*}
Combining the above bound, Equation~\eqref{eq:firsttermbound} and~\eqref{eq:sumbreak} we get that if $\rho^2=\Omega(\sqrt C)$, and $\sizegood  =\Omega( \frac{d\rho^4}\bsize\log(\frac{C_1 d\rho}{C}))$ then with probability  $\ge 1-\delta e^{\cO(d\log(C_1d\bsize)}-\frac{3}{d^2}$, for all unit vectors $u$, all vectors $w$ and $\kappa\le {(\sigma+\|w-w^*\|)}{d^2\bsize}$,
\begin{align*}
&\E{}_{\goodbatches}\mleft[\mleft(\clipbatchgrad \cdot u - \E{}_\dist[\clipbatchgrad\cdot u] \mright)^2\mright]\\
&\le \frac{2}{\bsize^2}\mleft(\frac{5}{2}\sigma^2\bsize+13C\bsize  \E{}_{\dist}[((w-w^*)\cdot x_i^\batch)^2]+ 384\bsize (\kappa \rho)^2 \max\mleft\{\frac{2\ln(1/\delta)}{\sizegood},\sqrt{\frac{2\ln(1/\delta)}{\sizegood}}\mright\} \mright)+    \cO\mleft(\frac{\sqrt{C}\kappa^2}{\rho^2}\mright).
\end{align*}

Recall that $1\le \mu_{\max}\le  \frac{d^4\bsize^2}{C}$. Choose $\rho^2 =\mu_{\max}\sqrt C\bsize$. Note that $\sqrt{\mu_{\max}(\sigma^2+ C \E{}_{\dist}[((w-w^*)\cdot x_i^\batch)^2])}\le {(\sigma+\|w-w^*\|)}{d^2\bsize}$.
Then from the above equation choosing $\rho^2 =\mu_{\max}\sqrt{C}\bsize$, for all \[\kappa\le \sqrt{\mu_{\max}(\sigma^2+ C \E{}_{\dist}[((w-w^*)\cdot x_i^\batch)^2])},\] with probability $\ge 1-\delta e^{\cO(d\log(C_1d\bsize)}-\frac{3}{d^2}$, for all unit vectors $u$, all vectors $w$ ,
\begin{align*}
&\E{}_{\goodbatches}\mleft[\mleft(\clipbatchgrad \cdot u - \E{}_\dist[\clipbatchgrad\cdot u] \mright)^2\mright]\\
&\le \cO\mleft(\frac{\sigma^2+ C \E{}_{\dist}[((w-w^*)\cdot x_i^\batch)^2]}{\bsize}\mright) \mleft( 1+\bsize\mu_{\max}^2\sqrt{C} \max\mleft\{\frac{2\ln(1/\delta)}{\sizegood},\sqrt{\frac{2\ln(1/\delta)}{\sizegood}}\mright\} \mright).
\end{align*}
Choose $\delta = e^{-\Theta(d\log(C_1d\bsize)}$, and $\sizegood = \Omega( \frac{d\rho^4}{\bsize}\log(\frac{C_1 d\rho}{C}) + C\mu_{\max}^4  d\bsize^2\log(C_1d\bsize)) = \Omega(\rho_{\max}^4\bsize^2 d\log(d))$.
Then with probability  $\ge 1-\frac{4}{d^2}$, for all unit vectors $u$, all vectors $w$ and for all $\kappa^2\le \mu_{\max}(\sigma^2+ C \E{}_{\dist}[((w-w^*)\cdot x_i^\batch)^2])$,  
\begin{align*}
&\E{}_{\goodbatches}\mleft[\mleft(\clipbatchgrad \cdot u - \E{}_\dist[\clipbatchgrad\cdot u] \mright)^2\mright]\le \cO\mleft(\frac{\sigma^2+ C \E{}_{\dist}[((w-w^*)\cdot x_i^\batch)^2]}{\bsize}\mright),
\end{align*}
which is the desired bound.

We complete the proof by proving Equation~\eqref{eq:firsttermbound}. 

\paragraph{Proof of Equation~\eqref{eq:firsttermbound}}
To complete the proof of the theorem next we prove Equation~\eqref{eq:firsttermbound} with the help of Equation~\eqref{eq:firstcon} and covering argument.
To use the covering argument, we first show that $g^\batch_i(w,\kappa,u,\rho)$ do not change by much by slight deviation of these parameters. From the definition of $Z^\batch_i(w,\kappa,u,\rho)$, the same conclusion would then hold for it.

By the triangle inequality, 
\begin{align*}
&|g^\batch_i(w,\kappa,u,\rho) - g^\batch_i(w',\kappa',u',\rho)| \\
&\le  |g^\batch_i(w',\kappa',u,\rho) - g^\batch_i(w',\kappa',u',\rho)|+|g^\batch_i(w,\kappa',u,\rho) - g^\batch_i(w',\kappa',u,\rho)|+|g^\batch_i(w,\kappa,u,\rho) - g^\batch_i(w,\kappa',u,\rho)|.    
\end{align*}
We bound each term on the right one by one.
To bound these terms we use Equation~\eqref{eq:gexp}, the assumption that $\|x_i^\batch\|\le C_1\sqrt d$ and the definition of the function $g()$.
For the first term,
\[
|g^\batch_i(w',\kappa',u,\rho) - g^\batch_i(w',\kappa',u',\rho)| \le \|(u-u')x_i^\batch\|\kappa'\le C_1\|u-u'\|\sqrt d \kappa',
\]
for the second term,
\[
|g^\batch_i(w,\kappa',u,\rho) - g^\batch_i(w',\kappa',u,\rho)| \le |u\cdot x_i^\batch| \cdot|(w-w')\cdot x_i^\batch|\le \| x_i^\batch\|^2\cdot\|w-w'\|\le C_1^2d \|w-w'\|,
\]
and for the last term
\[
|g^\batch_i(w,\kappa,u,\rho) - g^\batch_i(w,\kappa',u,\rho)|\le |\kappa-\kappa'|\cdot |u\cdot x_i^\batch| \le C_1\sqrt{d}|\kappa-\kappa'|.
\]
Combining the three bounds,
\begin{align*}
&|g^\batch_i(w,\kappa,u,\rho) - g^\batch_i(w',\kappa',u',\rho)| \le  C_1\|u-u'\|\sqrt d \kappa'+C_1^2d \|w-w'\|+C_1\sqrt{d}|\kappa-\kappa'|.    
\end{align*}

For $ \|u-u'\|\le 1/(24C_1d^5\bsize^3)$, $\kappa'\le 2d^4\sigma\bsize^2$, $\|w-w'\|\le \sigma/(12d C_1^2\bsize)$ and $|\kappa-\kappa'|\le \sigma/(12C_1d\bsize)$, 
\begin{align*}
&|g^\batch_i(w,\kappa,u,\rho) - g^\batch_i(w',\kappa',u',\rho)| \le  \sigma/4\bsize \text{ a.s.}    
\end{align*}
This would imply,
\begin{align*}
&|Z^\batch_i(w,\kappa,u,\rho) - Z^\batch_i(w',\kappa',u',\rho)| \le  \sigma/2\bsize \text{ a.s.}   
\end{align*}
Using this bound,
\begin{align}
\frac{1}{\sizegood}\sum_{b\in\goodbatches}\mleft(\sum_{i=1}^{\bsize} Z_i^\batch (w,\kappa,u,\rho)\mright)^2&\le  \frac{1}{\sizegood}\sum_{b\in\goodbatches}\mleft(\sum_{i=1}^{\bsize} \mleft(Z_i^\batch (w',\kappa',u',\rho')+\frac{\sigma}{2\bsize}\mright)\mright)^2\nonumber\\
&\le\frac{2}{\sizegood}\sum_{b\in\goodbatches}\mleft(\sum_{i=1}^{\bsize} Z_i^\batch (w',\kappa',u',\rho')\mright)^2 +\frac{2}{\sizegood}\sum_{b\in\goodbatches}\mleft(\sum_{i=1}^{\bsize} \frac{\sigma}{2\bsize}\mright)^2 \nonumber\\
&\le \frac{2}{\sizegood}\sum_{b\in\goodbatches}\mleft(\sum_{i=1}^{\bsize} Z_i^\batch (w',\kappa',u',\rho')\mright)^2+\frac{\sigma^2}2\label{eq:nettoall}.
\end{align}
Let $\cU\ed\{ u\in \reals^d : \|u\|=1\}$, $\cW\ed\{w\in\reals^d: \|w-w^*\| \le d^2  \sigma\bsize)\}$, and $\cK\ed \mleft[0,2{d^4\sigma\bsize^2}\mright]$.

Standard covering argument shows that there exist covers such that 
\begin{align}\label{eq:coverbound1}
\cU'\subseteq \cU: \forall u\in \cU, \min_{u'\in\cU'}\|u-u' \|\le \frac{1}{(24C_1d^5\bsize^3)} ,
\end{align}
\begin{align}\label{eq:coverbound2}
\cW'\subseteq \cW: \forall w\in \cW, \min_{w'\in\cW'}\|w-w' \|\le \frac{\sigma}{12C_1^2d\bsize} ,
\end{align}
and 
\begin{align}\label{eq:coverbound3}
\cK'\subseteq \cK: \forall \kappa\in \cK, \min_{\kappa'\in\cK', \kappa'\ge \kappa}|\kappa-\kappa' |\le  \frac{\sigma}{12C_1d\bsize} ,
\end{align}
and the size of each is $|\cU'|,|\cW'|, |\cK'|\le e^{\cO(d\log(C_1d\bsize)}$.

In equation~\eqref{eq:firstcon}, taking the union bound over all elements in $\cU'$, $\cW'$ and $\cK'$, it follows that  
with probability $\ge 1-\delta e^{\cO(d\log(C_1d\bsize)}$, 
for all $u'\in\cU'$, $w'\in\cW'$ and $\kappa'\in \cK'$
\begin{align*}
\frac{1}{\sizegood}\sum_{b\in\goodbatches}\mleft(\sum_{i=1}^{\bsize} Z_i^\batch (w',\kappa',u',\rho)\mright)^2\le \bsize(\sigma^2 +C\E{}_{\dist}[((w'-w^*)\cdot x_i^\batch)^2])  + 64\bsize (\kappa' \rho)^2 \max\mleft\{\frac{2\ln(1/\delta)}{\sizegood},\sqrt{\frac{2\ln(1/\delta)}{\sizegood}}\mright\} .     
\end{align*}
Combining the above bound with Equation~\eqref{eq:nettoall}, it follows that with probability $\ge 1-\delta e^{\cO(d\log(C_1d\bsize)}$, for all $u\in\cU$, $w\in\cW$ and $\kappa\in \cK$ and elements $u'$, $w'$ and $\kappa'$ in the respective nets satisfying equations~\eqref{eq:coverbound1},\eqref{eq:coverbound2}, and \eqref{eq:coverbound3},
\begin{align}
&\frac{1}{\sizegood}\sum_{b\in\goodbatches}\mleft(\sum_{i=1}^{\bsize} Z_i^\batch (w,\kappa,u,\rho)\mright)^2\le 2\bsize(\sigma^2 +C\E{}_{\dist}[((w'-w^*)\cdot x_i^\batch)^2]) + 128\bsize (\kappa' \rho)^2 \max\mleft\{\frac{2\ln(1/\delta)}{\sizegood},\sqrt{\frac{2\ln(1/\delta)}{\sizegood}}\mright\}+\frac{\sigma^2}2 \nonumber\\
&\le 2\bsize(\sigma^2 (1+\frac{1}{4\bsize})+C\E{}_{\dist}[((w'-w^*)\cdot x_i^\batch)^2]) + 128\bsize (\kappa \rho)^2 \max\mleft\{\frac{2\ln(1/\delta)}{\sizegood},\sqrt{\frac{2\ln(1/\delta)}{\sizegood}}\mright\}\nonumber\\
&\le 2\bsize(\frac{5}{4}\sigma^2 +2C\E{}_{\dist}[((w-w^*)\cdot x_i^\batch)^2])+ 128\bsize (\kappa \rho)^2 \max\mleft\{\frac{2\ln(1/\delta)}{\sizegood},\sqrt{\frac{2\ln(1/\delta)}{\sizegood}}\mright\}
. \label{eq:limwbound}    
\end{align}
here (a) follows from the bound $\kappa\ge \kappa'$ in Equation~\eqref{eq:coverbound3}, and (b) follows by first writing $w'-w^* = (w-w^*) + (w'-w)$ and then using the bound $\|w'-w\|\le\frac{\sigma}{12C_1^2d\bsize}$ in Equation~\eqref{eq:coverbound2}.

Next, we further remove the restriction $w\in \cW$ and extend the above bound to all vectors $w$.

Consider a $w\notin \cW$ and $\kappa\in \mleft[0,{(\sigma+\|w-w^*\|)}{d^2\bsize}\mright]$.
From the definition of $\cW$, we have $\|w-w^*\|> d^2\sigma\bsize$.
Let $w' = w^* + \frac{w-w^*}{\|w-w^*\|}d^2\sigma \bsize $ and $\kappa' = \frac{d^2\sigma\bsize}{\|w-w^*\|}\kappa$.
Observe that $\|w'-w\| =d^2\sigma \bsize$ and 
\[
\kappa'\le {(\sigma+\|w-w^*\|)}{d^2\bsize}\frac{d^2\sigma\bsize}{\|w-w^*\|}\le d^2\sigma n \frac{d^2\sigma\bsize}{\|w-w^*\|}+ d^4\sigma\bsize^2\le {d^2\sigma\bsize}+d^4\sigma\bsize^2\le 2d^4\sigma\bsize^2,
\]
hence, $w'\in \cW$ and $\kappa'\in \cK$.
From Equation~\eqref{eq:gexp},
\begin{align*}
&\mleft|\frac{\|w-w^*\|}{d^2\sigma\bsize}g^\batch_i(w',\kappa',u,\rho) - g^\batch_i(w,\kappa,u,\rho)\mright| \\
&\ineqlabel{a}=
\frac{\rho\|x_i^\batch \cdot u\|}{\|x_i^\batch\cdot u\|\vee \rho}\cdot \mleft| \frac{\|w-w^*\|}{d^2\sigma\bsize}\kappa'\mleft(\frac{(x^\batch_i\cdot (w'-w^*)-n^\batch_i)}{\|x^\batch_i\cdot (w'-w^*)-n^\batch_i\|\vee \kappa'}\mright) -
\kappa\mleft(\frac{(x^\batch_i\cdot (w-w^*)-n^\batch_i)}{\|x^\batch_i\cdot (w-w^*)-n^\batch_i\|\vee \kappa}\mright)\mright|\\
&\ineqlabel{b}\le
\|x_i^\batch \cdot u\|\cdot \mleft| \kappa\mleft(\frac{(x^\batch_i\cdot  \frac{w-w^*}{\|w-w^*\|}d^2\sigma \bsize-n^\batch_i)}{\|x^\batch_i\cdot  \frac{w-w^*}{\|w-w^*\|}d^2\sigma \bsize-n^\batch_i\|\vee \frac{d^2\sigma\bsize}{\|w-w^*\|}\kappa}\mright) -
\kappa\mleft(\frac{(x^\batch_i\cdot (w-w^*)-n^\batch_i)}{\|x^\batch_i\cdot (w-w^*)-n^\batch_i\|\vee \kappa}\mright)\mright|\\
&=
\|x_i^\batch \cdot u\|\cdot \mleft| 
\kappa\mleft(\frac{(x^\batch_i\cdot (w-w^*)-\frac{d^2\sigma \bsize}{\|w-w^*\|}n^\batch_i)}{\|x^\batch_i\cdot (w-w^*)-\frac{d^2\sigma \bsize}{\|w-w^*\|}n^\batch_i\|\vee \kappa}\mright)-\kappa\mleft(\frac{(x^\batch_i\cdot  (w-w^*)-n^\batch_i)}{\|x^\batch_i\cdot (w-w^*)-n^\batch_i\|\vee \kappa}\mright) \mright|\\
&\ineqlabel{c}\le \|x_i^\batch \cdot u\|\cdot \mleft|\frac{d^2\sigma \bsize}{\|w-w^*\|}n^\batch_i -n_i^\batch \mright|\\
&\ineqlabel{d}\le C_1\sqrt d |n_i^\batch|\frac{d^2\sigma\bsize}{\|w-w^*\|}\\
&\ineqlabel{e}\le  C_1\sqrt d |n_i^\batch|,    
\end{align*}
here (a) follows from the definition of $g_i^\batch$, inequality (b) follows as $\rho\le\|x_i^\batch\cdot u\|\vee \rho$, inequality (c) uses the fact that for any $a,\Delta$ and $b\ge 0$, we have $|b \frac{a +\Delta}{(a+\Delta)\vee b} - b \frac{a}{a\vee b}|\le |\Delta|$, inequality (c) uses $\|x_i^\batch\|\le  C_1\sqrt d$ and the last inequality (e) uses $\|w-w^*\|> d^2\sigma\bsize$.

Therefore,
\[
\mleft|\frac{\|w-w^*\|}{d^2\sigma\bsize}Z^\batch_i(w',\kappa',u,\rho) - Z^\batch_i(w,\kappa,u,\rho)\mright|\le C_1\sqrt d\mleft(|n_i^\batch|+\E[|n_i^\batch|]\mright)\le C_1\sqrt d\mleft(|n_i^\batch|+\sigma\mright).
\]
From the above equation,
\begin{align*}
\frac{1}{\sizegood}\sum_{b\in\goodbatches}\mleft(\sum_{i=1}^{\bsize} Z_i^\batch (w,\kappa,u,\rho)\mright)^2
& \le \frac{1}{\sizegood}\sum_{b\in\goodbatches}\mleft(\sum_{i=1}^{\bsize} \mleft(\frac{\|w-w^*\|}{d^2\sigma\bsize} Z_i^\batch (w',\kappa',u,\rho)+ C_1\sqrt d|n_i^\batch|+ C_1\sqrt d\sigma\mright)\mright)^2\\
&  \ineqlabel{a}\le\frac{1}{\sizegood}\sum_{b\in\goodbatches}\mleft(3 \mleft(\sum_{i=1}^{\bsize} \frac{\|w-w^*\|}{d^2\sigma\bsize}Z_i^\batch (w',\kappa',u,\rho)\mright)^2+3\mleft(\sum_{i=1}^{\bsize} C_1\sqrt d|n_i^\batch|\mright)^2+3\mleft(\sum_{i=1}^{\bsize} C_1\sqrt d\sigma\mright)^2\mright)\\
& \ineqlabel{b} \le\frac{3\|w-w^*\|^2}{d^4\sigma^2\bsize^2} \frac{1}{\sizegood}\sum_{b\in\goodbatches}\mleft(\sum_{i=1}^{\bsize} Z_i^\batch (w',\kappa',u,\rho)\mright)^2+\frac{3dC_1^2}{\sizegood}\sum_{b\in\goodbatches}\mleft(\bsize\sum_{i=1}^{\bsize}|n_i^\batch|^2\mright)+{3dC_1^2}{\bsize^2}\sigma^2\\
& \ineqlabel{c}\le\frac{3\|w-w^*\|^2}{d^4\sigma^2\bsize^2} \frac{1}{\sizegood}\sum_{b\in\goodbatches}\mleft(\sum_{i=1}^{\bsize} Z_i^\batch (w',\kappa',u,\rho)\mright)^2+3d C_1^2n^2\sigma^2(d^2+1), \end{align*}
here (a) and (b) uses $(\sum_{i=1}^t z_i)\le t\sum_{i=1}^t z_i^2$ and inequality (c) holds with with probability $\ge 1-\frac{1}{d^2}$ by Markov inequality, as $Pr[\frac{1}{\bsize\sizegood}\sum_{b\in\goodbatches}\sum_{i=1}^{\bsize}|n_i^\batch|^2 > d^2\E_{\dist}[(n_i^\batch)^2)]\le \frac1{d^2}$.

Recall that $\|w'-w\|\le d^2\sigma\bsize$ and $\kappa\le 2d^4\sigma\bsize^2$, therefore in the above equation, we can bound the first term on the right by using high probability bound in Equation~\eqref{eq:limwbound}.
Then, with probability $\ge 1-\delta e^{\cO(d\log(C_1d\bsize)}-\frac{1}{d^2}$,
\begin{align*}
&\frac{1}{\sizegood}\sum_{b\in\goodbatches}\mleft(\sum_{i=1}^{\bsize} Z_i^\batch (w,\kappa,u,\rho)\mright)^2\\
& \le \frac{3\|w-w^*\|^2} {d^4\sigma^2\bsize^2} \mleft(2\bsize(\frac{5}{4}\sigma^2 +2C\E{}_{\dist}[((w'-w^*)\cdot x_i^\batch)^2])+ 128\bsize (\kappa' \rho)^2 \max\mleft\{\frac{2\ln(1/\delta)}{\sizegood},\sqrt{\frac{2\ln(1/\delta)}{\sizegood}}\mright\}\mright)+3dC_1^2{\bsize^2}\sigma^2(d^2+1)\\
& \ineqlabel{a}= 12C\bsize \E{}_{\dist}[((w-w^*)\cdot x_i^\batch)^2]  + \frac{15\|w-w^*\|^2} {2d^4\sigma^2\bsize^2}\bsize\sigma^2 + 384\bsize (\kappa \rho)^2 \max\mleft\{\frac{2\ln(1/\delta)}{\sizegood},\sqrt{\frac{2\ln(1/\delta)}{\sizegood}}\mright\}+3dC_1^2{\bsize^2}\sigma^2(d^2+1)\\
& \ineqlabel{b}\le 13C\bsize \E{}_{\dist}[((w-w^*)\cdot x_i^\batch)^2]+ 384\bsize (\kappa \rho)^2 \max\mleft\{\frac{2\ln(1/\delta)}{\sizegood},\sqrt{\frac{2\ln(1/\delta)}{\sizegood}}\mright\},
\end{align*}
here equality (a) uses the relation $w'-w^* = \frac{w-w^*}{\|w-w^*\|}d^2\sigma \bsize $ and $\kappa' = \frac{d^2\sigma\bsize}{\|w-w^*\|}\kappa$, and (b) follows as $C\E{}_{\dist}[((w-w^*)\cdot x_i^\batch)^2]\ge C\frac{\|w-w^*\|^2\|\Sigma\|}{\connum}= C\frac{\|w-w^*\|^2}{\connum}\ge Cd^4\sigma^2n^2/\connum$, where $\connum$ is condition number of $\Sigma$, hence $C\E{}_{\dist}[((w-w^*)\cdot x_i^\batch)^2]\gg \frac{15\|w-w^*\|^2} {2d^4\sigma^2\bsize^2}\bsize\sigma^2$ and $C\E{}_{\dist}[((w-w^*)\cdot x_i^\batch)^2]\ge C\frac{\|w-w^*\|^2}{\connum}\gg \frac{15\|w-w^*\|^2} {2d^4\sigma^2\bsize^2}\bsize\sigma^2$ and $C\E{}_{\dist}[((w-w^*)\cdot x_i^\batch)^2]\gg 3dC_1^2{\bsize^2}\sigma^2(d^2+1)$.

The above bound holds for all unit vectors $u$, $w\notin \cW$ and $\kappa\le {(\sigma+\|w-w^*\|)}{d^2\bsize}$,
\[
\frac{1}{\sizegood}\sum_{b\in\goodbatches}\mleft(\sum_{i=1}^{\bsize} Z_i^\batch (w,\kappa,u,\rho)\mright)^2\le 13C\bsize  \E{}_{\dist}[((w-w^*)\cdot x_i^\batch)^2]+ 450\bsize (\kappa \rho)^2 \max\mleft\{\frac{2\ln(1/\delta)}{\sizegood},\sqrt{\frac{2\ln(1/\delta)}{\sizegood}}\mright\}.
\]
Recall that bound in Equation~\eqref{eq:limwbound} holds for all unit vectors $u$, $w\in \cW$ and $\kappa\le \cK'$ with probability  $\ge 1-\delta e^{\cO(d\log(C_1d\bsize)}$. Note that for $w\in \cW$,  $  {(\sigma+\|w-w^*\|)}{d^2\bsize}\le {2d^4\sigma\bsize^2}$, hence $[0,{(\sigma+\|w-w^*\|)}{d^2\bsize}]\subseteq \cK'$.
Hence the above bound holds for all unit vectors $u$, $w\in \cW$ and $\kappa\le {(\sigma+\|w-w^*\|)}{d^2\bsize}$.
Combining the two bounds, with probability  $\ge 1-\delta e^{\cO(d\log(C_1d\bsize)}-\frac{1}{d^2}$, for all unit vectors $u$, all vectors $w$ and $\kappa\le {(\sigma+\|w-w^*\|)}{d^2\bsize}$,
\begin{align*}
\frac{1}{\sizegood}\sum_{b\in\goodbatches}\mleft(\sum_{i=1}^{\bsize} Z_i^\batch (w,\kappa,u,\rho)\mright)^2\le \frac{5}{2}\sigma^2\bsize+13C\bsize  \E{}_{\dist}[((w-w^*)\cdot x_i^\batch)^2]+ 384\bsize (\kappa \rho)^2 \max\mleft\{\frac{2\ln(1/\delta)}{\sizegood},\sqrt{\frac{2\ln(1/\delta)}{\sizegood}}\mright\}.
\end{align*}
This completes the proof of Equation~\eqref{eq:firsttermbound}.
\end{proof}

\end{document}

%% file: locdef.tex
\newcommand{\btotal}{m}
\newcommand{\bsize}{n}

\newcommand{\goodfrac}{\alpha}
\newcommand{\allbatches}{{B}}
\newcommand{\goodbatches}{{G}}
\newcommand{\advbatches}{A}

\newcommand{\Bsc}{{\allbatches^\prime}}
\newcommand{\Gsc}{{\goodbatches^\prime}}

\newcommand{\size}[1]{|{#1}|}

\newcommand{\sizegood}{\size{\goodbatches}}

\newcommand{\dist}{\cD}
\newcommand{\batch}{b}
\newcommand{\reals}{\mathbb{R}}
\newcommand{\domain}{{\reals}^d}

\newtheorem{condition}{Condition}

\newcommand{\sampgrad}{\nabla f^\batch_i (w)}
\newcommand{\sampgradtilde}{\nabla f^\batch_i (\tilde w)}

\newcommand{\clipsampgrad}{\nabla f^\batch_i(w,\kappa)}

\newcommand{\connum}{{C_3}}

\newcommand{\weight}{\beta}

\newcommand{\batchgrad}{\nabla f^\batch (w)}
\newcommand{\batchgradtilde}{\nabla f^\batch (\tilde w)}
\newcommand{\batchgradexp}{\sum_{i\in [\bsize]}\frac{1}{\bsize} \sampgrad}

\newcommand{\clipbatchgrad}{\nabla f^\batch(w,\kappa)}
\newcommand{\clipbatchgradexp}{\sum_{i\in [\bsize]}\frac{1}{\bsize} \clipsampgrad}
\newcommand{\Cov}{\text{Cov}}
\newcommand{\samplenoise}{n_i^\batch}

\newcommand{\nhypcon}{C_{\npower}}
\newcommand{\npower}{p}

\newcommand{\ucb}{c_2}
\newcommand{\ucmf}{c_3}
\newcommand{\uccb}{c_4}
\newcommand{\uccc}{c_5}
\newcommand{\thresholdA}{\theta_0}
\newcommand{\thresholdB}{\theta_1}
\newcommand{\thresholdC}{\theta_2}
\newcommand{\wc}{w_C}